\documentclass{article}
\usepackage{natbib}
\setcitestyle{authoryear,round,citesep={;},aysep={,},yysep={;}}
\usepackage[margin=1in]{geometry}

\usepackage[utf8]{inputenc} 
\usepackage[T1]{fontenc}    
\usepackage{amsfonts,amsmath,amssymb,amsthm}       
\usepackage{thmtools, thm-restate}
\usepackage{graphicx}
\graphicspath{ {pics/} }
\usepackage{hyperref}
\usepackage[capitalise]{cleveref}
\hypersetup{
    colorlinks,
    linkcolor={red!50!black},
    citecolor={blue!50!black},
    urlcolor={blue!80!black}
}
\usepackage{url}            
\usepackage{booktabs}       

\usepackage{mathtools}
\usepackage{caption,subcaption,enumitem}
\usepackage{tablefootnote}
\usepackage{nicefrac}       
\usepackage{microtype}      
\usepackage{xcolor}         
\usepackage{bm}
\usepackage{physics}

\newcommand{\Yoh}{\ensuremath{\mY^{\mathsf{oh}}}}
\newcommand{\yoh}{\ensuremath{\vy^{\mathsf{oh}}}}

\newcommand{\tran}{\ensuremath{^\top}}

\newcommand{\EE}{\ensuremath{\mathbb{E}}}

\newcommand{\cn}{\ensuremath{\mathsf{CN}}}

\newcommand{\lambdabold}{\ensuremath{\boldsymbol{\lambda}}}

\newcommand{\zj}{\ensuremath{\vz_j}}
\renewcommand{\a}{\ensuremath{\alpha}}
\renewcommand{\b}{\ensuremath{\beta}}
\newcommand{\za}{\ensuremath{\vz_\a}}

\newcommand{\zb}{\ensuremath{\vz_\b}}
\newcommand{\ya}{\ensuremath{\vy_\a}}

\newcommand{\yb}{\ensuremath{\vy_\b}}

\newcommand{\dely}{\ensuremath{\boldsymbol{\Delta}y}}

\newcommand{\nc}{\ensuremath{k}}
\newcommand{\Ainv}{\ensuremath{\bm{A}^{-1}}}

\newcommand{\fhat}{\widehat{\vf}}

\newcommand{\hhat}{\widehat{\vh}}

\newcommand{\xtest}{\vx_{\mathsf{test}}}
\newcommand{\xwtest}{\xtest^w}

\newcommand{\cor}{\normalfont{\textsf{cor}}}
\newcommand{\su}{\normalfont{\textsf{SU}}}
\newcommand{\err}{\normalfont{\textsf{err}}}

\newcommand{\Asinv}{\ensuremath{\mA_{-s}^{-1}}}
\newcommand{\Akinv}{\ensuremath{\mA_{-k}^{-1}}}
\newcommand{\Ask}{\ensuremath{\mA_{-F}}}
\newcommand{\Askinv}{\ensuremath{\Ask^{-1}}}

\newcommand{\Ajinv}{\ensuremath{\mA_{-j}^{-1}}}

\newcommand{\cnf}{\cn_{\a, \b, F}}
\newcommand{\cnu}{\cn_{\a, \b, U}}

\newcommand{\cnl}{\cn_{\a, \b, L}}

\newcommand{\lambdaab}{\Lambda_{\a, \b}}
\newcommand{\corl}{\cor_{\a, \b, L}}
\newcommand{\corf}{\cor_{\a, \b, F}}
\newcommand{\coru}{\cor_{\a, \b, U}}
\newcommand{\diag}{\mathrm{diag}}
\newcommand{\offdiag}{\mathrm{offdiag}}

\newcommand{\Au}{\mA_U}
\newcommand{\Af}{\mA_{-F}}
\newcommand{\Auinv}{\mA_U^{-1}}

\newcommand{\Ak}{\mA_{-k}}
\newcommand{\As}{\mA_{-s}}

\newcommand{\Ar}{\mA_{-R}}
\newcommand{\Arinv}{\Ar^{-1}}
\newcommand{\At}{\mA_{-T}}
\newcommand{\Atinv}{\At^{-1}}
\newcommand{\AS}{\mA_{-S}}
\newcommand{\ASinv}{\AS^{-1}}

\newcommand{\Xs}{\mW_s}
\newcommand{\Xk}{\mW_k}
\newcommand{\Zk}{\mZ_k}

\newcommand{\Zf}{\mZ_F}
\newcommand{\Xsk}{\mW_{F}}
\newcommand{\Zsk}{\mZ_{F}}

\newcommand{\Zt}{\mZ_T}
\newcommand{\Xt}{\mW_T}

\newcommand{\Zl}{\mZ_L}
\newcommand{\Xr}{\mW_R}
\newcommand{\Zr}{\mZ_R}
\newcommand{\lamdel}{{\Lambda_{\delta}}}
\newcommand{\lamdelc}{\Lambda_{\delta}^c}
\newcommand{\Zb}{Z^{(\beta)}}
\newcommand{\Xab}{\mW_{\a, \b}}

\newcommand{\Ber}{\mathsf{Ber}}

\def\1{\bm{1}}

\def\vone{{\bm{1}}}

\def\vgamma{{\bm{\gamma}}}

\def\vf{{\bm{f}}}
\def\vg{{\bm{g}}}
\def\vh{{\bm{h}}}

\def\vu{{\bm{u}}}
\def\vv{{\bm{v}}}
\def\vw{{\bm{w}}}
\def\vx{{\bm{x}}}
\def\vy{{\bm{y}}}
\def\vz{{\bm{z}}}

\def\mA{{\bm{A}}}
\def\mB{{\bm{B}}}
\def\mC{{\bm{C}}}
\def\mD{{\bm{D}}}
\def\mE{{\bm{E}}}

\def\mH{{\bm{H}}}
\def\mI{{\bm{I}}}

\def\mM{{\bm{M}}}

\def\mP{{\bm{P}}}

\def\mV{{\bm{V}}}
\def\mW{{\bm{W}}}
\def\mX{{\bm{X}}}
\def\mY{{\bm{Y}}}
\def\mZ{{\bm{Z}}}

\DeclareMathAlphabet{\mathsfit}{\encodingdefault}{\sfdefault}{m}{sl}
\SetMathAlphabet{\mathsfit}{bold}{\encodingdefault}{\sfdefault}{bx}{n}

\def\gE{{\mathcal{E}}}

\DeclareMathOperator*{\argmax}{arg\,max}

\newcommand{\RR}{\mathbb{R}}
\DeclareMathOperator{\PP}{\mathbb{P}}

\DeclareMathOperator{\poly}{poly}

\newcommand{\ol}{\overline}
\newcommand{\ul}{\underline}

\newcommand{\mGamma}{\bm{\Gamma}}

\newcounter{cnstcnt}
\newcommand{\newc}{%
\refstepcounter{cnstcnt}%
\ensuremath{c_{\thecnstcnt}}}
\newcommand{\const}[1]{\ensuremath{c_{\ref{#1}}}}

\newcounter{Cnstcnt}
\newcommand{\newC}{%
\refstepcounter{Cnstcnt}%
\ensuremath{C_{\theCnstcnt}}}
\newcommand{\Const}[1]{\ensuremath{C_{\ref{#1}}}}

\newcounter{kappacnt}
\newcommand{\newk}{%
\refstepcounter{kappacnt}%
\ensuremath{\kappa_{\thekappacnt}}}
\newcommand{\konst}[1]{\ensuremath{\kappa_{\ref{#1}}}}

\newcounter{eventcnt}

\crefname{assumption}{assumption}{assumptions}
\Crefname{assumption}{Assumption}{Assumptions}

\newtheorem{theorem}{Theorem}[section]
\newtheorem{lemma}[theorem]{Lemma}
\newtheorem{corollary}[theorem]{Corollary}

\newtheorem{proposition}[theorem]{Proposition}

\newtheorem{definition}{Definition}

\newtheorem{assumption}{Assumption}
\newtheorem{conjecture}[theorem]{Conjecture}

\title{Precise Asymptotic Generalization for Multiclass Classification with Overparameterized Linear Models}

\author{%
  David X.~Wu \\
  Department of EECS\\
  UC Berkeley\\
  Berkeley, CA 94720 \\
  \texttt{david\_wu@berkeley.edu} \\
  \and
  Anant Sahai \\
  Department of EECS \\
  UC Berkeley\\
  Berkeley, CA 94720 \\
  \texttt{sahai@eecs.berkeley.edu} \\
}

\begin{document}

\maketitle

\begin{abstract}
    We study the asymptotic generalization of an overparameterized linear model for multiclass classification under the Gaussian covariates bi-level model introduced in \citet{subramanian2022generalization}, where the number of data points, features, and classes all grow together. We fully resolve the conjecture posed in \citet{subramanian2022generalization}, matching the predicted regimes for generalization. Furthermore, our new lower bounds are akin to an information-theoretic strong converse: they establish that the misclassification rate goes to 0 or 1 asymptotically. One surprising consequence of our tight results is that the min-norm interpolating classifier can be asymptotically suboptimal relative to noninterpolating classifiers in the regime where the min-norm interpolating regressor is known to be optimal. 
    
    The key to our tight analysis is a new variant of the Hanson-Wright inequality which is broadly useful for multiclass  problems with sparse labels. As an application, we show that the same type of analysis can be used to analyze the related multilabel classification problem under the same bi-level ensemble. 
\end{abstract}

\section{Introduction}
In this paper, we directly follow up on a specific line of work initiated by \citet{subramanian2022generalization,mult:wu2023lowerbounds}. For the sake of self-containedness, we briefly reiterate the context, directing the reader to \citet{subramanian2022generalization} and the references cited therein for more. A broader story can be found in \citet{survey:bartlett2021deep, survey:belkin2021fit, survey:dar2021farewell,oneto2023we}.  

Classical statistical learning theory intuition predicts that highly expressive models, which can interpolate random labels \citep{reg:zhang2016understanding, zhang2021understanding}, ought not to generalize well. However, deep learning practice has seen such models performing well when trained with good labels. Resolving this apparent contradiction has recently been the focus of a multitude of works, and this paper builds on one particular thread of investigation that can be rooted in \citet{reg:bartlett2020benign, reg:muthukumar2020harmless} where the concept of benign/harmless interpolation was crystallized in the context of overparameterized linear regression problems and conditions given for when this can happen. In \citet{binary:Muth20}, a specific toy "bi-level model" with Gaussian features was introduced to study overparameterized binary classification and show that successful generalization could happen even beyond the conditions for benign interpolation for regression. Following the introduction of the corresponding multi-class problem in \citet{multi_class_theory:Wang21} with a constant number of classes, an asymptotic setting where the number of classes can grow with the number of training examples was introduced in  \citet{subramanian2022generalization} where a  conjecture was presented for when minimum-norm interpolating classifiers will generalize. We are now in a position to state our main contributions; afterwards, we expand on the related works. 

\subsection*{Our contributions}
Our main contribution is crisply identifying the asymptotic regimes where an overparameterized linear model that performs minimum-norm interpolation does and does not generalize for multiclass classification under a Gaussian features assumption, thus resolving the main conjecture posed by \citep{subramanian2022generalization}. We improve on the analysis of \citet{subramanian2022generalization,mult:wu2023lowerbounds}, covering all regimes with the asymptotically optimal misclassification rate. 
When the model generalizes, it does so with a misclassification rate $o(1)$, and we show a matching "strong converse" establishing when it misclassifies, it does so with rate $1 - o(1)$, where the explicit rate is nearly identical to that of random guessing. The critical component of our analysis is a new variant of the Hanson-Wright inequality, which applies to bilinear forms between a vector with subgaussian entries and a vector that is bounded and has \emph{soft sparsity}, a notion we will define in \cref{sec:hanson-wright-exposition}. We show how this tool can be used to analyze other multiclass problems, such as multilabel classification. 

\subsection{Brief treatment of related work}
Our thread begins with a recent line of work that analyzes the generalization behavior of overparameterized linear models for regression \citep{reg:hastie2022surprises, reg:mei2022generalization, reg:bartlett2020benign, reg:belkin2020two, reg:muthukumar2020harmless}. These simple models demonstrate how the capacity to interpolate noise can actually aid in generalization: training noise can be harmlessly absorbed by the overparameterized model without contaminating predictions on test points.
In effect, extra features can be regularizing (in the context of descent algorithms' implicit regularization \citep{implicit_bias:srebro, implicit_bias:ji2019, implicit_bias:engl1996regularization, implicit_bias:gunasekar2018characterizing}), but an excessive amount of such regularization causes regression to fail because even the true signal will not survive the training process. Although works in this thread focus on very shallow networks, \citet{chatterji2023deep} established that deeper networks can behave similarly. Note that recently, \citet{mallinar2022benign} called out an alternative regime (behaving like 1-nearest-neighbor learning) called "tempered" overfitting in which training noise is not completely absorbed but the true signal does survive training.

The thread continues in a line of work that studies binary classification \citep{binary:Muth20, binary_classification:chatterji2021finite, binary_classification:wang2021benign} in similar overparameterized linear models. 
While confirming that the basic story is similar to regression, these works identify a further surprise: binary classification can work in some regimes where the corresponding regression problem would not work\footnote{Regression failing in the overparameterized regime is linked to the empirical covariance of the limited data not revealing the spiked reality of the underlying covariance \citep{wang2017asymptotics}. See Appendix J of \citet{subramanian2022generalization}. When regression doesn't generalize, we also get "support-vector proliferation" in classification problems \citep{binary:Muth20, binary_classification:hsu2021proliferation} which is also intimately related to the phenomenon of "neural collapse" \citep{papyan2020prevalence} as discussed, for example, in \citet{xu2023dynamics}.} due to the regularizing effect of overparameterization being too strong. Just as in the regression case, the results here are sharp in toy models: we can exactly characterize where binary classification using an interpolating classifier asymptotically generalizes.

With binary classification better understood, the thread continues to multiclass classification. After all, the current wave of deep learning enthusiasm originated in breakthrough performance in multiclass classification, and we have seen a decade of ever larger networks trained on ever larger datasets with ever more classes \cite{kaplan2020scaling}. Using similar toy models  \citep{reg:muthukumar2020harmless, binary:Muth20, reg:wang2022tight}, the constant number of classes case was studied in \citet{multi_class_theory:Wang21} to recover results similar to binary classification and subsequently generalized to general convex losses with regularization in \citet{mult:loureiro2021learning} and student-teacher networks in \citet{mult:cornacchia2023learning}. 

\citet{subramanian2022generalization} further introduced a model where the number of classes grows with the number of training points and proved an achievability result on how fast the number of classes can grow while still allowing the interpolating classifier to asymptotically generalize. While \citet{subramanian2022generalization} gave a conjecture for what the full region should be, there was no converse proof, and they could not show generalization in the entire conjectured region. \citet{mult:wu2023lowerbounds} proved a partial weak converse; they showed that the misclassification rate is bounded away from $0$ --- rather than tending to $1$ --- in some of the predicted regimes.

The model, formally defined in the following section, is a stylized version of the well-known spiked covariance model \citep{johnstone2001distribution,donoho2018optimal}. On the theoretical front, it is related to several problems such as PCA variants \citep{montanari2015non,richard2014statistical} and community detection in the stochastic block model \citep{abbe2017community}. These models have also been applied in practice for climate studies and functional data analysis \citep{johnstone2001distribution}. At a high level, spiked covariance models can be interpreted as a linearized version of the manifold hypothesis. 
\section{Problem setup}
\label{sec:setup}
The following exposition is lifted from \citet{subramanian2022generalization}, which we include for the sake of staying consistent and self-contained. In \cref{sec:exposition}, we include an alternative high level framing of the problem which readers may find helpful for intuition.

We consider the multiclass classification problem with $k$ classes. The training data consists of $n$ pairs $\{\vx_i,\ell_i\}_{i=1}^n$ where $\vx_i \in \RR^d$ are i.i.d standard Gaussian vectors\footnote{Following previous work, we are staying within a Gaussian features framework. However, recent developments have confirmed that these models are actually predictive when the features arise from nonlinearities in a lifting, as long as there is enough randomness underneath \citep{hu2022universality, lu2022equivalence, goldt2022gaussian, misiakiewicz2022spectrum, mcrae2022harmless, pesce2023gaussian, kaushik2023new}.}. 
We assume that the labels $\ell_i \in [k]$ are generated as follows.
\begin{assumption}[1-sparse noiseless model]
\label{assumption:1sparse}
 The class labels $\ell_i$ are generated based on which of the first $k$ dimensions of a point $\vx_i$ has the largest value,
\begin{align}
    \ell_i = \argmax_{m \in [k]} \vx_i[m].
    \label{eq:truelabels}
\end{align}
\end{assumption}
Let us emphasize at this point that the classifier that we analyze only observes the training data, and does not use any of the data-generating assumptions. 

For a vector $\vx$, we index its $j$th entry with $\vx[j]$. Hence, under Assumption~\ref{assumption:1sparse}, $\vx_i[m]$ can be interpreted as  how representative of class $m$ the $i$th training point is. 

For clarity of exposition in the analysis, we make explicit a feature weighting that transforms the training points:
\begin{align}
    \vx^w_i[j] =  \sqrt{\lambda_j} \vx_i[j] \quad \forall j \in [d]. \label{eq:lambdas}
\end{align}
Here $\lambdabold \in \mathbb{R}^d$ contains the squared feature weights. 
The feature weighting serves the role of favoring the true pattern, something that is essential for good generalization. Again, we emphasize that the classifier does not do any reweighting of features; this explicit step is purely syntactic. \footnote{Our weighted feature model is equivalent to other works (e.g. \citet{binary:Muth20}) that assume that the covariates come from a $d-$dimensional anisotropic Gaussian with a covariance matrix $\Sigma$ that favors the truly important directions \citep{pmlr-v162-wei22a}. These directions do not have to be axis-aligned --- we make that assumption only for notational convenience. In reality, the optimizer will never know these directions {\em a priori}.} 

The weighted feature matrix $\mX^w \in \mathbb{R}^{n \times d}$ is given by
\begin{align}
    \mX^w &= \begin{bmatrix} \vx^w_1 & \cdots  & \vx^w_n \end{bmatrix}^\top = \mqty[ \sqrt{\lambda_1} \vz_1 & \cdots & \sqrt{\lambda_d} \vz_d]\label{eq:weightedxmatrix}
\end{align}
where 
we introduce the notation $\vz_j \in \mathbb{R}^n$ to contain the $j^{th}$ feature from the $n$ training points. Note that $\vz_j \sim N(0, \mI_n)$ are i.i.d Gaussians. 
We use a one-hot encoding for representing the labels as the matrix $\mY^{\mathsf{oh}} \in \mathbb{R}^{n \times k}$
\begin{align}
    \Yoh &= \begin{bmatrix} \yoh_1 & \cdots & \yoh_k \end{bmatrix}, \quad\quad \text{where}\quad\quad 
    \yoh_m[i] = \begin{cases} 1, &  \text{if} \ \ell_i = m \\ 0, & \text{otherwise} \end{cases} \label{eq:yoh}. 
\end{align}
Since we consider linear models, we center the one-hot encodings and define
\begin{align}
    \vy_m \triangleq \yoh_m - \frac{1}{k} \vone
    \label{eq:zeromeany}.
\end{align}

Our classifier consists of $k$ coefficient vectors $\fhat_m$ for $m \in [k]$ that are learned by minimum-norm interpolation (MNI) of the zero-mean one-hot variants using the weighted features:\footnote{The classifier learned via this method is equivalent to those obtained by other natural training methods (SVMs or gradient-descent with exponential tailed losses like cross-entropy) under sufficient overparameterization \citep{multi_class_theory:Wang21, kaushik2023new}. Recently, \citet{lai2023general} showed via an extension of \citet{ji2021characterizing} that a much broader category of losses also asymptotically results in convergence to the same MNI solution for sufficiently overparameterized classification problems.}
\begin{align}
\fhat_m &= \arg \min_\vf \| \vf \|_2\label{eq:interpolateadjustedonehot}\\
\text{s.t.}\ & \mX^w \vf = \vy_m.
\end{align}

We can express these coefficients in closed form as
\begin{align} 
        \fhat_m =  (\mX^w)^\top \left(\mX^w (\mX^w)^\top\right)^{-1} \vy_m.\label{eq:closed-form-coefficients}
\end{align}

On a test point $\xtest \sim N(0,\mI_d)$\label{def:xtest} we predict a label as follows: First, we transform the test point into the weighted feature space to obtain $\xwtest$ where $\xwtest[j] = \sqrt{\lambda_j} \xtest[j]$ for $j \in [d]$\label{def:xwtest}. Then we compute $k$ scalar ``scores'' and assign the class based on the largest score as follows:
\begin{align}
    \hat{\ell} = \argmax_{1 \leq m \leq k } \fhat_m\tran \xwtest.
\end{align}
By assumption, a misclassification event $\gE_{\err}$ occurs whenever
\begin{align}
    \argmax_{1 \le m \le k} \xtest[m] \ne \argmax_{1 \le m \le k} \fhat_m\tran \xwtest. \label{eq:misclassification-definition}
\end{align}
    
We study where the MNI generalizes in an asymptotic regime where the number of training points, features, classes, and feature weights all scale according to the bi-level ensemble model\footnote{Such models are widely used to study learning even beyond this particular thread of work. For example, \citet{tan2023blessing} uses this to understand the privacy/generalization tradeoff of overparameterized learning.}:
\begin{definition}[Bi-level ensemble]\label{def:bilevel}
The bi-level ensemble is parameterized by $p,q,r$ and $t$ where $p > 1$, $0 \leq r < 1$, $0 < q < (p - r)$ and $0 \leq  t < r$. Here, parameter $p$ controls the extent of overparameterization, $r$ determines the number of favored features, $q$ controls the weights on favored features, and $t$ controls the number of classes. The number of features ($d$), number of favored features ($s$), and number of classes ($k$) all scale with the number of training points ($n$) as follows:
\begin{align}
    d = \lfloor n^p \rfloor, s = \lfloor n^r \rfloor, a = n^{-q}, k = c_k \lfloor n^{t} \rfloor, \label{eq:bilevelparamscaling}
\end{align}
where $c_k$ is a positive integer.
Define the feature weights by
\begin{align}
\sqrt{\lambda_j} = \begin{cases} \sqrt{\frac{ad}{s}}, & 1 \leq j \leq s \\
                        \sqrt{\frac{(1-a)d}{d-s}}, & \mathrm{otherwise}\end{cases} . \label{eq:bilevellambdascaling}
\end{align}
We introduce the notation $\lambda_F \triangleq \frac{ad}{s}$ and $\lambda_U \triangleq \frac{(1-a)d}{d-s}$ to distinguish between the (squared) favored and unfavored weights, respectively. 
\end{definition}

We visualize the bi-level model in \cref{fig:bilevelmodel}, reproduced from \citet{subramanian2022generalization}. Intuitively, the bi-level ensemble captures a simple family of overparameterized problems where learning can succeed. Although there are $d = n^p$ features where $d \gg n$, there is a low dimensional subspace of favored, higher weight features of dimension $s = n^r$, and $s \ll n$. From this perspective, the bi-level model can be viewed as a parameterized version of an approximate linear manifold hypothesis. Depending on the signal strength, the noise added from the $d-s$ unfavored features can either help generalization (``benign overfitting'') or overwhelm the true signal and cause the classifier to fail. 

\begin{figure}[ht!]
  \centering
  \includegraphics[width=0.7\columnwidth]{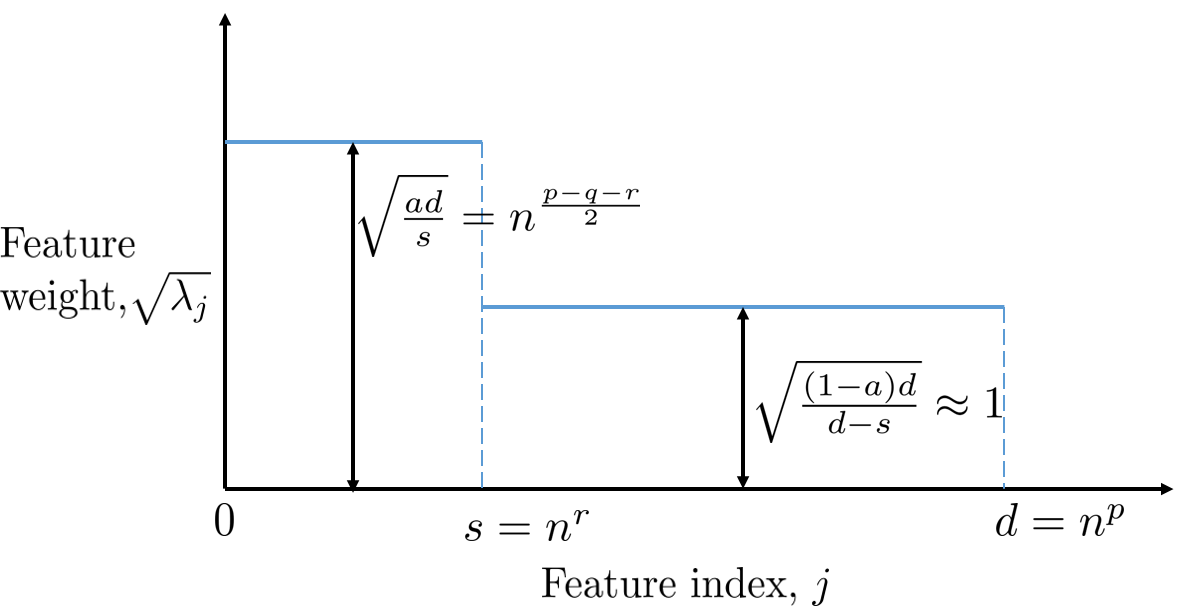}
  \caption{Bi-level feature weighting model. The first $s$ features have a higher weight and are favored during minimum-norm interpolation. These can be thought of as the square roots of the eigenvalues of the feature covariance matrix $\Sigma$ in a Gaussian model for the covariates as in \citet{reg:bartlett2020benign}.}
  \label{fig:bilevelmodel}
\end{figure}

\section{Main results}
In this section, we state our main results and compare them to what was known and conjectured previously.
\citet{subramanian2022generalization} use heuristic calculations to conjecture necessary and sufficient conditions for the bi-level model to generalize; we restate the conjecture here for reference.

\begin{conjecture}[Conjectured bi-level regions]\label{conjecture:regimes}
Under the bi-level ensemble model (\cref{def:bilevel}), when the true data generating process is 1-sparse (Assumption~\ref{assumption:1sparse}), as $n\rightarrow \infty$, the probability of misclassification $\Pr[\gE_{\err}]$ for MNI as described in \cref{eq:interpolateadjustedonehot} satisfies
\begin{align}
     \Pr[\gE_{\err}] & \rightarrow \begin{cases} 0,& \ \mathrm{if} \  t < \min\qty{1-r, p+1-2\max\qty{1, q+r}} \\
      1,& \ \mathrm{if} \ t > \min\qty{1-r, p+1-2\max\qty{1, q+r}}
      \end{cases}.
      \label{eq:conjectured_regimes}
\end{align}
\end{conjecture}

Our main theorem establishes that \cref{conjecture:regimes} indeed captures the correct generalization behavior of the overparameterized linear model.

\begin{theorem}[Generalization for bi-level-model]\label{thm:positive}
Under the bi-level ensemble model (\cref{def:bilevel}), when the true data generating process is 1-sparse (Assumption~\ref{assumption:1sparse}),
\cref{conjecture:regimes} holds.
\end{theorem}

For comparison, we quote the best previously known positive and negative results for the bi-level model, which only hold in the restricted regime where regression fails ($q+r>1$).
\begin{theorem}[Generalization for bi-level model \citep{subramanian2022generalization}]
\label{thm:old-positive}
In the same setting as \cref{conjecture:regimes}, in the regime where regression fails $(q+r>1$), as $n\rightarrow \infty$ we have $\Pr[\gE_{\err}] \to 0$ if 
\begin{align}
     t < \min\qty{1-r, p+1-2(q+r), p-2, 2q+r-2}.
\end{align}
\end{theorem}
\begin{theorem}[Misclassification in bi-level model \citep{mult:wu2023lowerbounds}]\label{thm:old-negative}
In the same setting as \cref{conjecture:regimes}, in the regime where regression fails $(q+r>1$), as $n\rightarrow \infty$ we have $\Pr[\gE_{\err}] \ge \frac{1}{2}$ if 
\begin{align}
    t > \min\qty{1-r, p+1-2(q+r)}.
\end{align}
\end{theorem}

Let us interpret the different conditions in \cref{conjecture:regimes}. To interpret the condition $t < 1 - r$, first rearrange it to $t + r < 1$. Recall that the parameter $r$ controls the number of favored features, and hence is a proxy for the ``effective dimension'' of the problem. On the other hand, the parameter $t$ controls the number of classes, so in a loose sense there are $n^{t+r}$ parameters being learned. From this perspective, the condition $t + r < 1$ says that the problem is ``effectively underparameterized''.  

The other condition on $p+1 - 2\max\qty{1, q+r}$ comes from looking at the noise from the unfavored features. To see why, recall that the squared favored feature weighting is $\lambda_F = n^{p-q-r}$. So for fixed $p$, the quantity $q+r$ controls the level of favored feature weighting. When $q+r>1$, the favored feature weighting is small enough that regression fails, and the empirical covariance becomes flat. In this case, the condition becomes $t < p+1-2(q+r) = (p-q-r) + 1-(q+r)$. As $q+r$ increases, the amount of favoring decreases, making it harder to generalize. 

For ease of comparison between our main result and \cref{thm:old-positive,thm:old-negative}, we visualize the regimes in \cref{fig:regimes}, as in \citet{subramanian2022generalization,mult:wu2023lowerbounds}. In particular, the blue starred and dashed regions in \cref{fig:regimes} indicate how \cref{thm:old-positive} only applies where regression fails. In contrast, our new result holds regardless of whether regression fails or not, as in the green diamond region and light blue triangle regions. The regions are also completely tight; the looseness between the prior \cref{thm:old-positive} and our result can be seen in the light blue square region. 

The weak converse in the prior \cref{thm:old-negative} captures some of the correct conditions for misclassification, but again only when $q+r>1$. As depicted in the maroon X region for $r < 0.25$ in \cref{fig:regimes}b, our main theorem gives a strong converse, whereas \cref{thm:old-negative} has nothing to say because $q+r < 1$. \cref{thm:old-negative} also only proves that the misclassification rate is asymptotically at least $\frac{1}{2}$. In the red circle and maroon X regions, we illustrate how our result pushes the misclassification rate to $1 - o(1)$, which requires a more refined analysis. We elaborate on this further in \cref{sec:proof-sketch}. 

We remark that it is simpler to analyze the case where regression fails, as the random matrices that arise in the analysis are \emph{flat}, i.e. approximately equal to a scaled identity matrix. However, in the regime where regression works, the same matrices have a \emph{spiked} spectrum, which complicates the analysis. To smoothly handle both cases, we leverage a new variant of the Hanson-Wright inequality to show concentration of certain sparse bilinear forms; see \cref{sec:su-cn-overview} for more details.

\begin{figure*}[!ht]
  \centering
  \includegraphics[width=0.8\textwidth]{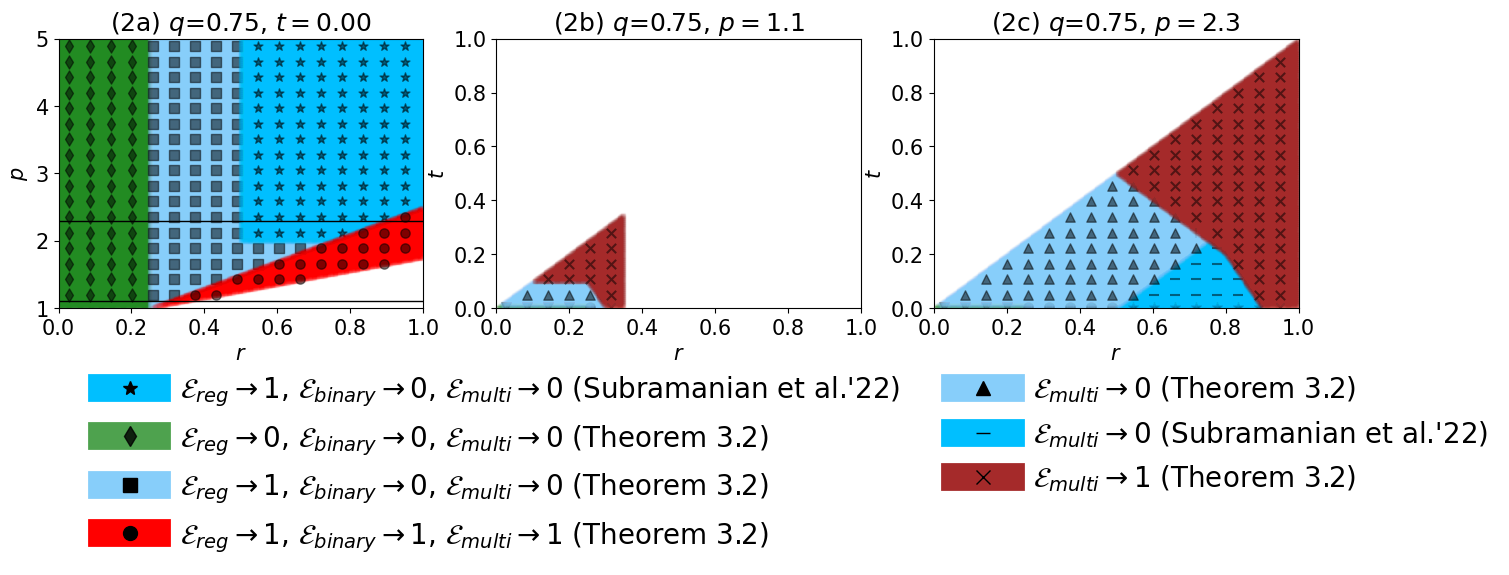}
  \caption{Example of regimes for multiclass/binary classification and regression. The white regions correspond to invalid regimes under the bi-level model. The entirety of 2b and all the light blue regions are new to this paper, as is showing that the error tends to $1$ in the maroon regions.
  }
  \label{fig:regimes}
\end{figure*}

\section{Technical overview}\label{sec:proof-sketch}
We now sketch out the proof for our main theorem. As in \citet{subramanian2022generalization}, the starting point is writing out the necessary and sufficient conditions for misclassification. 

Assume without loss of generality that the test point $\xtest \sim N(0, \mI_d)$ has true label $\alpha$ for some $\alpha \in [k]$. Let $\xwtest$ be the weighted version of this test point. 
From \eqref{eq:misclassification-definition}, an equivalent condition for misclassification is that for some $\beta \neq \alpha, \beta \in [k]$, we have $\fhat_\a^\top \xwtest < \fhat_\b^\top \xwtest$, i.e. the score for $\b$ outcompetes the score for $\a$. Define the Gram matrix $\mA \triangleq \mX^w(\mX^w)^\top$, the relative label vector $\dely \triangleq \ya - \yb \in \qty{-1, 0, 1}^n$, and the relative survival vector $\hhat_{\a, \b} \in \RR^d$ which compares the signal from $\a$ and $\b$:
\begin{align}
    \hhat_{\a,\b}[j] &\triangleq \lambda_j^{-1/2}(\fhat_\a[j] - \fhat_\b[j]) \label{eq:hhat_def}\\
    &= \vz_j^\top \Ainv \dely,
\end{align}
where to obtain the last line we have used \eqref{eq:closed-form-coefficients}.
By converting the misclassification condition into the unweighted feature space we see that we will have errors when
\begin{align}
    &\lambda_\a\hhat_{\a,\b}[\a] \xtest[\a] -  \lambda_\b\hhat_{\b,\a}[\b] \xtest[\b] < \sum_{j \notin \{\a, \b \}} \lambda_j \hhat_{\b,\a}[j]\xtest[j]. \label{eq:RHS}
\end{align}
Define the contamination term $\cn_{\a, \b}$:
\begin{equation}
    \cn_{\a,\b} \triangleq \sqrt{\sum_{j \notin \{\a, \b\}} \lambda_j^2 (\hhat_{\b,\a}[j])^2}. \label{eq:cn}
\end{equation}
Note that $\cn_{\a, \b}$ normalizes the RHS of \eqref{eq:RHS} into a standard Gaussian. Indeed, define
\begin{equation}
    \Zb \triangleq \frac{1}{\cn_{\a, \b}} \sum_{j \notin \{\a, \b \}} \lambda_j \hhat_{\b,\a}[j]\xtest[j] \sim N(0, 1). \label{def:zb}
\end{equation}

Since $\a, \b \in [k]$ are favored, we have $\lambda_\a = \lambda_\b = \lambda_F$. Hence an equivalent condition for misclassification is that there exists some $\b \neq \a$, $\b \in [k]$ such that 
\begin{align}
    &\frac{\lambda_F}{\cn_{\a, \b}}(\hhat_{\a,\b}[\a] \xtest[\a] -  \hhat_{\b,\a}[\b] \xtest[\b]) < Z^{(\b)}. \label{eq:equiv-misclassification}
\end{align}
We now translate the above criterion into \emph{sufficient} conditions for correct classification and misclassification and analyze these two cases separately. 

\paragraph{Correct classification:}

For correct classification, it suffices for the minimum value of the LHS of \cref{eq:equiv-misclassification} to outcompete the maximum value of the RHS, where the max is taken over $\b \in [k], \b \neq \a$. Some algebra, as in \citet{subramanian2022generalization}, shows that we correctly classify if
\begin{align}
    \underbrace{\frac{\min_\beta \lambda_F \hhat_{\a,\b}[\a]}{\max_\b \cn_{\a,\b}}}_{\su/\cn\text{ ratio}} \Bigg( \underbrace{\min_\b \left(\xtest[\a] - \xtest[\b]\right)}_{\text{closest feature margin}} - \underbrace{\max_\b |\xtest[\b]|}_{\text{largest competing feature}} &\cdot \underbrace{\max_\b \left|\frac{\hhat_{\a,\b}[\a] - \hhat_{\b,\a}[\b]}{\hhat_{\a,\b}[\a]} \right|}_{\text{survival variation}}\Bigg) \nonumber \\
   &> \underbrace{\max_{\beta} Z^{(\b)}}_{\text{normalized contamination}}. \label{eq:suff-class-standardized}
\end{align}

We will show that under the conditions specified in \cref{conjecture:regimes}, with high probability, the relevant survival to contamination ratio $\su/\cn$ grows at a polynomial rate $n^{v}$ for some $v>0$, whereas the term in the parentheses shrinks at a subpolynomial rate $\omega(n^{-\delta})$ for any $\delta > 0$. Further, by standard subgaussian maximal inequalities, the magnitudes of the \emph{normalized contamination} are no more than $O(\sqrt{\log(nk)})$ with high probability. Thus with high probability the LHS outcompetes the RHS, leading to correct classification. See \cref{sec:su-cn-overview} for more discussion on how we prove tight bounds on the survival-to-contamination ratios.

\paragraph{Misclassification:}
On the other hand, for misclassification it suffices for the maximum \emph{absolute} value of the LHS of \cref{eq:equiv-misclassification} to be outcompeted by the maximum value of the RHS. Some manipulations yield the following sufficient condition for misclassification: 
\begin{align}
        \underbrace{\frac{\max_\b \lambda_F \qty(\abs{\hhat_{\a,\b}[\a]} + \abs{\hhat_{\b,\a}[\b]})}{\min_\b \cn_{\a,\b}}}_{\su/\cn \text{ ratio}} \cdot  \underbrace{\max_{\gamma \in [k]} \abs{\xtest[\gamma]}}_{\text{largest label-defining feature}} 
   < \underbrace{\max_{\beta} Z^{(\b)}}_{\text{normalized contamination}}. \label{eq:suff-misclass-standardized}
\end{align}
We show that within the misclassification regimes in \cref{conjecture:regimes}, the survival-to-contamination ratio $\su/\cn$ \emph{shrinks} at a polynomial rate $n^{-w}$ for some $w > 0$. By standard subgaussian maximal inequalities, the largest label-defining feature is $O(\sqrt{\log(nk)})$ with high probability. Gaussian anticoncentration implies that for some $\b \neq \a, \b \in [k]$, $\Zb$ outcompetes the LHS with probability at least $\frac{1}{2} - o(1)$. 
Hence, we conclude that the model will misclassify with rate at least $\frac{1}{2}$ asymptotically.

Let us now describe how to boost the misclassification rate to $1 - o(1)$. Notice that the above argument only considered the competition between the LHS of \cref{eq:suff-misclass-standardized} and one of the $\Zb$'s on the RHS instead of the maximum $\Zb$. It's not hard to see from the definition of $\Zb$ in \cref{def:zb} that the $\Zb$ are jointly Gaussian. For intuition's sake, assuming the $\Zb$ were \emph{independent}, then $\max_\b \Zb$ would outcompete with probability $(\frac{1}{2} - o(1))^{k-1}$. 

In reality, the $\Zb$ are correlated, but we are able to show that the maximum correlation between the $\Zb$ is $\frac{1}{2} + o(1)$ with high probability. An application of Slepian's lemma (\citet{Slepian1962TheOB}) and some explicit bounds on orthant probabilities (\citet{pinasco2021orthant}) implies that $\max_\b \Zb > 0$ with probability at least $1 - \frac{1}{k^{1+o(1)}}$. Another application of anticoncentration implies that $\max_\b \Zb > n^{-w}$ with probability $1 - o(1)$, which finishes off the proof.

\subsection{Bounding the survival-to-contamination ratio}\label{sec:su-cn-overview}
Note that the critical \emph{survival-to-contamination} ratio appears in both \cref{eq:suff-class-standardized,eq:suff-misclass-standardized}. The most involved part of the proof is nailing down the correct order of growth of the survival to contamination ratio;  a similar analysis tightly bounds the survival variation and the correlation structure of the $\Zb$. 

To understand the relative survival and contamination, we must analyze the bilinear forms $\hhat_{\a, \b}[j] = \vz_j^\top \Ainv\dely$. Similarly, to control the correlation of the $\Zb$, we must understand the correlation between the $\hhat_{\a, \b}$ vectors, which reduces to understanding the bilinear forms $\vz_j^\top \Ainv \ya$ for $j \in [d], \a \in [k]$. The main source of inspiration for bounding these bilinear forms is the heuristic style of calculation carried out in Appendix K of \citet{subramanian2022generalization} that leads to \cref{conjecture:regimes}. 

To simplify the discussion, we temporarily restrict to the regime where regression fails ($q+r>1)$. However, our main technical tool seamlessly generalizes to the regime where regression works ($q+r < 1$). In the regime where regression fails, $\Ainv$ turns out to have a \emph{flat} spectrum: $\Ainv \approx \rho \mI$ for some constant $\rho > 0$. Assume for now that $\Ainv$ is \emph{exactly} equal to a scaled identity matrix. Then the survival is proportional to $\vz_\a^\top\dely$, which is a random inner product. Similarly, to bound the contamination terms we must control the random inner product $\vz_j^\top \dely$ for $j \not\in \qty{\a, \b}$. 

Since $\dely$ is a sparse vector --- it only has $\frac{2n}{k}$ nonzero entries in expectation --- a quick computation reveals that $\EE [\vz_\a^\top \dely] = \Tilde{O}(\frac{n}{k})$ and $\EE[\vz_j^\top \dely] = 0$. The deciding factor, then, is how tightly these quantities concentrate around their means. A na\"{i}ve application of Hoeffding implies a concentration radius of order $\Tilde{O}(\sqrt{n})$, which would lead to looseness in the overall result. The hope is to exploit sparsity to get a concentration radius of order $\Tilde{O}(\sqrt{n/k})$. This is where our new technical tool \cref{thm:bounded-hanson-wright-bilinear} comes in, which may be of independent interest; we present it in the following section.

\subsection{A new variant of the Hanson-Wright inequality}\label{sec:hanson-wright-exposition}
In reality, even in the regime where regression fails, $\Ainv$ is not actually perfectly flat. Even worse, in the regime where regression works, $\Ainv$ is actually spiked. Thus, we cannot simply reduce the bilinear form $\vz_j^\top \Ainv \dely$ to an inner product. Instead, we turn to the well-known Hanson-Wright inequality \citep{rudelson2013hanson}, which tells us that quadratic forms of random vectors with independent, mean zero, subgaussian entries concentrate around their mean. It was used extensively to study binary classification \citep{binary:Muth20}, and multiclass classification \citep{subramanian2022generalization}.

However, just as the classic Hoeffding inequality is loose, so too is the standard form of Hanson-Wright, because it also does not exploit sparsity. We restate the classic version of Hanson-Wright below. 

Define the subgaussian norm $\|\xi\|_{\psi_2}$ \citep{vershynin2018high} as
\begin{align}
    \norm{\xi}_{\psi_2} &= \inf_{K > 0} \qty{K: \EE\exp(\xi^2 / K^2) \le 2},\label{eq:subgaussiannorm}
\end{align}
\begin{theorem}[Hanson-Wright for quadratic forms, from \cite{rudelson2013hanson}]
Let $\vx = (X_1, \ldots, X_n) \in \RR^n$ be a random vector composed of independent random variables that are zero mean with $\norm{X_i}_{\psi_2} \le K$. 
There exists universal constant $c > 0$ such that for any deterministic $\mM \in \RR^{n \times n}$ and $\epsilon \ge 0$, 
\begin{align*}
    &\Pr\left[|\vx^\top \mM \vx - \EE[\vx^\top \mM \vx ]| > \epsilon\right] \le 2\exp(-c\min\qty{\frac{\epsilon^2}{K^4\norm{\mM}_F^2}, \frac{\epsilon}{K^2\norm{\mM}_2}}).\label{eq:hansonwright-alt}
\end{align*}
\end{theorem}

The original Hanson-Wright inequality quoted above only applies to quadratic forms, and moreover assumes that the matrix $\mM$ is deterministic. In our setting, we have a bilinear form $\vz_j^\top \Ainv \dely$, where $\Ainv$ is random. One can condition on the realization of $\Ainv$, but this removes independence and alters the distributions of the random variables involved. Of course, if we condition on a random matrix which is independent of the random vectors, then there is no issue.

Assuming a way around the independence issue, one could decompose the bilinear form with the identity 
\begin{align}
    4\vz_j^\top \Ainv \dely &= (\vz_j+ \dely)^\top\Ainv (\vz_j+ \dely) - (\vz_j - \dely)^\top \Ainv (\vz_j - \dely).
\end{align} 
This trick is used in both \cite{binary:Muth20,subramanian2022generalization}. In the binary classification case, one regains complete independence by using a leave-one-out trick. More precisely, define the leave-one-out matrix $\mA_{-j} = \sum_{i \neq j} \lambda_i \vz_i \vz_i^\top$ and let $\vy \in \qty{\pm 1}^n$ be the binary label vector. Then $\mA_{-j}$ is evidently independent of $\vz_j$ and $\vy$, and the Sherman-Morrison formula implies that $\vz_j^\top \Ainv \vy = \frac{\vz_j^\top \Ajinv \vy}{1+\vz_j^\top \Ajinv \vz_j}$. Because the denominator is a scalar which concentrates well due to Hanson-Wright, this allows for a completely tight characterization of binary classification.

However, this trick does not immediately work in the multiclass setting, because the labels depend on all of the $k>1$ label-defining features. Here, one needs to potentially remove $\omega(1)$ features from the Gram matrix $\mA$ to regain independence. In \cite{subramanian2022generalization}, they eschew Sherman-Morrison entirely and directly exploit the fact that $\Ainv$ is flat (as $q+r>1$). More precisely, they decompose $\Ainv = \rho \mI_n + \mE$, where $\norm{\mE}_2 \ll \rho$. This essentially shoves all the dependencies into $\mE$. While the $\rho \mI_n$ portion reduces to the inner product calculation discussed in \cref{sec:su-cn-overview}, they must use Cauchy-Schwarz to handle the dependent $\mE$ portion. This leads to inevitable looseness in the regimes for Theorem~\ref{thm:positive}, as Cauchy-Schwarz is a \emph{worst-case} bound. Interestingly, their Cauchy-Schwarz bound leverages the sparsity of the label vectors to gain a factor of $\sqrt{k}$, but the bound is still loose by a factor of $\sqrt{n}$.

This motivates a new variant of Hanson-Wright which fully leverages the (soft) sparsity inherent to multiclass problems with an increasing number of classes. We first formally define the notions of soft and hard sparsity.
\begin{definition}[Soft and hard sparsity]
For $\pi \le 1$, we say that random vector $\vy = (Y_i)_{i=1}^n$ has soft sparsity at level $\pi$ if $\abs{Y_i} \le 1$ almost surely and $\mathrm{Var}(Y_i) \le \pi$ for all $i$. On the other hand, we say that $\vy$ has hard sparsity at level $\pi$ if at most a $\pi$ fraction of the $Y_i$ are nonzero. 
\end{definition}
In particular, our variant \cref{thm:bounded-hanson-wright-bilinear} below requires that one of the vectors in the bilinear form has \emph{soft sparsity} at level $\pi$. Throughout, one should think of $\pi = o(1)$, and for us indeed $\pi = O(\frac{1}{k})$. One can check that a bounded random vector $\vy$ with \emph{hard sparsity} level $\pi$ must also have soft sparsity at level $O(\pi)$, so soft sparsity is more general for bounded random vectors. In \cref{table:variants} we compare our variant with several variants of Hanson-Wright which have appeared in the literature, some of which involve hard sparsity.
    
\begin{restatable}[Hanson-Wright for bilinear forms with soft sparsity]{theorem}{hansonwright}\label{thm:bounded-hanson-wright-bilinear}
Let $\vx = (X_1, \ldots, X_n) \in \RR^n$ and $\vy \in (Y_1, \ldots, Y_n) \in \RR^n$ be random vectors such that $(X_i, Y_i)$ are independent pairs of (possibly correlated) centered random 
variables such that $\norm{X_i}_{\psi_2} \le K$ and $Y_i$ has soft sparsity at level $\pi$, i.e. $\abs{Y_i} \le 1$ almost surely, and $\EE[Y_i^2] \le \pi$. Assume that \emph{conditioned on $Y_j$}, $\norm{X_j}_{\psi_2} \le K$. Then there exists an absolute constant $c > 0$ such that for all $\mM \in \RR^{n \times n}$ and $\epsilon \ge 0$ we have 
    \begin{align}
    &\Pr\qty[|\vx^\top \mM \vy - \EE[\vx^\top \mM \vy ]| > \epsilon] \le 2\exp(-c\min\qty{\frac{\epsilon^2}{K^2\pi\norm{\mM}_F^2}, \frac{\epsilon}{K\norm{\mM}_2}}).\label{eq:hansonwright-sparse-bilinear}
\end{align}
\end{restatable}

\begin{table}[]
\begin{tabular}{@{}llll@{}}
\toprule
Variant               & Assumptions on $\vy$                                     & Concentration radius \\ \midrule
Classic quadratic\textsuperscript{\hyperref[ref1]{a}}: $\vx^\top \mM \vx$          & same as $\vx$                                                   & $\Tilde{O}(\norm{\mM}_F)$        \\
Sparse bilinear\textsuperscript{\hyperref[ref2]{b}}: $\vx^\top \mM (\vgamma \circ \vy)$ & $\gamma_i \sim \Ber(\pi)$, indep. of $X_i$ but not $Y_i$
& $\Tilde{O}(\sqrt{\pi}\norm{\mM}_F)$    \\
\cref{thm:sparse-hanson-wright-bilinear-hard}: $\vx^\top \mM (\vgamma \circ \vy)$       & $\gamma_i \sim \Ber(\pi)$, indep. of $Y_i$ but not $X_i$ & $\Tilde{O}(\sqrt{\pi}\norm{\mM}_F)$    \\
\cref{thm:bounded-hanson-wright-bilinear}: $\vx^\top \mM \vy$ & $\abs{Y_i} \le 1$ a.s., $\EE Y_i^2 \le \pi$  & $\Tilde{O}(\sqrt{\pi}\norm{\mM}_F)$    \\ \bottomrule
\end{tabular}
\vspace{1ex}
\caption{Comparison of different variants of the Hanson-Wright inequality. In all variants, we assume that $(\vx, \vy) = (X_i, Y_i)_{i=1}^n$ are subgaussian, centered, and the pairs $(X_i, Y_i)$ are independent across $i$. We use $\circ$ to denote elementwise multiplication, which allows us to express hard sparsity with the sparsity mask $\vgamma \in \qty{0, 1}^n$. The concentration radius corresponds to the size of typical fluctuations guaranteed by the concentration inequality, i.e. the $\epsilon$ needed for high probability guarantees. \\
\small\textsuperscript{a} \citep[Theorem~1.1]{rudelson2013hanson}\label{ref1};  \textsuperscript{b} \citep[Theorem~1]{park2022sparse}\label{ref2}}\label{table:variants}
\end{table}

The full proof of \cref{thm:bounded-hanson-wright-bilinear} is deferred to \cref{sec:hanson-wright-bounded}. The main proof techniques are heavily inspired by those of \cite{rudelson2013hanson,zhou2019sparse,park2022sparse}. However, the proof of \cref{thm:bounded-hanson-wright-bilinear} is actually simpler than in \citet{park2022sparse}, as bounded with soft sparsity turns out to be easier to work with than subgaussian with hard sparsity. To actually apply Hanson-Wright, we carefully isolate the dependent portions using the Woodbury inversion formula, which generalizes the Sherman-Morrison formula for arbitrary rank updates. 

For the sake of completeness, in \cref{sec:hanson-wright-hard-sparsity} we also prove a variant of \cref{thm:bounded-hanson-wright-bilinear} where we assume that $\vy$ has subgaussian entries with hard sparsity. 

\begin{restatable}[Hanson-Wright for bilinear forms with hard sparsity]{theorem}{hansonwrighthard}\label{thm:sparse-hanson-wright-bilinear-hard}
    Let $\vx = (X_1, \ldots, X_n) \in \RR^n$ and $\vy \in (Y_1, \ldots, Y_n) \in \RR^n$ be random vectors such that the pairs $(X_i, Y_i)$ are independent pairs of (possibly correlated) centered random variables with subgaussian norm at most $K$. 
Suppose $\vgamma = (\gamma_1, \ldots, \gamma_n) \in \qty{0, 1}^n$ is an i.i.d. Bernoulli vector with bias $\pi$. Assume that $\vgamma$ is independent of $\vy$, $\gamma_j$ is independent of $X_i$ for $i \neq j$, and finally \emph{conditioned on $\gamma_j = 1$}, $X_j$ has subgaussian norm at most $K$. Then there exists an absolute constant $c > 0$ such that for all $\mM \in \RR^{n \times n}$ and $\epsilon \ge 0$ we have 
    \begin{align*}
    \Pr\qty[|\vx^\top \mM (\vy \circ \vgamma) - \EE[\vx^\top \mM (\vy \circ \vgamma) ]| > \epsilon] \le 2\exp(-c\min\qty{\frac{\epsilon^2}{K^4\pi\norm{\mM}_F^2}, \frac{\epsilon}{K^2\norm{\mM}_2}}).
    \end{align*}
\end{restatable}

We briefly illustrate how \cref{thm:bounded-hanson-wright-bilinear} can be used to get tighter results throughout our analysis. A quick calculation reveals that the label vectors $\dely$ and $\ya$ both have soft sparsity at level $\pi = O(1/k)$. However, $\ya$ does not have hard sparsity as required by the variants in \citet{park2022sparse} or \cref{thm:sparse-hanson-wright-bilinear-hard}. Since $\norm{\mM}_F^2 \le n \norm{\mM}_2^2$, we obtain a concentration radius $\epsilon$ which scales like $\sqrt{n/k}$ rather than $\sqrt{n}$ (obtained via vanilla Hanson-Wright) or $n/\sqrt{k}$ (obtained via Cauchy-Schwarz). This gain is crucial to tightly analyzing the survival, contamination, and correlation structure. 

\subsection{Completing the proof sketch}
Theorem~\ref{thm:bounded-hanson-wright-bilinear} and the above insights about sparsity and independence allow us to prove the following bounds on the relative survival and contamination terms which are tight up to log factors; see \cref{app:proof-sketch} for more details. For brevity's sake, we introduce the notation $\mu \triangleq n^{q+r-1}$.
\begin{restatable*}[Bounds on relative survival]{proposition}{survivalacc}\label{prop:relative-survival-bound-accurate}
    Under the bi-level ensemble model (\cref{def:bilevel}), when the true data generating process is 1-sparse (Assumption~\ref{assumption:1sparse}), if $t < \frac{1}{2}$, then with probability at least $1-O(1/nk)$
    \begin{align*}
        \lambda_F \hhat_{\a, \b}[\a] = \const{const_alpha}\min\qty{\mu^{-1}, 1} n^{-t}(1 \pm O(n^{-\konst{kappa_survival}}))\sqrt{\log k},
    \end{align*}
    where \const{const_alpha} and \konst{kappa_survival} are positive constants.

    If $t \ge \frac{1}{2}$, then 
    \begin{align*}
        \lambda_F \abs{\hhat_{\a, \b}[\a]} \le \const{const_rel_survival_upper}\min\qty{\mu^{-1}, 1} n^{-\frac{1}{2}}\sqrt{\log (nk)},
    \end{align*}
    where \const{const_rel_survival_upper} is a positive constant.
\end{restatable*}
Next, we state our bounds on contamination. 
\begin{restatable*}[Bounds on contamination]{proposition}{contaminationacc}\label{prop:contamination-bound-accurate}
    Under the bi-level ensemble model (\cref{def:bilevel}), when the true data generating process is 1-sparse (Assumption~\ref{assumption:1sparse}), with probability at least $1 - O(1/nk)$,  
    \begin{align*}
        \cn_{\a, \b} \le \underbrace{\min\qty{\mu^{-1}, 1}O(n^{\frac{r-t-1}{2}})\log(nsk)}_{\text{favored features}} + \underbrace{O(n^{\frac{1-t-p}{2}})\sqrt{\log(nsk)}}_{\text{unfavored features}}.
    \end{align*}

    Furthermore, if $t > 0$, then with probability at least $1-O(1/nk)$,
    \begin{align*}
        \cn_{\a, \b} \ge \underbrace{\min\qty{\mu^{-1}, 1}\Omega(n^{\frac{r-t-1}{2}})}_{\text{favored features}} + \underbrace{\Omega(n^{\frac{1-t-p}{2}})}_{\text{unfavored features}}.
    \end{align*}
\end{restatable*}

Translating the parameters above, we see that (i) the relative survival is diminished by a factor $1/k$ as long as $k = o(\sqrt{n})$, and a factor $1/\sqrt{n}$ for $k = \Omega(\sqrt{n})$ (this looseness ends up being negligible for the final result) and (ii) the contamination is diminished by a factor of $1/\sqrt{k}$. This essentially matches the expected behavior from the heuristic calculation in \citet{subramanian2022generalization}. With some straightforward algebra, we can compute the regimes where the survival-to-contamination ratio $\su/\cn$ grows or decays polynomially. This yields the stated regimes in \cref{conjecture:regimes}; see \cref{app:proof-sketch} for more details.

For technical reasons, the analogous bounds in \citet{subramanian2022generalization} are loose, giving rise to unnecessary conditions for good generalization such as $t < p-2$ and $t < 2q+r-2$. Moreover, we are able to give both upper and lower bounds on the survival and contamination terms, whereas they only give one-sided inequalities for each quantity. 
\section{Discussion}\label{sec:discussion}

In this paper we resolve the main conjecture of \citet{subramanian2022generalization}, identifying the exact regimes where an overparameterized linear model succeeds at multiclass classification. Our techniques also lay the foundation for investigating related generalization for other multiclass tasks and nonlinear algorithms. We hope that by bringing the rigorous proofs closer to the heuristic style of calculation, we open the path for analyzing more complicated and realistic models.

An important next step is to extend our results to more realistic spiked covariance models. For example, one typically observes power-law decay for the extreme eigenvalues in applications. We expect that the bi-level model can be relaxed to allow for constant deviations in the weightings for the favored, non-label defining features and power law decay for the unfavored features. The former change would likely only affect constants in certain areas of the argument that do not crucially depend on the exact constants involved, whereas the latter would likely just change the effective degree of overparameterization \citep{reg:bartlett2020benign}. However, even constant fluctuations in weighting for the label-defining features can lead to significant subtleties, as these constant deviations manifest as polynomially large variations in the number of examples of each class. Such heterogeneity between label-defining directions would likely lead to significantly messier conditions for generalization.

Another future direction is to move beyond Gaussian features. It is plausible that similar results would hold for vector subgaussian features which are rotationally invariant, allowing us to rotate into the basis where the features have diagonal covariance. One place where Gaussianity is crucially used is to obtain an explicit lower bound on the margin between the features. 

As an example application, we sketch out how our proof techniques imply precise conditions for a variant of the learning task called multilabel classification. In a simple model for multilabel classification, each datapoint can have several of $k$ possible labels --- corresponding to the positive valued features --- but in the training set only one such correct label is provided at random for each datapoint. We deem that the model generalizes if for any queried label it successfully labels test inputs as positive or negative. We can use the MNI approach here to learn classifiers.

Some thought reveals that the main difference between multilabel classification and multiclass classification from a survival and contamination perspective is that positive features no longer need to outcompete other features. Thus, the main object of study would be the bilinear forms $\zj^\top \Ainv \ya$, which is possible thanks to \cref{thm:bounded-hanson-wright-bilinear}. The survival and contamination terms are only affected by the expected values of these bilinear forms, but the expected values match the multiclass behavior up to log factors, which do not affect the regimes where $\su/\cn$ will grow or shrink polynomially. A similar analysis thus reveals that MNI will generalize in exactly the same regimes as in \cref{conjecture:regimes}. Here, the model  generalizes in the sense that with high probability over the labels the model will correctly classify, and failure to generalize means that the model will do no better than a coin toss. 

Perhaps surprisingly, resolving \cref{conjecture:regimes} also implies that MNI is asymptotically \emph{suboptimal} compared to a natural \emph{non-interpolative} approach: simply make $\fhat_m$ equal to the average\footnote{Note that \citet{frei2023benign} point out that even leaky ReLU networks trained with a gradient flow can behave like averages of training examples.} of all positive training examples of class $m$. A straightforward analysis, detailed in the supplementary material, reveals this scheme fails to generalize exactly when $t < \min\qty{1-r, p+1-2(q+r)}$, even in the regime where regression succeeds ($q+r<1$). This is particularly interesting because we have shown that in the regime where regression succeeds, MNI generalizes only when $t < \min\qty{1-r, p-1}$, which is a smaller region. In light of this gap, it would be interesting to identify the \emph{information-theoretic} barrier for multiclass classification, especially within the broader context of statistical-computation gaps (see e.g. \citet{wu2021statistical,brennan2020reducibility}).

\section*{Acknowledgments and Disclosure of Funding}
DW acknowledges support from NSF Graduate Research Fellowship DGE-2146752. We acknowledge funding support from NSF AST-2132700 for this work. DW appreciates helpful discussions with Sidhanth Mohanty and Prasad Raghavendra. 

\bibliography{references}
\bibliographystyle{plainnat}

\clearpage
\appendix
\tableofcontents
\clearpage
\section{Preliminaries and notation}\label{app:setup}
For positive integers $n$, we use the shorthand $[n] \triangleq \qty{1, \ldots, n}$. For a vector $\vv \in \RR^n$, $\norm{\vv}_2$ always denotes the Euclidean norm. We index entries by using square brackets, so $\vv[j]$ denotes the $j$th entry of $\vv$. For any matrix $\mM \in \RR^{m \times n}$, we denote its $ij$th entry by $m_{ij}$, $\norm{\mM}_2$ denotes the spectral norm, and $\norm{\mM}_F = \Tr(\mM^\top \mM)$ denotes the Frobenius norm. We use $\sigma_{\max{}}(\mM)$ and $\sigma_{\min{}}(\mM)$ to denote the maximum and minimum singular values of $\mM$, respectively. If $\mM \in \RR^{n \times n}$ is symmetric, we write $\mu_1(\mM) \ge \mu_2(\mM) \ge \ldots \ge \mu_n(\mM)$\label{def:eigs} to denote the ordered eigenvalues of $\mM$. Given two vectors $\vv, \vu \in \RR^{n}$, we write $\vv \circ \vu \in \RR^{n}$ to denote the entrywise product of $\vv$ and $\vu$. 

We make extensive use of big-$O$ notation. In this paragraph, $c$ refers to a positive constant which does not depend on $n$, and all statements hold for sufficiently large $n$. If $f(n) = O(g(n))$,  then $f(n) \le cg(n)$ for some $c$. If $f(n) = \Tilde{O}(g(n))$, then $f(n) \le cg(n)\poly\log(n)$ for some $c$. If $f(n) = o(g(n))$, then for all $c > 0$ we have $f(n) \le c g(n)$. We write $f(n) = \Omega(g(n))$ if $f(n) \ge cg(n)$ for some $c$. Finally, we write $f(n) = \Theta(g(n))$ if there exists positive constants $c_1$ and $c_2$ such that $c_1g(n) \le f(n) \le c_2g(n)$.
\begin{table}[!ht]
  \caption{Notation}
  \label{tab:notation}
  \centering
  \begin{tabular}{llll}
    \toprule
    Symbol & Definition & Dimension & Source \\
    \midrule
    $\nc$ & Number of classes & Scalar & Sec. \ref{sec:setup} \\
    $n$ & Number of training points & Scalar & Sec. \ref{sec:setup} \\
    $d$ & Dimension of each point --- the total number of features & Scalar & Sec. \ref{sec:setup} \\
    $s$ & The number of favored features & Scalar & Def. \ref{def:bilevel} \\
    $a$ & The constant controlling the favored weights & Scalar & Def. \ref{def:bilevel} \\
    $p$ & Parameter controlling overparameterization ($d = n^p$) & Scalar & Def. \ref{def:bilevel} \\
    $r$ & Parameter controlling the number of favored features ($s = n^r$) & Scalar & Def. \ref{def:bilevel} \\
    $q$ & Parameter controlling the favored weights ($a = n^{-q}$) & Scalar & Def. \ref{def:bilevel} \\
    $t$ & Parameter controlling the number of classes ($\nc = c_k n^t$) & Scalar & Def. \ref{def:bilevel} \\
    $c_k$ & The number of classes when $t=0$ ($\nc = c_k n^t$) & Scalar & Def. \ref{def:bilevel} \\
    $\lambda_j$ & Squared weight of the $j$th feature & Scalar & Def. \ref{def:bilevel} \\
    $\vx_i$ & $i$th training point (unweighted) & Length-$d$ vector & Sec. \ref{sec:setup} \\
    $\ell_i$ & Class label of $i$th training point & Scalar & Eqn. \ref{eq:truelabels} \\
    $\vw_i$ & $i$th training point (weighted) & Length-$d$ vector & Eqn. \ref{eq:lambdas} \\
    $\mX^w$ & Weighted feature matrix & $(n \times d)$-matrix & Eqn. \ref{eq:weightedxmatrix} \\
    $\zj$ & The collected $j$th features of all training points & Length-$n$ vector & Eqn. \ref{eq:weightedxmatrix} \\
    $\yoh_m$ & One-hot encoding of all the training points for label $m$ & Length-$n$ vector & Eqn.
    \ref{eq:yoh} \\
    $\Yoh$ & One-hot label matrix & ($n \times \nc$)-matrix & Eqn. \ref{eq:yoh} \\
    $\vy_m$ & Zero-mean encoding of the training points for label $m$ & Length-$n$ vector & Eqn.     \ref{eq:zeromeany} \\
    $\hat{\vf}_m$ & Learned coefficients for label $m$ using min-norm interpolation &  Length-$d$ vector & Eqn. \ref{eq:closed-form-coefficients} \\
    $\xtest$ & A single test point & Length-$d$ vector & Sec. \hyperref[def:xtest]{2} \\
     $\xtest^w$ & A single weighted test point & Length-$d$ vector & Sec. \hyperref[def:xwtest]{2} \\
    $\mA$ & Gram matrix $\mA = \mX^w(\mX^w)^\top$ & ($n \times n$)-matrix & Sec. \ref{sec:proof-sketch} \\
    $\mu_i(\mA)$ & The $i$th eigenvalue of matrix $\mA$, sorted in descending order & Scalar & App.  \ref{def:eigs}\\
    $\lambda_F$ & Squared favored feature weights: $\lambda_F = \frac{ad}{s}$ & Scalar & Def. \ref{def:bilevel} \\
    $\lambda_U$ & Squared unfavored feature weights: $\lambda_F = \frac{(1-a)d}{d-s}$ & Scalar & Def. \ref{def:bilevel} \\
    $\hhat_{\a,\b}$ & Relative survival $\hhat_{\a,\b}[j] = \lambda_j^{-1/2}(\hat{f}_\a[j] - \hat{f}_\b[j])$ & Length-$d$ vector & Eqn. \ref{eq:hhat_def} \\
    $\cn_{\a,\b}$ & 
    Normalizing factor $\cn_{\a,\b}= \sqrt{\left(\sum_{j \notin \{\a, \b\}} \lambda_j^2 (\hhat_{\b,\a}[j])^2\right)}$ & Scalar & Eqn. \ref{eq:cn} \\
    $\norm{\cdot}_{\psi_2}$ & 
    The subgaussian norm of a scalar random variable & Scalar & Eqn. \ref{eq:subgaussiannorm} \\
    $\norm{\cdot}_{\psi_1}$ & 
    The subexponential norm of a scalar random variable & Scalar & Eqn. \ref{eq:subexponentialnorm} \\
   $\mu$ & Factor controlling whether regression works, $\mu \triangleq n^{q+r-1} $& Scalar & App. \ref{eq:mu} \\
    $\Zt$ & Unweighted subset of favored features, where $T \subseteq [s]$ & $(n \times \abs{T})$-matrix & App. \ref{eq:Zt} \\
    $\Xt$ & Weighted subset of favored features, $\Xt = \sqrt{\lambda_F} \Zt$ & $(n \times \abs{T})$-matrix & App. \ref{eq:Xt} \\
    $\mA_{-T}$ & Leave-$T$-out Gram matrix, where $T \subseteq [s]$, $\mA_{-T} = \mA - \Xt \Xt^\top$ & $(n \times n)$-matrix & Eqn. \ref{eq:At} \\
    $\mH_k$ & Hat matrix, $\mH_k = \Xk \Akinv \Xk$ & $(k \times k)$-matrix & Eqn. \ref{eq:hatk} \\
    \bottomrule
  \end{tabular}
\end{table}

Let us now describe the organization of the appendix. In \cref{app:proof-sketch}, we give a more detailed proof sketch and introduce the main propositions that complete the proof of \cref{thm:positive}. In \cref{sec:main-tools}, we introduce the main tools that allow us to prove that the critical bilinear forms $\vz_j^\top \Ainv \dely$ concentrate: our new variant of the Hanson-Wright inequality, the Woodbury inversion formula, and Wishart concentration to bound the spectra of the relevant random matrices that appear. In \cref{sec:utility} we apply these tools to bound some useful quantities that repeatedly appear in the rest of the proofs. After that, we proceed to bound the survival, contamination, and correlation structure in \crefrange{sec:survival}{sec:correlation-bounds}. In \cref{sec:averaging}, we present the analysis for the averaging scheme described in \cref{sec:discussion}. Finally, we prove our new variants of Hanson-Wright (\cref{thm:bounded-hanson-wright-bilinear,thm:sparse-hanson-wright-bilinear-hard}) in \cref{sec:hanson-wright-bounded,sec:hanson-wright-hard-sparsity}, respectively. 

\subsection{Proof of \cref{thm:positive}}\label{app:proof-sketch}
In this section, we fill in some of the details of the proof sketch of \cref{thm:positive}. After recalling the beginning of the proof, we will split up the proof into two subtheorems: one for the positive result where MNI generalizes (\cref{thm:positive-side}), and another for the negative result where MNI misclassifies (\cref{thm:negative-side}). 

Assume without loss of generality that the test point $\xtest \sim N(0, \mI_d)$ has true label $\alpha$ for some $\alpha \in [k]$. Let $\xwtest$ be the weighted version of this test point. 
From \eqref{eq:misclassification-definition}, an equivalent condition for misclassification is that for some $\beta \neq \alpha, \beta \in [k]$, we have $\fhat_\a^\top \xwtest < \fhat_\b^\top \xwtest$, i.e. the score for $\b$ outcompetes the score for $\a$. Define the Gram matrix $\mA \triangleq \mX^w(\mX^w)^\top$, the relative label vector $\dely \triangleq \ya - \yb \in \qty{-1, 0, 1}^n$, and the relative survival vector $\hhat_{\a, \b} \in \RR^d$ which compares the signal from $\a$ and $\b$:
\begin{align}
    \hhat_{\a,\b}[j] &\triangleq \lambda_j^{-1/2}(\fhat_\a[j] - \fhat_\b[j]) \\
    &= \vz_j^\top \Ainv \dely,
\end{align}
where to obtain the last line we have used the explicit formula for the MNI classifiers \eqref{eq:closed-form-coefficients}.
By converting the misclassification condition into the unweighted feature space we see that we will have errors when
\begin{align}
    &\lambda_\a\hhat_{\a,\b}[\a] \xtest[\a] -  \lambda_\b\hhat_{\b,\a}[\b] \xtest[\b] < \sum_{j \notin \{\a, \b \}} \lambda_j \hhat_{\b,\a}[j]\xtest[j]. \label{eq:RHS-new}
\end{align}
Define the contamination term $\cn_{\a, \b}$:
\begin{equation}
    \cn_{\a,\b} \triangleq \sqrt{\sum_{j \notin \{\a, \b\}} \lambda_j^2 (\hhat_{\b,\a}[j])^2}. \label{eq:cn-new}
\end{equation}
Note that $\cn_{\a, \b}$ normalizes the RHS of \eqref{eq:RHS-new} into a standard Gaussian. Indeed, define
\begin{equation}
    \Zb \triangleq \frac{1}{\cn_{\a, \b}} \sum_{j \notin \{\a, \b \}} \lambda_j \hhat_{\b,\a}[j]\xtest[j] \sim N(0, 1). \label{def:zb-new}
\end{equation}

Since $\a, \b \in [k]$ are favored, we have $\lambda_\a = \lambda_\b = \lambda_F$. Hence an equivalent condition for misclassification is that there exists some $\b \neq \a$, $\b \in [k]$ such that 
\begin{align}
    &\frac{\lambda_F}{\cn_{\a, \b}}(\hhat_{\a,\b}[\a] \xtest[\a] -  \hhat_{\b,\a}[\b] \xtest[\b]) < Z^{(\b)}. \label{eq:equiv-misclassification-new}
\end{align}
We will translate the above criterion into \emph{sufficient} conditions for correct classification and misclassification and analyze these two cases separately. 

First, let us present our tight characterization of the survival and contamination terms, which will be useful for both sides of the theorem. Recall our definition of $\mu \triangleq n^{q+r-1}$\label{eq:mu}; whether this quantity polynomially shrinks or decays directly determines if regression works or fails. 
\survivalacc
\contaminationacc
We defer the proof of \cref{prop:relative-survival-bound-accurate} to \cref{sec:survival} and the proof of \cref{prop:contamination-bound-accurate} to \cref{sec:contamination}.
Combining \cref{prop:relative-survival-bound-accurate,prop:contamination-bound-accurate} yields the following sufficient conditions for when the $\su/\cn$ ratio grows or shrinks polynomially.
\begin{proposition}[Regimes for survival-to-contamination]\label{prop:su-cn-bound}
Under the bi-level ensemble model (\cref{def:bilevel}), when the true data generating process is 1-sparse (Assumption~\ref{assumption:1sparse}), as $n\rightarrow \infty$, with probability at least $1 - O(1/n)$, the survival-to-contamination ratio satisfies 
\begin{align}
    \frac{\min_{\b} \lambda_F \hhat_{\a, \b}[\a]}{\max_\b \cn_{\a, \b}} \ge n^v \text{ for some } v > 0 \text{ if }  &t < \min\qty{1-r, p+1-2\max\qty{1, q+r}} \\
    \frac{\max_\b \lambda_F \abs{\hhat_{\a, \b}[\a]}}{\min_\b \cn_{\a, \b}} \le n^{-w} \text{ for some } w > 0 \text{ if }  &t > \min\qty{1-r, p+1-2\max\qty{1, q+r}} 
\end{align}
Here, the max and min are being taken over $\b \neq \a, \b \in [k]$.
\end{proposition}
\begin{proof}
We do casework on whether we want to prove an upper bound or lower bound on $\su/\cn$. First, suppose we want to prove the lower bound, so assume $t < \min\qty{1-r, p+1 - 2\max\qty{1, q+r}}$. Since $t < r$ by the definition of the bi-level ensemble (\cref{def:bilevel}), we have that $t < \frac{1}{2}$. So by union bounding over $\b$, \cref{prop:relative-survival-bound-accurate} implies that with probability $1 - O(1/n)$
\begin{align}
    \min_\b \lambda_F \hhat_{\a, \b}[\a] \ge \min\qty{\mu^{-1}, 1}\Omega(n^{-t})\sqrt{\log k}.
\end{align}
 Then from \cref{prop:contamination-bound-accurate}, by union bounding over $\b$ we see that with probability $1 - O(1/n)$,
\[
    \max_{\b} \cn_{\a, \b} \le \underbrace{\min\qty{\mu^{-1}, 1}\Tilde{O}(n^{\frac{r-t-1}{2}})}_{\text{favored features}} + \underbrace{\Tilde{O}(n^{\frac{1-t-p}{2}})}_{\text{unfavored features}}.
\]
Let us combine these two bounds. If we compare the survival to the contamination coming from favored features, we obtain 
\begin{align}
    \frac{\min\qty{\mu^{-1}, 1}n^{-t}\sqrt{\log k}}{\min\qty{\mu^{-1}, 1}\Tilde{O}(n^{\frac{r-t-1}{2}})} &\ge \frac{n^{-t - \frac{r-t-1}{2}}}{\log(nsk)} \\
    &\ge \frac{n^{\frac{1-r-t}{2}}}{\log(nsk)},
\end{align}
where we have included the explicit $\poly\log$ factors for precision. Hence, if $t < 1-r$, the numerator grows polynomially and dominates the denominator. Now let's compare the survival to the contamination coming from unfavored features. This yields
\begin{align}
    \frac{\min\qty{\mu^{-1}, 1}n^{-t}\sqrt{\log k}}{\Tilde{O}(n^{\frac{1-t-p}{2}})} &\ge \frac{\min\qty{\mu^{-1}, 1} n^{-t - \frac{1-t-p}{2}}}{\log (nsk)} \\
    &\ge \frac{n^{-\max\qty{q+r-1, 0}} \cdot n^{\frac{p - t - 1}{2}}}{\log(nsk)}\\
    &\ge \frac{n^{\frac{p+1 - 2\max\qty{1, q+r} - t}{2}}}{\log(nsk)}.
\end{align}
Hence, by union bounding, we see that with probability $1 - O(1/n)$,
\begin{align}
    \min_\b \frac{\lambda_F \hhat_{\a, \b}[\a]}{\cn_{\a, \b}} \ge n^v,
\end{align}
where $v \triangleq \frac{1}{4}\qty(\min\qty{1-r, p+1 - 2\max\qty{1, q+r}} - t )> 0$ by assumption. 

For the upper bound, suppose $t > \min\qty{1-r, p+1 - 2\max\qty{1, q+r}}$. Hence $t > 0$, and by union bounding we conclude that with probability at least $1 - O(1/n)$, 
\begin{align}
    \max_\b \lambda_F \abs{\hhat_{\a, \b}[\a]} \le \min\qty{\mu^{-1}, 1} O(n^{-\frac{1}{2}})\sqrt{\log k}
\end{align}
and 
\begin{align}
    \min_\b \cn_{\a, \b} &\ge \min\qty{\mu^{-1}, 1} \Omega(n^{\frac{r-t-1}{2}}) + \Omega(n^{\frac{1-t-p}{2}}).
\end{align}
Combining these and union bounding yields that with probability $1 - O(1/n)$,
\begin{align}
    \min_\b \frac{\lambda_F \hhat_{\a, \b}[\a]}{\cn_{\a, \b}} \le n^{-w},
\end{align}
where $w \triangleq \frac{1}{4}\qty(t - \min\qty{1-r, p+1 - 2\max\qty{1, q+r}})> 0$ by asssumption. 
\end{proof}

We now sketch out a proof of both the positive and negative sides of \cref{thm:positive}. We point out that the regimes for generalization and misclassification exactly match the regimes above for where the $\su/\cn$ ratio grows or shrinks polynomially.
\begin{theorem}[Positive side of \cref{thm:positive}]\label{thm:positive-side}
    Under the bi-level ensemble model (\cref{def:bilevel}), when the true data generating process is 1-sparse (Assumption~\ref{assumption:1sparse}), as $n\rightarrow \infty$, the probability of misclassification for MNI satisfies $\Pr[\gE_{\err}] \to 0$ if 
    \begin{align*}
        t < \min\qty{1-r, p+1-2\max\qty{1, q+r}}.
    \end{align*}
\end{theorem}
\begin{proof}[Proof sketch]
    
For correct classification, it suffices for the minimum value of the LHS of \cref{eq:equiv-misclassification-new} to outcompete the maximum value of the RHS, where the max is taken over $\b \in [k], \b \neq \a$. Some algebra, as in \citet{subramanian2022generalization}, shows that we correctly classify if
\begin{align}
    \underbrace{\frac{\min_\beta \lambda_F \hhat_{\a,\b}[\a]}{\max_\b \cn_{\a,\b}}}_{\su/\cn\text{ ratio}} \Bigg( \underbrace{\min_\b \left(\xtest[\a] - \xtest[\b]\right)}_{\text{closest feature margin}} - \underbrace{\max_\b |\xtest[\b]|}_{\text{largest competing feature}} &\cdot \underbrace{\max_\b \left|\frac{\hhat_{\a,\b}[\a] - \hhat_{\b,\a}[\b]}{\hhat_{\a,\b}[\a]} \right|}_{\text{survival variation}}\Bigg) \nonumber \\
   &> \underbrace{\max_{\beta} Z^{(\b)}}_{\text{normalized contamination}}. \label{eq:suff-class-standardized-new}
\end{align}

By our lower bound on the survival to contamination ratio (\cref{prop:su-cn-bound}), assuming $t < \min\qty{1-r, p+1-2(q+r)}$, then with probability at least $1-O(1/n)$ we have that $\frac{\lambda_F \hhat_{\a,\b}[\a]}{ \cn_{\a,\b}} \ge n^{u}$ for some constant $u > 0$. By Lemmas B.2 and B.3 in \citet{subramanian2022generalization} for every $\epsilon > 0$,  with probability at least $1-\epsilon$, we have $\min_\beta \xtest[\a] - \xtest[\b] \ge \Omega(\frac{1}{\sqrt{\log k}})$. 

Next, by standard subgaussian maxima tail bounds we have that $\abs{\xtest[\b]} \le 2\sqrt{\log(nk)}$ and $Z^{(\b)} \le 2\sqrt{\log(nk)}$ with probability at least $1-O(1/nk)$. Finally, applying our upper bound on the relative survival variance (\cref{prop:survival-variance-upper-bound}, which we prove below), the survival variation is at most a  polynomially decaying $n^{-w}$ with probability at least $1 - O(1/nk)$.

By union bounding, we see that with probability at least $1 - O(1/n) - \epsilon$, the LHS outcompetes the RHS, implying that the model correctly classifies.

\end{proof}

In fact, given \cref{prop:relative-survival-bound-accurate}, it is straightforward to bound the survival variation.
\begin{proposition}[Upper bound on the survival variation]\label{prop:survival-variance-upper-bound}
Suppose that $t < 1-r$. With probability at least $1-2/n$, we have 
\begin{equation}
    \abs{\frac{\hhat_{\a,\b}[\a] - \hhat_{\b,\a}[\b]}
{\hhat_{\a,\b}[\a]}} \le \const{const_variance} n^{-w},
\end{equation}
where \const{const_variance} and $w$ are both positive constants. 
\end{proposition}
\begin{proof}
    Since we have $\hhat_{\a, \b}[\a] = \vz_\alpha^\top \Ainv \dely$, the survival variation is 
\[
\frac{\hhat_{\a,\b}[\a] - \hhat_{\b,\a}[\b]}{\hhat_{\a,\b}[\a]} = \frac{\vz_\alpha^\top \Ainv \dely + \vz_\beta^\top \Ainv \dely}{\vz_\alpha^\top \Ainv \dely} 
\]
Since $t < 1-r$ and $t < r$ by definition, we know that $t < \frac{1}{2}$,and we can apply \cref{prop:relative-survival-bound-accurate} to see that with probability at least $1-2/n$ we have
\begin{align*}
    \vz_\a^\top \Ainv \dely &= \max\qty{\mu^{-1}, 1} n^{-t}(1 \pm O(n^{-\konst{kappa_survival}}))\sqrt{\log k} = -\zb^\top \Ainv \dely
\end{align*}

Hence we have 
\begin{equation}
    \abs{\frac{\hhat_{\a,\b}[\a] - \hhat_{\b,\a}[\b]}
{\hhat_{\a,\b}[\a]}} \le \const{const_variance} n^{-\konst{kappa_survival}}
\end{equation}
where \newc\label{const_variance} is an appropriately defined positive constant.
\end{proof}

\begin{theorem}[Negative side of \cref{thm:positive}]\label{thm:negative-side}
     Under the bi-level ensemble model (\cref{def:bilevel}), when the true data generating process is 1-sparse (Assumption~\ref{assumption:1sparse}), as $n\rightarrow \infty$, the probability of misclassification for MNI satisfies $\Pr[\gE_{\err}] \to 1$ if 
    \begin{align*}
        t > \min\qty{1-r, p+1-2\max\qty{1, q+r}}.
    \end{align*}
\end{theorem}
\begin{proof}[Proof sketch]
    On the other hand, for misclassification it suffices for the maximum \emph{absolute} value of the LHS of \cref{eq:equiv-misclassification-new} to be outcompeted by the maximum value of the RHS. Some manipulations yield the following sufficient condition for misclassification: 
\begin{align}
        \underbrace{\frac{\max_\b \lambda_F \qty(\abs{\hhat_{\a,\b}[\a]} + \abs{\hhat_{\b,\a}[\b]})}{\min_\b \cn_{\a,\b}}}_{\su/\cn \text{ ratio}} \cdot  \underbrace{\max_{\gamma \in [k]} \abs{\xtest[\gamma]}}_{\text{largest label-defining feature}} 
   < \underbrace{\max_{\beta} Z^{(\b)}}_{\text{normalized contamination}}. \label{eq:new-suff-cond-misclassification}
\end{align}
Within the misclassification regimes in \cref{conjecture:regimes}, \cref{prop:su-cn-bound} implies that the survival-to-contamination ratio $\su/\cn$ \emph{shrinks} at a polynomial rate $n^{-w}$ for some $w > 0$. By standard subgaussian maximal inequalities, the largest label-defining feature is $O(\sqrt{\log(nk)})$ with high probability. Gaussian anticoncentration implies that for some $\b \neq \a, \b \in [k]$, $\Zb$ outcompetes the LHS, which is bounded above by $n^{-w}$, with probability at least $\frac{1}{2} - o(1)$. 
Hence, we conclude that the model will misclassify with rate at least $\frac{1}{2}$ asymptotically.

Let us now describe how to boost the misclassification rate to $1 - o(1)$. Notice that the above argument only considered the competition between the LHS of \cref{eq:new-suff-cond-misclassification} and one of the $\Zb$'s on the RHS instead of the maximum $\Zb$. It's not hard to see from the definition of $\Zb$ in \cref{def:zb-new} that the $\Zb$ are jointly Gaussian. For intuition's sake, assuming the $\Zb$ were \emph{independent}, then $\max_\b \Zb$ would outcompete with probability $(\frac{1}{2} - o(1))^{k-1}$. 

In reality, the $\Zb$ are correlated, but we are able to show that the maximum correlation between the $\Zb$ is $\frac{1}{2} + o(1)$ with high probability. An application of Slepian's lemma (\citet{Slepian1962TheOB}) and some explicit bounds on orthant probabilities (\citet{pinasco2021orthant}) implies that $\max_\b \Zb > 0$ with probability at least $1 - \frac{1}{k^{1+o(1)}}$. An application of anticoncentration for Gaussian maxima \citep{chernozhukov2015comparison} implies that $\max_\b \Zb > n^{-w}$ with probability $1 - o(1)$, which finishes off the proof.
\end{proof}
To fill in the details of the above proof sketch, we will prove the following proposition in \cref{sec:correlation-bounds}. 
\begin{restatable}[Correlation bound]{proposition}{correlation}\label{prop:correlation-bound-tight}
Assume we are in the bi-level ensemble model (\cref{def:bilevel}), the true data generating process is 1-sparse (Assumption~\ref{assumption:1sparse}), and the number of classes scales with $n$ (i.e. $t > 0$). Then for every $\epsilon > 0$, we have 
\begin{align}
    \Pr[\max_{\b \in [k], \b \neq \a} Z^{(\b)} > n^{-u}] \ge 1 - \Theta\qty(\frac{1}{k^{1+o(1)}}) - \epsilon
\end{align}
for sufficiently large $n$ and any $u > 0$.
\end{restatable}

\section{Main tools}\label{sec:main-tools}
In this section we introduce our suite of technical tools that allow us to prove the desired rates of growth for survival, contamination, and correlation. 
\subsection{Hanson-Wright Inequality}\label{sec:hanson-wright-formal}
As established in \cref{sec:proof-sketch}, we need to use the Hanson-Wright inequality to prove our tight characterization of generalization. For the sake of precision, we explicitly state our definitions of subgaussian and subexponential which we use throughout the rest of the paper.

The subgaussian norm $\|\xi\|_{\psi_2}$ of a random variable $\xi$ is defined as in \citet{rudelson2013hanson},
\begin{align}
    \norm{\xi}_{\psi_2} &= \inf_{K > 0} \qty{K: \EE\exp(\xi^2 / K^2) \le 2}.
\end{align}
The sub-exponential norm $\norm{\xi}_{\psi_1}$ is defined as in \citet[Definition 2.7.5]{vershynin2018high}:
\begin{align}
    \norm{\xi}_{\psi_1} &= \inf_{K > 0} \qty{K:  \EE\exp(\abs{\xi}/ K) \le 2}. \label{eq:subexponentialnorm}
\end{align}

We will occasionally need to use the following variant of Hanson-Wright for nonsparse bilinear forms, first proved in \citet{park2021estimating}. 
\begin{theorem}[Hanson-Wright for bilinear forms without sparsity]\label{thm:hanson-wright-bilinear}
Let $\vx = (X_1, \ldots, X_n) \in \RR^n$ and $\vy \in (Y_1, \ldots, Y_n)$ be random vectors such that the pairs $(X_i, Y_i)$ are all independent of each other (however $X_i$ and $Y_i$ can be correlated). Assume also that $\EE[X_{i}] = \EE[Y_i] = 0$ and $\max\qty{\norm{X_i}_{\psi_2}, \norm{Y_i}_{\psi_2}} \le K$. Then there exists an absolute constant $c > 0$ such that for all $\mM \in \RR^{n \times n}$ and $\epsilon \ge 0$ we have 
    \begin{align}
    &\Pr\qty[|\vx^\top \mM \vy - \EE[\vx^\top \mM \vy ]| > \epsilon] \le 2\exp(-c\min\qty{\frac{\epsilon^2}{K^4\norm{\mM}_F^2}, \frac{\epsilon}{K^2\norm{\mM}_2}}).\label{eq:hansonwright-bilinear}
\end{align}
\end{theorem}

Finally, we restate our new version of Hanson-Wright for bilinear forms with soft sparsity, which we prove in \cref{sec:hanson-wright-bounded}.

\hansonwright*

\subsection{Gram matrices and the Woodbury formula}\label{sec:features}
In order to apply Hanson-Wright to the bilinear form $\vx^\top \mM \vy$, we need to have a deterministic matrix $\mM$ such that the hypotheses are satisfied. However, in our setting we study bilinear forms such as $\vz_j^\top \Ainv \dely$. Here, the inverse Gram matrix $\Ainv$ is not independent of $\vz_j$ or $\dely$, so we cannot simply condition on $\Ainv$. The way around this is to cleverly decompose $\Ainv$ using the so-called Woodbury inversion formula (stated formally below), which generalizes the leave-one-out trick and Sherman-Morrison used to study binary classification in \citet{binary:Muth20}. To that end, we will explicitly decompose the Gram matrix $\mA \triangleq \sum_{j \in [d]} \lambda_j \vz_j \vz_j^\top$ based on whether the features $\vz_j$ are favored or not.

We now introduce some notation to keep track of which matrices contain or leave out which indices. In general, we use subscripts to denote which sets of features we preserve or leave out; we use a minus sign to signify leaving out. 
The $k$ label-defining features are represented with a subscript $k$, whereas the $s-k$ favored but not label defining features are represented with a subscript $F$. The rest of the $d-s$ unfavored features are represented with a subscript $U$. 

For notational convenience, we introduce some new notation for the weighted features, as the superscript $w$ to denote weighted features is rather cumbersome. We denote the weighted label-defining feature matrix by $\Xk \triangleq \mqty[\vw_1 & \cdots & \vw_k] \in \RR^{n \times k}$ \label{eq:Xt}, where the vectors $\vw_i \triangleq \sqrt{\lambda_i} \vz_i \in \RR^n$ denote the weighted observations for feature $i$. Define the unweighted label-defining feature matrix $\Zk \triangleq \mqty[\vz_1 & \cdots & \vz_k] \in \RR^{n \times k}$\label{eq:Zt}. Similarly, define $\Xsk \triangleq \mqty[\vw_{k+1} & \cdots & \vw_s] \in \RR^{n \times (s-k)}$, which contains the rest of the weighted favored features and the corresponding unweighted version $\Zsk$. 

Let $\Ak \triangleq \sum_{i \not\in [k]} \vw_i \vw_i^\top$ denote the leave-$k$-out Gram matrix which removes the $k$ label-defining features. Similarly, let $\Af \triangleq \sum_{i \not\in [s] \setminus [k]} \vw_i \vw_i^\top \in \RR^{n \times n}$ to denote leave-$(s-k)$-out Gram matrix which removes the favored but not label-defining features. Finally, 
let $\Au \triangleq \sum_{i \not\in [s]} \vw_i \vw_i^\top \in \RR^{n \times n}$  denote the leave-$s$-out matrix which only retains the unfavored features. We will also sometimes write $\As$ instead of $\Au$ to emphasize that the $s$ favored features have all been removed.

Define the so-called hat matrices by\label{eq:hatk}
\begin{align}
    \mH_k &\triangleq \Xk^\top \Akinv \Xk \in \RR^{k \times k} \\
    \mH_F &\triangleq \Xsk^\top \Askinv \Xsk \in \RR^{(s-k) \times (s-k)}.
\end{align}
These hat matrices appear in the Woodbury inversion formula. For the sake of notational compactness, define 
\begin{align}
    \mM_k &\triangleq \Xk (\mI_k + \mH_k)^{-1} \Xk^\top \in \RR^{n \times n} \\
    \mM_F &\triangleq \Xsk (\mI_{s-k} + \mH_F)^{-1} \Xsk^\top \in \RR^{n \times n}.
\end{align}

The Woodbury inversion formula yields 
\begin{align}
    \Ainv &= (\Xk \Xk^\top + \Ak)^{-1} \\
    &= \Akinv - \Akinv \Xk(\mI_k + \mH_k)^{-1}\Xk^\top \Akinv \label{eq:woodbury-expansion-verbose} \\
    &= \Akinv - \Akinv \mM_k \Akinv. \label{eq:woodbury-expansion}
\end{align}
Left multiplying \eqref{eq:woodbury-expansion-verbose} by $\Xk^\top$ yields
\begin{align}
    \Xk^\top \Ainv &= \Xk^\top \Akinv - \mH_k(\mI_k+\mH_k)^{-1} \Xk^\top \Akinv \\
    &= (\mI_k - \mH_k(\mI_k + \mH_k)^{-1})\Xk^\top \Akinv \\
    &= (\mI_k + \mH_k)^{-1} \Xk^\top \Akinv.\label{eq:push-through-woodbury}
\end{align}
We can derive completely analogous identities using $\Askinv$ instead of $\Akinv$. The above exposition is summarized by the following lemma.
\begin{lemma}\label{fact:woodbury}
We have 
\begin{align}
    \Xk^\top \Ainv \dely &= (\mI_k + \mH_k)^{-1} \Xk^\top \Akinv \dely \\
    \Xsk^\top \Ainv \dely &= (\mI_{s-k} + \mH_F)^{-1} \Xsk^\top \Askinv \dely.
\end{align}
\end{lemma}
Lemma~\ref{fact:woodbury} is quite powerful. Indeed, consider the action of the linear operator $\Xk^\top \Ainv: \RR^n \to \RR^k$ on $\dely$. The action is identical to that of the linear operator $\Xk^\top \Akinv: \RR^n \to \RR^k$, up to some invertible transformation. This new linear operator is nice because $\Akinv$ is independent of $\Xk$ and $\dely$, as it removes all of the label-defining features. Reclaiming independence sets the stage for using our variant of Hanson-Wright.

How does the invertible operator $(\mI_k + \mH_k)^{-1}$ act? Our general strategy is to show that $\mH_k$ is itself close to a scaled identity matrix, i.e. $\mH_k \approx \nu \mI_k$ for an appropriately defined $\nu$. Then for any $i \in [k]$, we have that 
\[
\vw_i^\top \Ainv \dely \approx (1+\nu)^{-1} \vw_i^\top \Akinv \dely. 
\]
Of course, there will be some error in this approximation, as $\mH_k$ is not \emph{exactly} equal to $\nu \mI_k$. Nevertheless, we can bound away the error that arises from this approximation.

\subsection{Concentration of spectrum}\label{sec:spectrum-concentration}
As foreshadowed in the previous section, we will leverage the fact that the hat matrices such as $\mH_k$ are close to a scaled identity. To formalize this, we appeal to random matrix theory and show that the spectra of various random matrices are very close to being flat (i.e. all eigenvalues are within $1 + o(1)$ of each other). To that end, we present the following standard characterization of the spectrum of a standard Wishart matrix, which is Equation 2.3 in \citet{rudelson2010non}.

\begin{lemma}[Concentration of spectrum for Wishart matrices]\label{lemma:wishart-concentration}
Let $\mM \in \RR^{M \times m}$ with $M > m$ be a real matrix with iid $N(0,1)$ entries. Then for any $\epsilon \ge 0$, we have with probability at least $1 - 2e^{-\epsilon^2/2}$ that 
\begin{equation}
\sqrt{M} - \sqrt{m} - \epsilon \le \sigma_{\min{}}(\mM) \le \sigma_{\max{}}(\mM) \le \sqrt{M} + \sqrt{m} + \epsilon.
\end{equation}
In other words, the singular values of $\mM$ satisfy subgaussian concentration. 
\end{lemma}

Since $\mu_m(\mM^\top \mM) = \sigma_{\min{}}(\mM)^2$ and $\mu_1(\mM^\top \mM) = \sigma_{\max{}}(\mM)^2$, we can conclude that if $m = o(M)$, then for any $\epsilon > 0$ we have 
\begin{equation}
    M - 2\sqrt{Mm} - \epsilon + o(\sqrt{Mm}) \le \mu_m(\mM^\top \mM) \le \mu_1(\mM^\top \mM) \le M + 2\sqrt{Mm} + \epsilon + o(\sqrt{Mm}),
\end{equation}
with probability at least $1 - 2e^{-\epsilon^2/2}$. 

On the other hand, consider $\mM \mM^\top \in \RR^{M \times M}$. Its spectrum is just that of $\mM^\top \mM \in \RR^{m \times m}$ with an additional $M-m$ zeros corresponding to the fact that $m < M$.

We can use \cref{lemma:wishart-concentration} to prove concentration of the spectrum of the various matrices introduced in \cref{sec:features}. Let us summarize some convenient forms of these results; their proofs are deferred to \cref{sec:technical-spectrum}. 
\begin{proposition}[Gram matrices have a flat spectrum]\label{prop:au-ak-flat}
    Recall that $\Au = \mA_{-s} = \sum_{j > s} \lambda_j \vz_j \vz_j^\top \in \RR^{n \times n}$ is the unfavored Gram matrix and $\mA_{-k} = \sum_{j > k} \lambda_j \vz_j \vz_j^\top \in \RR^{n \times n}$ is the leave-$k$-out Gram matrix.

    Then the following hold with probability at least $1 - 2e^{-n} - 2e^{-\sqrt{n}}$, 
    \begin{enumerate}[label=\normalfont{(\alph*)}]
        \item For all $i \in [n]$, we have $\mu_i(\Au) = n^p(1 \pm O(n^{-\konst{kappa_au}}))$.
        \item For all $i \in [s-k]$, we have 
        \begin{align}
            \mu_i(\Ak) = (1 + \mu^{-1})n^p(1 \pm O(n^{-\konst{kappa_ak_spiked}})),
        \end{align}
        where \konst{kappa_ak_spiked} is a positive constant.
        Moreover, for all $i \in [n] \setminus [s-k]$, we have 
        \begin{align}
            \mu_i(\Ak) = n^p(1 \pm O(n^{-\konst{kappa_au}})),
        \end{align}
        where \konst{kappa_au} is a positive constant.
    \end{enumerate}
\end{proposition}

As a simple corollary, we can obtain the following cruder bounds on the trace and spectral norm of $\Akinv$ and $\Asinv$.

\begin{corollary}[Trace and spectral norm of $\Akinv$]\label{corollary:trace-spectral-norm}
In the bi-level model, with probability at least $1-2e^{-n}$, we have 
\begin{align}
\Tr(\Auinv) &= n^{1-p}(1 \pm O(n^{-\konst{kappa_au}}))\sqrt{\log k} \\
\Tr(\Akinv) &= n^{1-p}(1 \pm O(n^{-\konst{kappa_akinv_trace}}))\sqrt{\log k} 
\end{align}
and 
\begin{equation}
\max\qty{\norm{\Akinv}_2, \norm{\Auinv}_2} \le \const{const_akinv_spectral_norm}n^{-p},
\end{equation}
where \const{const_akinv_spectral_norm}, \konst{kappa_au}, and \konst{kappa_akinv_trace} are all positive constants.
\end{corollary}
\begin{proof}
We prove the claim for $\Akinv$; the proof for $\Auinv$ is similar or easier because $\Auinv$ has a flat spectrum (\cref{prop:au-ak-flat}).

    If $q+r<1$, the upper bound for the spectral norm similarly follows. For the trace bounds, we can apply \cref{prop:au-ak-flat}, we have 
    \begin{align}
        \Tr(\Akinv) &= (n - n^r + n^t) n^{-p}(1 \pm O(n^{-\konst{kappa_au}})) + (n^r - n^t) \cdot (1 + \mu^{-1}) n^{-p}(1 \pm O(n^{-\konst{kappa_ak_spiked}})) \\
        &= n^{1-p}(1 \pm O(n^{-\konst{kappa_akinv_intermediate}}))
    \end{align}
    where
    \[
        \newk\label{kappa_akinv_intermediate} = \min\qty{r-1, 2-q-2r} > 0,
    \]
    as $q+2r < 2(q+r) < 2$ by assumption.
    
    On the other hand, the claim is obviously true when $q+r > 1$, as the entire spectrum of $\Akinv$ is $(1 \pm O(n^{-\konst{kappa_akinv_regression_fails}}))n^{-p}$ with an appropriately defined positive constant $\newk\label{kappa_akinv_regression_fails}$. The spectral norm bound follows by defining  \newc\label{const_akinv_spectral_norm} to be any positive constant greater than 1 which absorbs the $o(1)$ deviation terms in the spectrum. 
    
    The proof concludes by setting $\newk\label{kappa_akinv_trace} = \min\qty{\konst{kappa_akinv_intermediate}, \konst{kappa_akinv_regression_fails}}$.
\end{proof}

Finally, we have the following proposition which controls the spectrum of hat matrices such as $\mH_k \triangleq \Xk^\top \Akinv \Xk \in \RR^{k \times k}$. The intuition is that even though the spectrum of $\Akinv$ may be spiked, the spectrum of $\Xk^\top \Akinv \Xk$ is ultimately flat because we are taking an extremely low dimensional projection which is unlikely to see significant contribution from the spiked portion of $\Akinv$. 

In fact, we can prove a more general statement, which will be useful for us in the proof. Let $\varnothing \neq T \subseteq S \subseteq [s]$; here $T$ and $S$ index nonempty subsets of the $s$ favored features. Then we can define $\Xt$ to be the matrix of weighted features in $T$ and the leave-$T$-out Gram matrix 
\begin{align}
    \At \triangleq \sum_{j \not\in T} \lambda_j \vz_j \vz_j^\top.\label{eq:At}
\end{align}
Now define the $(T, S)$ hat matrix as $\mH_{T, S} \triangleq \Xt^\top \mA_{-S}^{-1} \Xt$. Evidently we have $\mH_k = \mH_{[k], [k]}$, so our notion is more general. The full proof is deferred to \cref{sec:technical-spectrum}. 

\begin{restatable}[Generalized hat matrices are flat]{proposition}{hatmatrix}\label{prop:concentration-hat-matrix-simple}
    Assume we are in the bi-level ensemble \cref{def:bilevel}. For any nonempty $T \subseteq S \subseteq [s]$, with probability at least $1 - 2e^{-\sqrt{n}} - 2e^{-n}$, we have all the eigenvalues tightly controlled:
    \begin{align}
        \mu_i((\mI_{\abs{T}} + \mH_{T, S})^{-1}) = \min\qty{\mu, 1}(1 \pm c_{T, S}n^{-\konst{kappa_hatts_final}}).
    \end{align}
    where $c_{T,S}$ and $\konst{kappa_hatts_final}$ are positive constants that depend on $\abs{T}$ and $\abs{S}$. 
\end{restatable}

\section{Utility bounds: applying the tools}\label{sec:utility}

Wishart concentration allows us to tightly bound the hat matrix and pass to studying bilinear forms of the form $\vw_i^\top \Akinv \dely$ rather than $\vw_i^\top \Ainv \dely$. Since $\Akinv$ is independent of $\Xk$ and $\dely$, we can condition on $\Akinv$ and then apply Hanson-Wright (\cref{thm:bounded-hanson-wright-bilinear}) to these bilinear forms for every realization of $\Akinv$. In this section, we will explicitly calculate the scaling of the typical value of these bilinear forms using the bi-level ensemble scaling; these will prove to be useful throughout the rest of the paper.

We first state the following proposition which bounds the correlation between the relevant label-defining features and the label vectors; it is a combination of Propositions D.5 and D.6 in \citep{subramanian2022generalization}.  
\begin{proposition} \label{prop:feature-label-correlation}
For any distinct $\a, \b \in [k]$, we have
\begin{align}
  \frac{1}{\sqrt{\pi \ln 2}}  \cdot \frac{n}{\nc} \cdot \sqrt{\ln \nc} \leq \EE[\za \tran \ya] \leq \sqrt{2}  \cdot \frac{n}{\nc} \cdot \sqrt{\ln \nc} \label{eqn:expzayabound}
\end{align}
and 
\begin{align}
      -\sqrt{2}  \cdot \frac{n}{\nc} \cdot \frac{1}{k-1} \cdot \sqrt{\ln \nc} \leq  \EE[\za \tran \yb] \leq - \frac{1}{\sqrt{\pi \ln 2}}  \cdot \frac{n}{\nc}\cdot  \frac{1}{k-1} \cdot \sqrt{\ln \nc} \label{eqn:expzaybbound}
    \end{align}
\end{proposition}

With the above proposition in hand, we can prove the following lemma which gives concentration of the bilinear forms that we study.
\begin{lemma}\label{lemma:hanson-wright-feature-label-concentration}
Let $i \in [d]$ and $\dely = \ya - \yb$ where $\a, \b \in [k]$ and $\b \neq \a$. Let $\mM \in \RR^{n \times n}$ be a (random) matrix which is independent of $\vz_i$ and $\dely$. Then conditioned on $\mM$, with probability at least $1 - 1/nk$, 
\[
\abs{\vz_i^\top \mM \dely - \EE[\vz_i^\top \mM \dely | \mM]} \le \const{chansonwright} \sqrt{\frac{n}{k}}\norm{\mM}_2\sqrt{\log(nk)},
\]
and the same holds with $\dely$ replaced with $\ya$.
Here, \const{chansonwright} is an appropriately chosen universal positive constant.

Moreover, we have
\begin{enumerate}[label=\normalfont{(\arabic*)}]
    \item For any distinct $\a, \b \in [k]$, we have 
    \begin{align}
        \EE[\vz_\a^\top \mM \dely |\mM] &= \const{const_alpha}\frac{\sqrt{\log k}}{k}\tr(\mM) = -\EE[\vz_\b^\top \mM \dely] \\
        \EE[\vz_\a^\top \mM \ya |\mM] &= \const{const_ya}\frac{\sqrt{\log k}}{k}\tr(\mM),
    \end{align}
    where \const{const_alpha} and \const{const_ya} are positive constants.
    \item For $i \in [d] \setminus \qty{\a, \b}$, we have
        \begin{align}
        \EE[\vz_i^\top \mM \dely |\mM] &= 0.
        \end{align}
    \item For $i \in [d] \setminus \qty{\a}$, we have
         \begin{align}
            \EE[\vz_i^\top \mM \ya |\mM] &= -\const{const_zi_ya}\frac{\sqrt{\log k}}{k(k-1)},
        \end{align}
        where \const{const_zi_ya} is a positive constant.
\end{enumerate}
\end{lemma}
\begin{proof}
Let us check the conditions for our new variant of Hanson-Wright with soft sparsity (\cref{thm:bounded-hanson-wright-bilinear}). 
We want to apply it to the random vectors $(\vz_i, \dely) = (\vz_i[j], \dely[j])_{j=1}^n$. Some of the hypotheses are immediate by definition. Evidently, $(\vz_i[j], \dely[j])$ are independent across $j$, and are mean zero. Since $\vz_i[j] \sim N(0, 1)$, it is subgaussian with parameter at most $K = 2$. For the bounded and soft sparsity assumption, we clearly have $\abs{\dely[j]}\le 1$ and $\ya[j] \le 1$ almost surely. Also, since $\dely[j]^2 \sim \Ber(\frac{2}{k})$, we have $\EE[\dely[j]^2] = \frac{2}{k}$. Similarly, $\EE[\ya[j]^2] = \frac{1}{k}(1 - \frac{1}{k})^2 + (1 - \frac{1}{k})\frac{1}{k^2}\le \frac{2}{k}$.

The more complicated condition is the subgaussianity of $\vz_i[j]$ conditioned on the value of $\dely[j]$ or $\ya[j]$. Regardless of whether we're conditioning on $\dely$ or $\ya$, it suffices to instead prove that $\vz_i[j]$ is subgaussian conditioned on whether feature $i$ won the competition for datapoint $j$. First, suppose $i$ won, i.e. $\yoh_i[j] = 1$. Then the Borell-TIS inequality \citep[Theorem~2.1.1]{adler2007random} implies that $\vz_i[j]$ satisfies a subgaussian tail inequality. By the equivalent conditions for subgaussianity \citet[Proposition~2.5.2]{vershynin2018high}, it follows that $\vz_i[j] - \EE[\vz_i[j] | \yoh_i[j] = 1]$ conditionally has subgaussian norm bounded by some absolute constant $K$. If $i$ doesn't win (or doesn't participate in the competition), then Proposition D.2 in \citet{subramanian2022generalization} implies that $\vz_i[j] - \EE[\vz_i[j] | \yoh_i[j] = 0]$ conditionally has subgaussian norm bounded by $6$. 

Finally, since $\mM$ is independent of $\vz_i$ and $\dely$, we can condition on $\mM$ and apply \cref{thm:bounded-hanson-wright-bilinear} to the bilinear form for every realization of $\mM$.

Hence, we conclude that with probability at least $1-1/nk$ we have
\begin{align}
    \abs{\vz_i^\top \mM \dely - \EE[\vz_i^\top \mM \dely | \mM]} \le \const{chansonwright}\sqrt{\frac{n}{k}}\norm{\mM}_2\sqrt{\log(nk)},
\end{align}
where \newc\label{chansonwright} is an appropriately chosen absolute constant based on $K$ and the constant $c$ defined in \cref{thm:bounded-hanson-wright-bilinear}. 

Now we can compute $\EE[\vz_i^\top \mM \dely | \mM]$ to prove the rest of the theorem. If $i = \alpha$, we have
\[
\EE[\vz_\a^\top \mM \dely |\mM]= \tr(\mM \EE[\dely\vz_\a^\top]).
\]
Let us now compute $\EE[\dely \vz_\a^\top]$. From \cref{eqn:expzayabound} in \cref{prop:feature-label-correlation}, we have $\EE[\ya \vz_\a^\top] = \const{const_ya}\frac{\sqrt{\log k}}{k}\mI_n$, where $\frac{1}{\sqrt{\pi \log 2}} \le \newc\label{const_ya} \le \sqrt{2}$. 
Similarly, we have $\EE[\yb \vz_\a^\top] = -\const{const_zi_ya}\frac{\sqrt{\log k}}{k(k-1)}\mI_n$
where $\frac{1}{\sqrt{\pi \log 2}} \le \newc\label{const_zi_ya} \le \sqrt{2}$.
It follows that $\EE[\dely \vz_\a^\top] = \Theta\qty(\frac{\sqrt{\log k}}{k})\mI_n$.

For $i \in [d] \setminus \qty{\a, \b}$, by symmetry we obtain $\EE[\ya \vz_i^\top] = \EE[\yb \vz_i^\top]$. This implies $\EE[\dely \vz_i^\top] = \EE[\ya \vz_i^\top] - \EE[\yb \vz_i^\top] = 0$, so we obtain
\begin{align}
    \EE[\vz_i^\top \mM \dely |\mM] &= \tr(\mM \EE[\dely\vz_i^\top]) \\  &=
    0.
\end{align}
\end{proof}

Plugging in the bi-level scaling, we obtain the following corollary.
\begin{corollary}[Asymptotic concentration of bilinear forms]\label{corollary:asymptotic-concentration}
In the bi-level model, for any $i \in [k]$, we have with probability at least $1-O(1/nk)$ that 
\[
\abs{\vz_i^\top \Akinv \dely - \EE[\vz_i^\top \Akinv \dely]} \le \const{const_asymptotic_bilinear} n^{\frac{1-t}{2}-p}\sqrt{\log (nk)}.
\]
Moreover, we have
\begin{enumerate}[label=\normalfont{(\arabic*)}]
    \item For any distinct $\a, \b \in [k]$,
        \begin{align}
        \EE[\vz_\a^\top \Akinv \dely] = \const{const_alpha}n^{1-t-p}(1 \pm O(n^{-\konst{kappa_akinv_trace}}))\sqrt{\log k} = -\EE[\vz_\b^\top \Akinv \dely] \label{eq:keycorrC3}
        \end{align}
\end{enumerate}
The same statements hold (with different constants) if we replace $\Akinv$ with $\Asinv$.
\end{corollary}
\begin{proof}
From \cref{corollary:trace-spectral-norm}, we have $\norm{\Akinv}_2 \le \const{const_asymptotic_bilinear}n^{-p}$, where $\newc\label{const_asymptotic_bilinear}$ is an appropriately chosen universal positive constant based on \const{chansonwright}. Recall that $\Ak$ is obtained by removing the $k$ label-defining features, so in particular $\Akinv$ is independent of $(\vz_i, \dely)$ for $i \in [k]$. Hence, the conditions for \cref{lemma:hanson-wright-feature-label-concentration} are satisfied. Then applying the union bound for the spectral norm bound on $\Akinv$, we see that with probability at least $1-O(1/nk)$, the deviation term from Hanson-Wright is at most $\const{const_asymptotic_bilinear} n^{\frac{1}{2}-p}\sqrt{\log (nk)}$. 

We now turn to calculating the asymptotic scalings for the expectations. 
From \cref{lemma:hanson-wright-feature-label-concentration}, we know that $\EE[\vz_\a^\top \Akinv \dely |\Akinv] = \const{const_alpha}\frac{\sqrt{\log k}}{k}\tr(\Akinv)$. Applying the high probability bound on $\tr(\Akinv)$ from \cref{corollary:trace-spectral-norm}, we obtain that with probability at least $1-O(1/nk)$ that
\begin{align}
  -\EE[\zb^\top \Akinv \dely] = \EE[\vz_\a^\top \Akinv \dely] &= \const{const_alpha} n^{-t} n^{1-p}(1 \pm O(n^{-\konst{kappa_akinv_trace}})) \sqrt{\log k} \\
  &= \const{const_alpha} n^{1-t -p}(1 \pm O(n^{-\konst{kappa_akinv_trace}}))\sqrt{\log k}  
\end{align}
where in the second line we have applied \cref{corollary:trace-spectral-norm} and  \newc\label{const_alpha} is an appropriately chosen positive constant. This proves \eqref{eq:keycorrC3}. 
\end{proof}

With \cref{corollary:asymptotic-concentration} in hand, we are now in a position to do some straightforward calculations and bound some quantities which will pop up in the survival and contamination analysis.

\begin{proposition}[Worst-case bound based on Hanson-Wright]\label{prop:worst-case-correlation}
Let $T \subseteq [s]$ be a subset of favored features such that $\qty{\a, \b} \subseteq T$. Assume that $\abs{T} = n^{\tau}$ for some $\tau \le r$. Then with probability at least $1-O(1/nk)$, we have
\begin{align}
    \norm{\Zt^\top \Atinv \dely}_2 \le \const{const_norm_bilinear}(n^{1-t-p} + n^{\frac{1+\tau-t}{2}-p})\sqrt{\log(nk\abs{T})}.
\end{align}
\end{proposition}
\begin{proof}
WLOG, suppose $\a = 1$ and $\b = 2$.
By \cref{corollary:asymptotic-concentration} we have with probability at least $1-1/n$ that
\begin{align}
    \abs{\Zt^\top \Atinv \dely} \le \mqty[\const{const_alpha}n^{1-t-p}\sqrt{\log k} \\ \const{const_alpha}n^{1-t-p}\sqrt{\log k} \\ \const{const_asymptotic_bilinear}n^{\frac{1-t}{2}-p}\sqrt{\log (nk)} \\ \vdots \\ \const{const_asymptotic_bilinear}n^{\frac{1-t}{2}-p}\sqrt{\log (nk)}].
\end{align}

Hence, the norm of this vector is at most 
\begin{align}
\norm{\Zt^\top \Atinv \dely}_2 &\le 2\const{const_alpha}n^{1-t-p}\sqrt{\log k} + n^{\frac{\tau}{2}}\const{const_asymptotic_bilinear}n^{\frac{1-t}{2}-p}\sqrt{\log (nk)} \\
&\le \const{const_norm_bilinear}(n^{1-t-p} + n^{\frac{1+\tau-t}{2}-p})\sqrt{\log(nk)},
\end{align}
where \newc\label{const_norm_bilinear} is a positive constant. 

\end{proof}

\section{Bounding the survival}\label{sec:survival}
Recall that the relative survival was defined to be $\lambda_F \hhat_{\a, \b}[\a] = \lambda_F \za^\top \Ainv \dely$. The strategy is to apply our variant of Hanson-Wright to $\za^\top \Ainv \dely$. Unfortunately, $\Ainv$ is not independent of $\za$ or $\dely$, so we need to use Woodbury to extract out the independent portions and bound away the dependent portion. As we'll see shortly, the error from the dependent portions can also be controlled using Hanson-Wright. Let us now recall \cref{prop:relative-survival-bound-accurate} for reference.

\survivalacc*

\begin{proof}
Recall that $\hhat_{\a, \b}[\a] = \vz_\a^\top \Ainv \dely$. 
We first observe that for $i \in [k]$, the dependence between $\Ainv$ and $\vz_i$ as well as $\Ainv$ and $\dely$ only comes through the $k$ label defining features. Hence, we can use the Woodbury identity to extract out the independent portions of $\Ainv$. 

Indeed, our ``push through'' lemma for Woodbury (\cref{fact:woodbury}) and concentration of the hat matrix (\cref{prop:concentration-hat-matrix-simple}) implies that with extremely high probability
\begin{align}
    \Zk^\top \Ainv \dely &= (\mI_k+\mH_k)^{-1}\Zk^\top \Akinv \dely \\
    &= \min\qty{\mu, 1}(\mI_k + \mE)\Zk^\top \Akinv \dely, \label{eq:bilinear-form-pushed-through}
\end{align}
where $\norm{\mE}_2 = O(n^{-\konst{kappa_hatts_final}})$. 

Let $\vu_\a \in \RR^k$ denote the $\a$th row vector in $\mE$, and let $\vu_\a^{-} \in \RR^{k-1}$ denote the subvector of $\vu_\a$ without index $\a$. By reading off the $\a$th row of \cref{eq:bilinear-form-pushed-through}, we see that 
\begin{align}
    \za^\top \Ainv \dely &= \min\qty{\mu, 1}(\za^\top \Akinv\dely + \ev{\vu_\a, \Zk^\top \Akinv \dely})
\end{align}
Since $\norm{\vu_\a}_2 \le \norm{\mE}_2 = O(n^{-\konst{kappa_hatts_final}})$, it follows from Cauchy-Schwarz that
\begin{align}
    \abs{\za^\top \Ainv \dely - \min\qty{\mu, 1} \za^\top \Akinv \dely} \le \min\qty{\mu, 1} O(n^{-\konst{kappa_hatts_final}})\norm{\Zk^\top \Akinv \dely}_2. \label{eq:relative-survival-conc}
\end{align}

Let us pause for a moment and interpret \cref{eq:relative-survival-conc}. The term $\min\qty{\mu, 1}$ is merely capturing the difference in behavior when regression works and fails; if regression works ($q+r<1$) then it becomes $\mu$, and if regression fails ($q+r>1$), then it becomes $1$. This behavior should be expected: in the regression works case, we expect the effect of interpolation to be a \emph{regularizing} one: the signals are attenuated by a factor of $\mu$. The RHS of \cref{eq:relative-survival-conc} is an error term, capturing how differently $\za^\top \Ainv \dely$ behaves from the expected behavior $\min\qty{\mu, 1} \za^\top \Akinv \dely$. 

Let us now bound the error term. From \cref{prop:worst-case-correlation} we have with probability at least $1-O(1/nk)$ that
\begin{align}
    \norm{\Zk^\top \Akinv \dely}_2 &\le \const{const_norm_bilinear}(n^{1-t-p} + n^{\frac{1}{2}-p})\sqrt{\log(nk^2)}.
\end{align}

Let us do casework on $t$. For $t < \frac{1}{2}$, we have $\frac{1}{2} - p < 1-t-p$, so we conclude that the error term is $\min\qty{\mu, 1}O(n^{1-t-p} \cdot n^{-\konst{kappa_worst_case_error}})\sqrt{\log(nk^2)}$, where \newk\label{kappa_worst_case_error} is a positive constant.

On the other hand, our Hanson-Wright calculations imply (\cref{corollary:asymptotic-concentration}) that with probability at least $1 - O(1/nk)$ that
\begin{align*}
    \abs{\za^\top \Akinv \dely - \const{const_alpha}n^{1-t-p}(1 \pm O(n^{-\konst{kappa_akinv_trace}}))\sqrt{\log k}} \le \const{const_asymptotic_bilinear} n^{\frac{1}{2}-p}\sqrt{\log (nk)}.
\end{align*}
Again, since $t < \frac{1}{2}$, the deviation term is $o(n^{1-t-p})\sqrt{\log k}$.

Hence we conclude that with probability $1 - O(1/nk)$ we have
\begin{align*}
    \za^\top \Ainv \dely = \const{const_alpha}\min\qty{\mu, 1}n^{1-t-p}(1 \pm O(n^{-\konst{kappa_survival}})) \sqrt{\log k},
\end{align*}
where \newk\label{kappa_survival} is a positive constant.

Completely analogous logic handles the bounds for $\zb^\top \Ainv \dely$. Let us now return back to the quantity of interest, $\lambda_F \hhat_{\a, \b}[\a]$. We can compute 
\begin{align*}
    \lambda_F \za^\top \Ainv \dely &= n^{p-q-r} \cdot \const{const_alpha}\min\qty{\mu, 1}n^{1-t-p}(1 \pm O(n^{-\konst{kappa_survival}})) \sqrt{\log k} \\
    &= \const{const_alpha}\mu^{-1} \min\qty{\mu, 1}n^{-t}(1 \pm O(n^{-\konst{kappa_survival}}))\sqrt{\log k} \\
    &= \const{const_alpha}\min\qty{1, \mu^{-1}}n^{-t}(1 \pm O(n^{-\konst{kappa_survival}}))\sqrt{\log k}.
\end{align*}

On the other hand, if $t \ge \frac{1}{2}$ the error terms all dominate, and we replace $n^{1-t-p}$ with $n^{\frac{1}{2} - p}$ everywhere. We conclude that with probability at least $1-O(1/nk)$,
\begin{align}
    \abs{\za^\top \Ainv \dely} &\le \const{const_rel_survival_upper} \min\qty{\mu, 1}n^{\frac{1}{2}-p}\sqrt{\log(nk)},
\end{align}
where \newc\label{const_rel_survival_upper} is a positive constant. Plugging in the scaling for $\lambda_F$ yields the desired result.
\end{proof}

\section{Bounding the contamination}\label{sec:contamination}
In this section we give a tight analysis of the contamination term. First, we rewrite the squared contamination term and separate it out into the contamination from the $k-2$ label-defining features which are not $\a$ or $\b$, the rest of the $s-k$ favored features, and the remaining $d-s$ unfavored features. From \cref{eq:cn-new}, we have
\begin{align}
    \cn_{\a, \b}^2 &= \sum_{j \in [d] \setminus \qty{\a, \b}} \lambda_j^2 (\vz_j^\top \Ainv \dely)^2 \\
    &=  \dely^\top \Ainv \qty(\sum_{j \in [d] \setminus \qty{\a, \b}} \lambda_j^2 \vz_j \vz_j^\top)\Ainv \dely \\
    &= \underbrace{\dely^\top\Ainv \qty(\sum_{j \in [k] \setminus \qty{\a, \b}} \lambda_j^2 \vz_j \vz_j^\top) \Ainv \dely}_{\triangleq \cnl^2} 
 + \underbrace{\dely^\top\Ainv \qty(\sum_{j \in [s] \setminus [k]} \lambda_j^2 \vz_j \vz_j^\top) \Ainv \dely}_{\triangleq \cnf^2} \\
    &\quad\quad + \underbrace{\dely^\top\Ainv \qty(\sum_{j > s} \lambda_j^2 \vz_j \vz_j^\top) \Ainv \dely}_{\triangleq \cnu^2}. \label{def:cnf-cnu}
\end{align}
Here, $\cnl$ corresponds to contamination from label defining features, $\cnf$ corresponds to contamination from favored features, and $\cnu$ corresponds to contamination from unfavored features. The reason for separating out the contamination into these three subterms is that we will need slightly different arguments to bound each of them, although Hanson-Wright and Woodbury are central to all of the arguments. In \cref{sec:cnl-cnf-upper-bound} we prove the upper bound on $\cnl + \cnf$; in \cref{sec:cnl-cnf-lower-bound} we prove the lower bound. Finally, in \cref{sec:cnu-bounds} we bound $\cnu$. After putting these bounds together, we will obtain the main bounds on the contamination, which we restate here for reference.

\contaminationacc*

\subsection{Upper bounding the contamination from label-defining+favored features}\label{sec:cnl-cnf-upper-bound}
In this section, we upper bound the contamination coming from the $s-2$ favored features which are not $\a$ or $\b$. This culminates in the following lemma.
\begin{lemma}\label{lemma:cnl-cnf-upper-bound}
In the same setting as \cref{prop:contamination-bound-accurate}, we have with probability $1 - O(1/nk)$ that
    \begin{align*}
        \cnl^2 + \cnf^2 \le \const{const_cnf_cnl}^2 \min\qty{1, \mu^{-2}} n^{r-t-1} \log(nsk)^2,
    \end{align*}
    where \const{const_cnf_cnl} is a positive constant.
\end{lemma}
\begin{proof}
Let $\Xr \in \RR^{n \times (s-2)}$ as the weighted feature matrix which includes all of the $s-2$ favored features aside from $\a, \b$. We can then define $\Ar \triangleq \mA - \Xr \Xr^\top$ and $\mH_R = \Xr^\top \Arinv \Xr$. Using Woodbury,  an analogous computation to \cref{fact:woodbury} implies that 
    \begin{align}
        \Xr \Ainv \dely = (\mI_{s-2} + \mH_{R})^{-1} \Xr \Arinv \dely.
    \end{align}
The contamination from all of the $s-2$ favored features that are not $\a$ or $\b$ satisfies
\begin{align*}
    \cnl^2 + \cnf^2 &= \lambda_F \dely^\top \Ainv \sum_{j \in [s] \setminus \qty{\a, \b}} \vw_j\vw_j^\top \Ainv \dely \\
    &= \lambda_F \dely^\top \Ainv \Xr \Xr^\top \Ainv \dely \\
    &= \lambda_F \dely^\top \Arinv \Xr (\mI_{s-2} + \mH_{R})^{-2} \Xr^\top \Arinv \dely.
\end{align*}
Since \cref{prop:concentration-hat-matrix-simple} implies that $\mu_1((\mI_{s-2} + \mH_{R})^{-2}) \le \const{const_hatr_coarse}\min\qty{\mu^2, 1}$ with extremely high probability where \newc\label{const_hatr_coarse} is a positive constant, we know the contamination is with extremely high probability upper bounded by the following quadratic form:
\begin{align}
    &\const{const_hatr_coarse}\min\qty{\mu^2, 1}\lambda_F \dely^\top \Arinv \Xr \Xr^\top \Arinv \dely \\
    &\quad\quad = \const{const_hatr_coarse}\min\qty{\mu^2, 1}\lambda_F^2 \sum_{j \in [s] \setminus \qty{\a, \b}} \ev{\vz_j, \Arinv \dely}^2. \label{eq:sum-subexp}
\end{align}

We still cannot apply Hanson-Wright, because $\Arinv$ is not independent of $\dely$. However, we can use Woodbury again to take out $\za, \zb$ from $\Arinv$.

Define $\Xab = \mqty[\vw_\a & \vw_\b]$ and $\mH_{\a, \b}^{(s)} = \Xab^\top \Asinv \Xab$. Then Woodbury implies that
\begin{align}
    \Arinv = \Asinv - \Asinv \Xab(\mI_2 + \mH_{\a, \b}^{(s)})^{-1}\Xab^\top \Asinv.
\end{align}
Hence 
\begin{align}
    \vz_j^\top \Arinv \dely = \vz_j^\top \Asinv \dely - \vz_j^\top \Asinv \Xab (\mI_2 + \mH_{\a, \b}^{(s)})^{-1} \Xab^\top \Asinv \dely. \label{eq:label-defining-decomposition}
\end{align}

We will use Hanson-Wright and Cauchy-Schwarz to argue that the second term in \cref{eq:label-defining-decomposition} above will be dominated by the first term. Indeed, \cref{corollary:asymptotic-concentration} implies that $\vz_j^\top \Asinv \dely \le \const{const_asymptotic_bilinear}n^{\frac{1-t}{2}-p}\sqrt{\log(nsk)}$ with probability at least $1 - O(1/nsk)$, so it suffices to show that the other term is dominated by $n^{\frac{1-t}{2}-p}$. We will show that its contribution for each $j$ is $\min\qty{1, \mu^{-1}}\Tilde{O}(n^{\frac{1}{2} - t - p})$. By Cauchy-Schwarz, the magnitude of the second term of \cref{eq:label-defining-decomposition} is at most 
\begin{align}
    & \norm{(\mI_2 + \mH_{\a, \b})^{-1}}_2 \norm{\Xab^\top \Asinv \vz_j}_2 \norm{\Xab^\top \Asinv \dely}_2 \\&\quad\quad \le \const{const_label_defining_deviation}\lambda_F \min\qty{\mu, 1} n^{\frac{1}{2} - p}\sqrt{\log(nsk)} \cdot n^{1-t-p}\sqrt{\log k} \\
    &\quad\quad\le \const{const_label_defining_deviation} \min\qty{1, \mu^{-1}} n^{\frac{1}{2} - t - p}\log(nsk),
\end{align}
where \newc\label{const_label_defining_deviation} is a positive constant. In the second line, we have used \cref{prop:concentration-hat-matrix-simple} to upper bound $\norm{(\mI_2 + \mH_{\a, \b}^{(s)})^{-1}}_2 \le O(\min\qty{\mu, 1})$ and we have used  \Cref{thm:hanson-wright-bilinear,prop:au-ak-flat} to deduce that that $\abs{\vz_j^\top \Asinv \vz_a} \le O(n^{\frac{1}{2} - p})\sqrt{\log(nsk)}$ with probability at least $1-O(1/nsk)$. Similarly, we used the scaling from \cref{corollary:asymptotic-concentration} to deduce that $\abs{\vz_\a^\top \Asinv \dely} \le O(n^{1-t-p})\sqrt{\log k}$, and similarly for $\b$. 

Hence $\vz_j^\top \Arinv \dely$ is $O(n^{\frac{1-t}{2}-p})\log(nsk)$ with probability $1 - O(1/nsk)$. By union bounding over $j$ and plugging our upper bound back into \cref{eq:sum-subexp}, we conclude that with probability at least $1 - O(1/nk)$
\begin{align}
    \cnl^2 + \cnf^2 &\le  \const{const_hatr_coarse}\min\qty{\mu^2, 1}\lambda_F^2 n^r \cdot O(n^{1 - t - 2p}) \log(nsk)^2 \\ 
    &= \mu^{-2}\min\qty{\mu^2, 1}O(n^{r-t-1})\log(nsk)^2 \\
    &\le \const{const_cnf_cnl}^2 \min\qty{1, \mu^{-2}} n^{r-t-1} \log(nsk)^2,
\end{align}
where \newc\label{const_cnf_cnl} is a positive constant, concluding the proof.
\end{proof}

\subsection{Lower bounding the contamination from label-defining+favored features}\label{sec:cnl-cnf-lower-bound}
In this section, we upper bound the contamination coming from the $s-2$ favored features which are not $\a$ or $\b$. This culminates in the following lemma.
\begin{lemma}\label{lemma:cnl-cnf-lower-bound}
 In the same setting as \cref{prop:contamination-bound-accurate}, if $t > 0$, with probability at least $1 - O(1/nk)$, we have 
    \begin{align*}
        \cnl^2 + \cnf^2 \ge \const{const_cnl_cnf_lower}^2 \min\qty{1, \mu^{-2}} n^{r-t-1},
    \end{align*}
    where \const{const_cnl_cnf_lower} is a positive constant.
\end{lemma}
\begin{proof}
Following the beginning of the proof of \cref{lemma:cnl-cnf-upper-bound} and what we know about the flatness of the spectra of hat matrices from \cref{prop:concentration-hat-matrix-simple}, we can deduce that there is some positive constant \newc\label{const_hatr_lower} such that with extremely high probability 
\begin{align}
    \cnl^2 + \cnf^2 \ge \const{const_hatr_lower} \min\qty{\mu^2, 1} \lambda_F^2 \sum_{j \in [s] \setminus \qty{\a, \b}} \ev{\vz_j, \Arinv \dely}^2. \label{eq:cnl-cnf-lower-bound-coupling}
\end{align}
We will further lower bound this by throwing out all label-defining $j$. In other words, the goal now is to lower bound
\begin{align}
    \sum_{j \in [s] \setminus [k]} \ev{\vz_j, \Arinv \dely}^2 = \ev{\Zf^\top \Arinv \dely, \Zf^\top \Arinv \dely}. \label{eq:new-cnl-cnf-lower-bound}
\end{align}
The main idea is to use Bernstein's inequality, but unfortunately $\Arinv$ is not independent of $\dely$, so we will again resort to Woodbury to take out $\za$ and $\zb$. As in the proof for upper bounding the favored contamination, we have $\Xab = \mqty[\vw_\a & \vw_\b]$ and $\mH_{\a, \b}^{(s)} = \Xab^\top \Asinv \Xab$. Then we can deduce from another application of Woodbury that 
\begin{align}
    \vz_j^\top \Arinv \dely = \vz_j^\top \Asinv \dely - \vz_j^\top \Asinv \Xab (\mI_2 + \mH_{\a, \b})^{-1} \Xab^\top \Asinv \dely. 
\end{align}
Again, we can argue that with probability $1 - O(1/nsk)$, the second term is upper bounded in magnitude by \begin{align}
    \min\qty{1, \mu^{-1}} n^{\frac{1}{2} - t - p}\log(nsk) = \min\qty{1, \mu^{-1}} O(n^{\frac{1-t}{2} - p} \cdot n^{-\konst{kappa_bernstein_error}}),
\end{align}
where \newk\label{kappa_bernstein_error} is a positive constant 
because $t > 0$. Since Hanson-Wright (\cref{corollary:asymptotic-concentration}) implies that $\vz_j^\top \Asinv \dely = \Tilde{O}(n^{\frac{1-t}{2} - p})$, this implies that $\vz_j^\top \Arinv \dely = \Tilde{O}(n^{\frac{1-t}{2} - p})$, and similarly for $\b$. Hence we have 
\begin{align}
    (\vz_j^\top \Arinv \dely)^2 &= (\vz_j^\top \Asinv \dely + \min\qty{1, \mu^{-1}}O(n^{\frac{1-t}{2} -p -\konst{kappa_bernstein_error}}))^2 \\
    &= (\vz_j^\top \Asinv \dely)^2 + \min\qty{1, \mu^{-1}}O(n^{\frac{1-t}{2} - p -\konst{kappa_bernstein_error}})\Tilde{O}(n^{\frac{1-t}{2} - p}) \\
    &= (\vz_j^\top \Asinv \dely)^2 + \min\qty{1, \mu^{-1}}o(n^{1-t-2p}). \label{eq:final-cnl-cnf-decomp}
\end{align}

We are now in a position to analyze the contribution from the first term of \cref{eq:final-cnl-cnf-decomp} to \cref{eq:new-cnl-cnf-lower-bound}: its contribution is $\ev{\Zf^\top \Asinv \dely, \Zf^\top \Asinv \dely}$. This does have all the independence required to apply Bernstein, because $(\Asinv, \dely)$ are independent of $\Zf$. Hence conditioned on $\Asinv$ and $\dely$, $\ev{\Zf^\top \Asinv \dely, \Zf^\top \Asinv \dely}$ is a sum of $s-k$ subexponential variables, and by Lemma 2.7.7 of \citet{vershynin2018high} each of these random variables conditionally has subexponential norm at most $\norm{\Asinv \dely}_2^2$ and conditional mean $\ev{\Asinv \dely, \Asinv \dely}$. 

We can use Hanson-Wright (\cref{thm:bounded-hanson-wright-bilinear}) to bound both of these quantities. Indeed, it implies that with probability at least $1 - O(1/nk)$, 
\begin{align}
    \norm{\Asinv \dely}_2^2 \le O(n^{1-t-2p}).
\end{align}
Let us now compute the Hanson-Wright bound for $\ev{\Asinv \dely, \Asinv \dely}$. Note that $\Asinv$ is independent of $\dely$, so we can condition on $\Asinv$ and conclude that with probability at least $1 - O(1/nk)$
\begin{align}
    \ev{\Asinv \dely, \Asinv \dely} &\ge \EE[\ev{\Asinv \dely, \Asinv \dely} | \Asinv] - O(n^{\frac{1-t}{2}})\norm{\mA_{-s}^{-2}}_2 \sqrt{\log(nk)} \\
    &= \Tr(\mA_{-s}^{-2} \EE[\dely \dely^\top]) - O(n^{\frac{1-t}{2}})\norm{\mA_{-s}^{-2}}_2 \sqrt{\log(nk)} \\
    &= \frac{2}{k}\Tr(\mA_{-s}^{-2}) - O(n^{\frac{1-t}{2}})\norm{\mA_{-s}^{-2}}_2 \sqrt{\log(nk)},
\end{align}
where we have used the fact that $\dely$ is mean zero and $\dely[i]^2 \sim \Ber(\frac{2}{k})$.

From \cref{prop:au-ak-flat}, we obtain the scaling for $\Tr(\mA_{-s}^{-2})$ and $\norm{\mA_{-s}^{-2}}_2$. This implies that with probability at least $1 - O(1/nk)$
\begin{align}
\ev{\Asinv \dely, \Asinv \dely} &\ge \Omega(n^{1-t-2p}) - O(n^{\frac{1-t}{2} - 2p})\sqrt{\log(nk)} \\
&\ge \Omega(n^{1-t-2p}).
\end{align}
as $t < 1$.

Bernstein and the union bound implies that with probability at least $1 - O(1/nk)$, 
\begin{align}
    \ev{\Zf^\top \Asinv \dely, \Zf^\top \Asinv \dely} &\ge \qty(\sum_{j \in [s] \setminus [k]} \Omega(n^{1-t-2p})) - O(n^{\frac{r}{2} + 1 - t - 2p})\\
    &\ge \Omega(n^{r + 1 - t - 2p}),
\end{align}
as $r > 0$.

To wrap up, we will need to upper bound the contribution of the error term in \cref{eq:final-cnl-cnf-decomp}. Its contribution from summing over $j \in [s] \setminus [k]$ is $\min\qty{1, \mu^{-2}} o(n^{r+1-t-2p})$, which is negligible compared to the Bernstein term, which as we just proved is $\Omega(n^{r+1-t-2p})$. Hence $\ev{\Zf^\top \Arinv \dely, \Zf^\top \Arinv \dely} \ge \Omega(n^{r+1-t-2p})$ with high probability, and inserting this back into our lower bound \cref{eq:cnl-cnf-lower-bound-coupling}, we see that 
\begin{align}
    \cnl^2 + \cnf^2 &\ge \const{const_hatr_lower} \min\qty{\mu^2, 1} \lambda_F^2 \Omega(n^{r+1-t-2p}) \\
    &=  \const{const_hatr_lower} \mu^{-2} \min\qty{\mu^2, 1} \Omega(n^{r-t-1}) \\
    &\ge \const{const_cnl_cnf_lower}^2 \min\qty{1, \mu^{-2}} n^{r-t-1},
\end{align}
where \newc\label{const_cnl_cnf_lower} is a positive constant.
\end{proof}

\subsection{Bounding the unfavored contamination}\label{sec:cnu-bounds}
Finally, we wrap up the section by proving matching upper and lower bounds for the unfavored contamination $\cnu$. 
\begin{lemma}[Bounding unfavored contamination]\label{prop:bound-unfavored-contamination}
In the same setting as \cref{prop:contamination-bound-accurate}, if $t > 0$ or $t = 0$ and $q+r>1$, with probability $1-O(1/nk)$, the contamination from the unfavored features satisfies 
\[
\cnu^2 = \const{const_quadratic_sparse}^2(1 \pm o(1))n^{1-t-p},
\]
where \const{const_quadratic_sparse} is a positive constant. 

On the other hand, if $t = 0$ and $q+r < 1$, then with probability $1 - O(1/nk)$, the unfavored contamination satisfies
\[
\cnu^2 \le \const{const_cnu_upper}^2 n^{1-t-p}\log(nsk).
\]
\end{lemma}
\begin{proof}
By Woodbury, we have 
\begin{align}
    \Ainv = \Auinv - \Auinv \mM_s \Auinv,
\end{align}
where 
\begin{align}
    \mM_s \triangleq \Xs(\mI_s + \mH_s)^{-1} \Xs^\top,
\end{align}
and $\mH_s \triangleq \Xs^\top \Asinv \Xs$.

Now we have 
\begin{align}
    \cnu^2 &= \dely^\top \Ainv \Au \Ainv \dely \\
    &= \dely^\top (\Auinv - \Auinv \mM_s \Auinv)\Au(\Auinv - \Auinv \mM_s \Auinv) \dely \\
    &= \dely^\top \qty(\Auinv - 2\Auinv \mM_s \Auinv + \Auinv \mM_s \Auinv \mM_s \Auinv) \dely.
\end{align}

By \cref{thm:bounded-hanson-wright-bilinear}, we have with probability at least $1-O(1/nk)$
\begin{align}
    \dely^\top \Auinv \dely = \const{const_quadratic_sparse}^2(1 \pm o(1))n^{1-t-p},
\end{align}
where \newc\label{const_quadratic_sparse} is a positive constant.

On the other hand, we have that 
\begin{align*}
    \mM_s \Auinv \mM_s &= \Xs (\mI_s+ \mH_s)^{-1} \Xs^\top \Auinv \Xs (\mI_s+ \mH_s)^{-1} \Xs^\top \\
    &= \Xs (\mI_s+ \mH_s)^{-1} \mH_s (\mI_s+ \mH_s)^{-1} \Xs^\top \\
    &= \Xs (\mI_s - (\mI_s + \mH_s)^{-1})(\mI_s+ \mH_s)^{-1} \Xs^\top \\
    &= \Xs ((\mI_s + \mH_s)^{-1} - (\mI_s + \mH_s)^{-2})\Xs^\top
\end{align*}
Due to \cref{prop:concentration-hat-matrix-simple}, $\mu_i((\mI_s + \mH_s)^{-1}) = \min\qty{\mu, 1}(1 \pm o(1))$ for all $i$ with very high probability. Hence to handle the error terms that are not $\dely^\top \Auinv \dely$, it suffices to asymptotically bound $\dely^\top \Auinv \mM_s \Auinv \dely$. In turn, we can couple this to the quadratic form $\min\qty{\mu, 1}(1 \pm o(1)) \dely^\top \Auinv \Xs \Xs \Auinv \dely$. By \cref{prop:worst-case-correlation}, we have with probability at least $1 - O(1/nk)$
\begin{align}
    &\min\qty{\mu, 1}(1 \pm o(1)) \dely^\top \Auinv \Xs \Xs \Auinv \dely \\
    &\quad\quad \le   \const{const_norm_bilinear}^2\lambda_F \min\qty{\mu, 1}(1 \pm o(1)) (n^{2-2t-2p} + n^{r+1-t-2p})\log(nsk) \\
    &\quad\quad \le \const{const_norm_bilinear}^2 \min\qty{1, \mu^{-1}}(1 \pm o(1))(n^{1-2t-p} + n^{r-t-p})\log(nsk) \label{eq:cnu-error-term}
\end{align}

For $t > 0$, we claim that the term in \cref{eq:cnu-error-term} is $o(n^{1-t-p})$, because $1-2t-p < 1-t-p$ and $r-t-p < 1-t-p$. Hence if $t > 0$ then by union bound we have with probability at least $1-O(1/nk)$ that
\begin{align}
    \cnu^2 &= \const{const_quadratic_sparse}^2(1 \pm o(1))n^{1-t-p},
\end{align}
as desired.
On the other hand, if $t = 0$ and $q + r > 1$, then $\min\qty{1, \mu^{-1}} \log(nsk) = \mu^{-1} \log(nsk) = o(1)$.
In this case, the deviation term is still negligible compared to $n^{1-t-p}$, and the above bound still holds.

Hence, we only have an issue if $t = 0$ and $q + r < 1$, so that $\min\qty{1, \mu^{-1}} = 1$. 
In this case, the deviation term is $\Tilde{O}(n^{1-2t-p}) = \Tilde{O}(n^{1-t-p})$. 
However, this won't affect the fact that the upper bound on contamination will still be $\Tilde{O}(n^{1-t-p})$. 
In summary, even if $t = 0$, we have  
\begin{align}
    \cnu^2 \le \const{const_cnu_upper}^2 n^{1-t-p}\log(nsk),
\end{align}
where \newc\label{const_cnu_upper} is an appropriately defined positive constant. 

It turns out we don't have to worry about this edge case at all for the lower bound on $\cnu$, because the stated conditions for misclassification imply that $t > 0$ anyway. This completes the proof of the lemma.
\end{proof}
\section{Obtaining tight misclassification rate}\label{sec:correlation-bounds}
In this section, we will prove \cref{prop:correlation-bound-tight}.
Let us restate the main proposition and sketch out its proof more formally.


\correlation*
\begin{proof}[Proof sketch]
Note that the $\Zb$'s that must outcompete the decaying survival to contamination ratio are jointly Gaussian, as they are projections of a standard Gaussian vector $\xtest \in \RR^d$. Hence if we want to study the probability that $\max_\b \Zb$ outcompetes $n^{-u}$, we have to understand the correlation structure of the $\Zb$'s.

We will argue that for $\b, \gamma \in [k]$ with $\a, \b, \gamma$ pairwise distinct, the correlation between $\Zb$ and $Z^{(\gamma)}$ is $\frac{1}{2} \pm o(1)$ with high probability. To that end, we want to look at the correlation (inner product) between the vectors $\qty{\lambda_j\hhat_{\a, \b}[j]}$ for $j \not\in \qty{\a, \b}$ and $\qty{\lambda_j\hhat_{\a, \gamma}[j]}$ for $j\not\in \qty{\a, \gamma}$. However, note that by independence of the components of $\xtest$ from every other random variable and the fact that they are mean zero, we have
\[
\EE[\hhat_{\a, \b}[\gamma]\xtest[\gamma] \hhat_{\a, \gamma}[\b] \xtest[\b]] = 0.
\]
Hence it suffices to look at the correlation for $j \not\in \qty{\a, \b, \gamma}$. 

We assume WLOG that $\a = 1, \b = 2$, $\gamma = 3$. Let
\[
\lambdaab \triangleq\diag(1 - \1_{j = \a} - \1_{j = \b})_{j \in [d]} \circ \diag(\lambda_j)_{j \in [d]} \in \RR^{d \times d} 
\]
represent the diagonal matrices containing the squared feature weights with indices $\a, \b$ zeroed out.
Next, let $\vv_{\a, \b} \in \RR^{d}$ denote the vector with $\vv_{\a, \b}[\a] = \vv_{\a, \b}[\b] = 0$ and $\vv_{\a, \b}[j] = \lambda_j \hhat_{\a, \b}[j] $ for $j \in [d], j\not\in \qty{\a, \b}$. Hence $\vv_{\a, \b} = \lambdaab^{1/2} (\fhat_{\a} - \fhat_{\b})$. Since $Z^{(\b)} = \ev{\vv_{\a, \b}, \xtest}$, in order to analyze the correlations between $Z^{(\b)}$ and $Z^{(\gamma)}$, it suffices to analyze $\vv_{\a, \b}$. Indeed, we will show that the weighted halfspaces $\lambdaab^{1/2} \fhat_{\a} \in \RR^d$ and $\lambdaab^{1/2} \fhat_{\b} \in \RR^d$ are asymptotically orthogonal.

In other words, we need to show that 
\[
\frac{\ev{\lambdaab^{1/2} \fhat_{\a}, \lambdaab^{1/2} \fhat_{\b}}}{\norm{\lambdaab^{1/2} \fhat_{\a}}_2\norm{\lambdaab^{1/2} \fhat_{\b}}_2} = o(1)
\]
with probability at least $1-O(1/nk)$; we can then union bound against all choices of $\b$. This is the most technically involved part of the proof, and is the content of \cref{prop:asymptotic-orthogonal}.

This in turn will imply (see \cref{lemma:almost-orthogonal-correlation}) that the maximum (and minimum) correlation between the $\vv_{\a, \b}$ for different $\b$ is $\frac{1}{2} \pm o(1)$. Let $(\ol{Z}_\b)_{\b \in [k], \b \neq \a}$ be equicorrelated gaussians with correlation $\ol{\rho} = \frac{1}{2} + o(1)$, and $(\ul{Z}_\b)_{\b \in [k], \b \neq \a}$ be equicorrelated gaussians with correlation $\ul{\rho} = \frac{1}{2} - o(1)$. By Slepian's lemma, for any $u > 0$, the probability of $\max_\b \Zb$ losing to $n^{-u}$ is sandwiched as 
\begin{align*}
    \Pr[\max_\b \ul{Z}_\b \le n^{-u}] \le \Pr[\max_\b \Zb \le n^{-u}] \le \Pr[\max_\b \ol{Z}_\b \le n^{-u}],
\end{align*}
where we have adopted the shorthand $\max_\b$ to denote $\max_{\b \in [k], \b \neq \a}$.

Theorem 2.1 of \citet{pinasco2021orthant} shows that jointly gaussian vectors in $\RR^k$ with equicorrelation $\rho$ lie in the positive orthant with probability $\Theta(k^{1 - 1/\rho})$. In particular, applied to $\ol{Z}_\b$, with correlation $\ol{\rho} = \frac{1}{2} + o(1)$, we find that 
\begin{align*}
    \Pr[\max_\b \ol{Z}_\b \le 0] = \Theta(k^{-1 + o(1)}),
\end{align*}
and similarly for $\ul{Z}_\b$. Anticoncentration for Gaussian maxima \citep[Corollary~1]{chernozhukov2015comparison} implies that we can transfer over the bound on $\Pr[\max_\b \ol{Z}_\b \le 0]$ to a bound on $\Pr[\max_\b \ol{Z}_\b \le n^{-u}]$ to show that for every $\epsilon > 0$, we have 
\begin{align}
    \Theta(k^{-1 - o(1)}) - \epsilon \le \Pr[\max_\b Z^{(\b)} \le n^{-u}] \le \Theta(k^{-1 + o(1)}) + \epsilon
\end{align}
for sufficiently large $n$. Taking the complement of the above event concludes the proof.
\end{proof}

\subsection{Main results for tight misclassification rates}
The main result in this section is the following proposition, which states that the halfspace predictions are asymptotically orthogonal. Its proof is deferred to the subsequent sections.
\begin{proposition}\label{prop:asymptotic-orthogonal}
Assume we are in the bi-level ensemble model (\cref{def:bilevel}), the true data generating process is 1-sparse (Assumption~\ref{assumption:1sparse}), and the number of classes scales with $n$ (i.e. $t > 0$).

For any distinct $\a, \b \in [k]$, with probability at least $1-O(1/nk)$, we have 
    \[
\frac{\ev{\lambdaab^{1/2} \fhat_{\a}, \lambdaab^{1/2} \fhat_{\b}}}{\norm{\lambdaab^{1/2} \fhat_{\a}}_2\norm{\lambdaab^{1/2} \fhat_{\b}}_2} = o(1).
\]
\end{proposition}

Given \cref{prop:asymptotic-orthogonal}, we can show that the $Z^{(\b)}$ have correlations that approach $\frac{1}{2}$. The intuitive reason that this correleation approaches $\frac{1}{2}$ is that the contribution from $\a$ is common. The following lemma formalizes this intuition.
\begin{lemma}[Correlation of relative differences of almost orthogonal vectors]\label{lemma:almost-orthogonal-correlation}
Suppose that we have $n$ unit vectors $\vx_1, \ldots, \vx_n \in \RR^d$ such that $\abs{\ev{\vx_i, \vx_j}} \le \gamma$ for $\gamma > 0$. Then for any distinct $i, j, k \in [n]$, we have 
\[
\abs{\frac{\ev{\vx_j - \vx_i, \vx_i - \vx_k}}{\norm{\vx_j - \vx_i}\norm{\vx_i - \vx_k}} - \frac{1}{2}} \le  \frac{2\gamma}{1-\gamma}.
\]
\end{lemma}
\begin{proof}
For any $i \neq j$, we have $\norm{\vx_i - \vx_j}^2 = 2 - 2\ev{\vx_i, \vx_j}$. Hence we have 
\[
2 - 2\gamma \le \norm{\vx_i - \vx_j}^2 \le 2 + 2\gamma. 
\] Also
\begin{align*}
2 - 2\gamma &\le \norm{\vx_j - \vx_k}^2 \\
&= \norm{\vx_i - \vx_j}^2 + \norm{\vx_i - \vx_k}^2 - 2\ev{\vx_j - \vx_i, \vx_i - \vx_k} \\
&\le 4 + 4\gamma - 2\ev{\vx_j - \vx_i, \vx_i - \vx_k}.
\end{align*}
Since $\norm{\vx_i - \vx_j} \ge \sqrt{2-2\gamma}$, we can rearrange and obtain that 
\[
\frac{\ev{\vx_j - \vx_i, \vx_i - \vx_k}}{\norm{\vx_j - \vx_i}\norm{\vx_i - \vx_k}} \le \frac{1 + 3\gamma}{2-2\gamma}.
\]
Similarly we can reverse the inequalities and get 
\[
2 + 2\gamma \ge 4 - 4\gamma - 2\ev{\vx_j - \vx_i, \vx_i - \vx_k},
\]
so 
\[
\frac{\ev{\vx_j - \vx_i, \vx_i - \vx_k}}{\norm{\vx_j - \vx_i}\norm{\vx_i - \vx_k}} \ge \frac{1-3\gamma}{2+2\gamma}.
\]
\end{proof}

Combining \cref{prop:asymptotic-orthogonal} with \cref{lemma:almost-orthogonal-correlation} yields the following formal statement about the correlations between the $Z^{(\b)}$.
\begin{lemma}[Asymptotic correlation of relative survivals]\label{lemma:asymptotic-correlation-relative-survival}
For any distinct $\alpha, \beta, \beta' \in [k]$, under the same assumptions as \cref{prop:asymptotic-orthogonal}, as $n \to \infty$, with probability at least $1 - O(1/n)$, we have
\[
\abs{\EE[Z^{(\b)}Z^{(\b')}] - \frac{1}{2}} \le o(1).
\]
As a consequence, the asymptotic correlation between the relative survivals approaches $\frac{1}{2}$ at a polynomial rate. 
\end{lemma}
\begin{proof}
Plugging in the result of \cref{prop:asymptotic-orthogonal} into \cref{lemma:almost-orthogonal-correlation}, we obtain the stated result.
\end{proof}

\subsection{Lower bounding the denominator}
Let us now begin to prove \cref{prop:asymptotic-orthogonal}. The first step is to bound the denominator of the normalized correlation.
Writing out the definitions, we have 
\begin{align*}
    \norm{\lambdaab^{1/2} \fhat_{\a}}^2 &=  \sum_{j \not\in \qty{\a, \b}} \lambda_j^2 \ya^\top \Ainv \vz_j \vz_j^\top \Ainv \ya \\
    &= \lambda_F^2 \sum_{j \not\in \qty{\a, \b}, j\in[s]} \ya^\top \Ainv \vz_j \vz_j^\top \Ainv \ya + \lambda_U^2 \sum_{j > s} \ya^\top \Ainv \vz_j \vz_j^\top \Ainv \ya 
\end{align*}
Note that these two terms are respectively analogous to $\cnl^2 + \cnf^2$ and $\cnu^2$. In fact, the proofs of the lower bounds for contamination essentially transfer over verbatim to the lower bounds on the denominator, because Hanson-Wright implies that we can show that $\norm{\Asinv \ya}_2$ concentrates the same way that $\norm{\Asinv \dely}_2$ does. In essence, we are able to show the following proposition.
\begin{proposition}[Lower bound on norm of scaled halfspaces]
Under the same assumptions as \cref{prop:asymptotic-orthogonal}, for any $\a, \b \in [k]$, with $\a \neq \b$, with probability at least $1 - O(1/nk)$, we have 
\begin{align*}
     \norm{\lambdaab^{1/2} \fhat_{\a}}^2 \ge \min\qty{1, \mu^{-2}}\Omega(n^{r-t-1}) + \Omega(n^{1-t-p}).
\end{align*}
\end{proposition}

\subsection{Upper bounding the numerator: the unnormalized correlation}
We now turn to the more involved part of the bound: proving an upper bound on the numerator. As before, we can bound the split up the numerator into favored and unfavored terms. For each term, we will show that it is dominated by the denominator, in the precise sense that each term is 
\begin{align*}
    o(\min\qty{1, \mu^{-2}} n^{r-t-1} + n^{1-t-p}).  
\end{align*}

Now, let's look at the numerator, which is the bilinear form 
\begin{align}
    &\lambda_F^2 \sum_{j \not\in \qty{\a, \b}, j \in [s]} \ya^\top \Ainv \vz_j \vz_j^\top \Ainv \yb + \lambda_U^2 \sum_{j > s} \ya^\top \Ainv \vz_j \vz_j^\top \Ainv \yb. \\
    &\quad= \underbrace{\lambda_F^2 \ev{\Zl^\top \Ainv \ya, \Zl^\top \Ainv \yb}}_{\corl} + \underbrace{\lambda_F^2 \ev{\Zsk^\top \Ainv \ya, \Zsk^\top \Ainv \yb}}_{\corf} + \underbrace{\lambda_U^2 \sum_{j > s} \ya^\top \Ainv \vz_j \vz_j^\top \Ainv \yb}_{\coru}
\end{align}
We refer to the the first term as the label defining correlation $\corl$, the second term as the favored correlation $\corf$, and the last term as the unfavored correlation $\coru$. Here, we abuse terminology slightly and refer to these inner products as \emph{correlations}, even though strictly speaking, they are unnormalized. 

\subsubsection{Bounding the favored correlation} 
We now bound the correlation coming from the favored features; we will ultimately show that its contribution is $\min\qty{1, \mu^{-2}} o(n^{r-t-1})$. Recall that $\Xr \in \RR^{n \times (s-2)}$ is the weighted feature matrix for the $s-2$ favored features aside from $\a$ and $\b$. Then the label-defining+favored correlation $\corl + \corf$ can be written succinctly as 
\begin{align}
    \lambda_F^2\ev{\Zr^\top \Ainv \ya, \Zr^\top \Ainv \yb}.
\end{align}
Why should we be able to bound this better than Cauchy-Schwarz? Intuitively, although there is a mild dependence between $\ya$ and $\yb$, it is not strong enough to cause $\Zr^\top \Ainv \ya$ and $\Zr^\top \Ainv \yb$ to point in the same direction. 

To formalize this argument, we will first follow the strategy to bound the favored \emph{contamination}. In particular, using the push-through form of Woodbury (\cref{fact:woodbury}) we see that
\begin{align}
    \ev{\Zr^\top \Ainv \ya, \Zr^\top \Ainv \yb} &= \ya^\top \Arinv \Zr (\mI_{s-2} + \mH_{R})^{-2} \Zr^\top \Arinv \yb.
\end{align}
Now, we can apply \cref{prop:concentration-hat-matrix-simple} to replace $(\mI_{s-2} + \mH_{R})^{-2}$ with $\min\qty{\mu^2, 1}(\mI_{s-2} + \mE)$, where $\norm{\mE}_2 = O(n^{-\konst{kappa_hatts_final}})$ with extremely high probability. Cauchy-Schwarz yields that
\begin{align}
    &\ev{\Zr^\top \Ainv \ya, \Zr^\top \Ainv \yb} \\
    &\quad\quad \le \min\qty{\mu^2, 1}\ev{\Zr^\top \Arinv \ya, \Zr^\top \Arinv \yb}\label{eq:favored-corr-bernstein} \\
    &\quad\quad\quad + \min\qty{\mu^2, 1} \norm{\mE}_2 \norm{\Zr^\top \Arinv \yb}_2 \norm{\Zr^\top \Arinv \ya}_2. \label{eq:favored-corr-cs}
\end{align}
The term in \cref{eq:favored-corr-cs} can be bounded in the same way that we bounded the favored contamination. Indeed, since we can swap in $\ya$ and $\yb$ with $\dely$, the argument that proved the bounds on $\norm{\Zr^\top \Arinv \dely}_2$ port over immediately. After using the scaling for $\lambda_F$ and the fact that $\norm{\mE}_2 = O(n^{-\konst{kappa_hatts_final}})$, we conclude that this Cauchy-Schwarz error term is at most $\min\qty{1, \mu^{-2}}o(n^{r-t-1})$ with probability at least $1 - O(1/nk)$. 

Let us now turn to the term in \cref{eq:favored-corr-bernstein}. As in the proof for the lower bound for favored contamination \cref{lemma:cnl-cnf-lower-bound}, to get better concentration than Cauchy-Schwarz, we want to use Bernstein. We can rewrite it suggestively as 
\begin{align}
    \sum_{j \in [s] \setminus \qty{\a, \b}} (\vz_j^\top \Arinv \ya)(\vz_j^\top \Arinv \yb) \label{eq:bernstein-decomp}
\end{align}
We cannot immediately power through with the calculation, because $\Arinv$ is not independent of $\ya$ or $\yb$. The main idea is to again use Woodbury and show that the dependent portions contribute negligibly to $\vz_j^\top \Arinv \ya$. Therefore the dependent contributions get dominated by the lower bound on the correlation.

As in the proof for bounding the favored contamination, we can further define $\Xab = \mqty[\vw_\a & \vw_\b]$ and $\mH_{\a, \b}^{(s)} = \Xab^\top \Asinv \Xab$. Then we can deduce from another application of Woodbury that 
\begin{align}
    \vz_j^\top \Arinv \ya = \vz_j^\top \Asinv \ya - \vz_j^\top \Asinv \Xab (\mI_2 + \mH_{\a, \b})^{-1} \Xab^\top \Asinv \ya. 
\end{align}
Again, we can argue that the second term is bounded in magnitude by \begin{align}
    \min\qty{1, \mu^{-1}} n^{\frac{1}{2} - t - p}\log(ns) = \min\qty{1, \mu^{-1}} O(n^{\frac{1-t}{2} - p} \cdot n^{-\konst{kappa_bernstein_error}}),
\end{align}
because $t > 0$. Since Hanson-Wright (\cref{corollary:asymptotic-concentration}) implies that $\vz_j^\top \Asinv \ya = \Tilde{O}(n^{\frac{1-t}{2} - p})$, this implies that $\vz_j^\top \Arinv \ya = \Tilde{O}(n^{\frac{1-t}{2} - p})$, and similarly for $\b$. Hence we have 
\begin{align}
    &(\vz_j^\top \Arinv \ya)(\vz_j^\top \Arinv \yb) \\
    &\quad\le (\vz_j^\top \Asinv \ya + \min\qty{1, \mu^{-1}}O(n^{\frac{1-t}{2} -p -\konst{kappa_bernstein_error}}))(\vz_j^\top \Asinv \yb + \min\qty{1, \mu^{-1}}O(n^{\frac{1-t}{2} - p -\konst{kappa_bernstein_error}})) \\
    &\quad\le (\vz_j^\top \Asinv \ya)(\vz_j^\top \Asinv \yb) + \min\qty{1, \mu^{-1}}O(n^{\frac{1-t}{2} - p -\konst{kappa_bernstein_error}})\Tilde{O}(n^{\frac{1-t}{2} - p}) \\
    &\quad\le (\vz_j^\top \Asinv \ya)(\vz_j^\top \Asinv \yb) + \min\qty{1, \mu^{-1}}o(n^{1-t-2p}).
\end{align}

This implies that we can rewrite \cref{eq:bernstein-decomp} as 
\begin{align}
    \qty(\sum_{j \in [s] \setminus \qty{\a, \b}} (\vz_j^\top \Asinv \ya)(\vz_j^\top \Asinv \yb)) \pm \min\qty{1, \mu^{-1}}o(n^{r+1-t-2p}). \label{eq:bernstein-decomp-expanded}
\end{align}
Let us argue that the second term in \cref{eq:bernstein-decomp-expanded} will be negligible compared to the denominator, which is $\min\qty{1, \mu^{-2}} \Omega(n^{r-t-1})$. Tracing back up the stack, we see that its contribution to the favored correlation will be at most 
\begin{align}
    \lambda_F^2 \min\qty{\mu^2, 1} \cdot \min\qty{1, \mu^{-1}}o(n^{r+1-t-2p}) &\le \mu^{-2}\min\qty{\mu^{2}, \mu^{-1}} o(n^{r-t-1}) \\
    &\le \min\qty{1, \mu^{-3}} o(n^{r-t-1}) \\
    &\le \min\qty{1, \mu^{-2}} o(n^{r-t-1}).
\end{align}

Turning back to the first term, we are now in a position to apply Bernstein. Note that $(\Asinv, \ya, \yb)$ are independent of $\Zr$. Hence conditioned on $\Asinv, \ya$, and $\yb$, $\ev{\Zr^\top \Asinv \ya, \Zr^\top \Asinv \yb}$ is a sum of $s-2$ subexponential variables, and by Lemma 2.7.7 of \citet{vershynin2018high} each of these random variables conditionally has subexponential norm at most $\norm{\Asinv \ya}_2\norm{\Asinv \yb}_2$ and conditional mean $\ev{\Asinv \ya, \Asinv \yb}$. 

We can use Hanson-Wright (\cref{thm:bounded-hanson-wright-bilinear}) to bound both of these quantities. Indeed, it implies that with probability at least $1 - O(1/nk)$, 
\begin{align}
    \norm{\Asinv \ya}_2\norm{\Asinv \yb}_2 \le O(n^{1-t-2p}).
\end{align}
Let us now compute the Hanson-Wright bound for $\ev{\Asinv \ya, \Asinv \yb}$. Note that $\Asinv$ is independent of $\ya$ and $\yb$, so we can condition on $\Asinv$ and conclude that with probability at least $1 - O(1/nk)$
\begin{align}
    \ev{\Asinv \ya, \Asinv \yb} \le \EE[\ev{\Asinv \ya, \Asinv \yb} | \Asinv] + \const{const_asymptotic_bilinear}n^{\frac{1-t}{2}}\norm{\mA_{-s}^{-2}}_2 \sqrt{\log(nk)}.
\end{align}
We can rewrite the expectation as 
\begin{align}
    \Tr(\mA_s^{-2} \EE[\yb \ya^\top]).
\end{align}
Clearly, $\EE[\yb \ya^\top]$ is diagonal, and each diagonal entry is equal to . Let $\rho = \frac{1}{k}$. Then since $\ya = 1-\rho$ implies $\yb = -\rho$ and vice versa, we get 
\begin{align}
    \EE[\ya[i]\yb[i]] &= 2(1-\rho)(-\rho)\Pr[\ya[i] = 1 - \rho] + (-\rho)^2 \Pr[\ya[i] = \yb[i] = -\rho] \\
    &\le -2\rho^2(1-\rho) + \rho^2\Pr[\ya[i] = -\rho]\\
    &\le -\rho^2(1-\rho).
\end{align}
In other words, the expectation is \emph{negative}, so we can neglect it in our upper bound.

On the other hand, the deviation term is with very high probability at most 
\begin{align}
    \const{const_asymptotic_bilinear}n^{\frac{1-t}{2}}\norm{\mA_{-s}^{-2}}_2 \sqrt{\log(nk)} &\le \const{const_asymptotic_bilinear} n^{\frac{1-t}{2} - 2p}\sqrt{\log(nk)}.
\end{align}

Combining all of our bounds, Bernstein yields
\begin{align}
    \ev{\Zr^\top \Asinv \ya, \Zr^\top \Asinv \yb} &\le  \qty(\sum_{j \in [s] \setminus \qty{\a, \b}} \const{const_asymptotic_bilinear} n^{\frac{1-t}{2}-2p}\sqrt{\log(nk)}) + n^{\frac{r}{2} + 1-t-2p} \\
    &\le n^{r + \frac{1-t}{2} -2p}\sqrt{\log(nk)} + n^{\frac{r}{2} - t + 1-2p}.
\end{align}

Again, let's trace all the way back to \cref{eq:favored-corr-bernstein} and then the favored correlation bound. We have shown that $\ev{\Zr^\top \Asinv \ya, \Zr^\top \Asinv \yb}$'s contribution to the favored correlation is at most 
\begin{align}
    &\const{const_bernstein_final_contribution}\lambda_F^2 \min\qty{\mu^2, 1} (n^{r + \frac{1-t}{2} -2p}\sqrt{\log(nk)} + n^{\frac{r}{2} - t + 1-2p}) \\
    &\quad\quad = \const{const_bernstein_final_contribution} \mu^{-2} \min\qty{\mu^2, 1} (n^{r - \frac{1+t}{2} - 1}\sqrt{\log(nk)} + n^{\frac{r}{2} - t - 1}) \\
    &\quad\quad \le \const{const_bernstein_final_contribution} \min\qty{1, \mu^{-2}} (n^{r - \frac{1+t}{2} - 1}\sqrt{\log(nk)} + n^{\frac{r}{2} - t - 1}) \\
    &\quad\quad \le \const{const_bernstein_final_contribution} \min\qty{1, \mu^{-2}} o(n^{r-t-1}),
\end{align}
where the last line follows becuase $0 < t < r < 1$, and \newc\label{const_bernstein_final_contribution} is a positive constant. 

\subsubsection{Bounding the unfavored correlation}
Now, let us show that the unfavored correlation $\coru$ is negligible; more precisely, we'll show that it's $\min\qty{1, \mu^{-2}} o(n^{r-t-1}) + o(n^{1-t-p})$. We can rewrite $\coru$ as
\[
\lambda_U \ya^\top \Ainv \Au \Ainv \yb, 
\]
and play the same game with using Woodbury to replace $\Ainv$ with $\Ainv - \Auinv \mM_s \Auinv$, where we recall that 
\[
\mM_s \triangleq \Xs(\mI_s + \mH_s)^{-1} \Xs^\top \in \RR^{n \times n}.
\]
This yields 
\begin{align}
    \ya^\top \Auinv \yb - 2\ya^\top \Auinv \mM_s \Auinv \yb + \ya^\top \Auinv \mM_s \Auinv \mM_s \Auinv \yb. \label{eq:unfavored-corr-expansion}
\end{align}

Let us first focus on the first term of \cref{eq:unfavored-corr-expansion}.  Hanson-Wright implies that with probability at least $1 - O(1/n)$, \begin{align}
    \abs{\ya^\top \Auinv \yb - \EE[\ya^\top \Auinv \yb | \Auinv]} \le n^{\frac{1-t}{2}} \norm{\Auinv}_2\sqrt{\log n}.
\end{align}
Also, $\EE[\ya^\top \Auinv \yb | \Auinv] = \Tr(\Auinv \EE[\yb \ya^\top]) = \Theta(n^{1-2t-p})$ with high probability, and $\norm{\Auinv}_2 \le n^{-p}$ with extremely high probability. Hence we see that $\ya^\top \Auinv \yb \le O(n^{\frac{1-t}{2}-p}) \le o(n^{1-t-p})$ as $t > 0$.

Next, let's turn to the second and third terms of \cref{eq:unfavored-corr-expansion}. We claim that only the second term will be relevant to bound asymptotically, and moreover that they are both $\min\qty{1, \mu^{-2}}o(n^{r-t-1})$. 
Since 
\begin{align*}
    \mM_s \Auinv \mM_s &= \Xs (\mI_s + \mH_s)^{-1} \mH_s (\mI_s + \mH_s)^{-1} \Xs^\top \\
    &= \Xs ((\mI_s + \mH_s)^{-1} - (\mI_s + \mH_s)^{-2}) \Xs^\top,
\end{align*} the second and third term can be rewritten as 
\begin{align*}
    &- 2\ya^\top \Auinv \Xs (\mI_s + \mH_s)^{-1} \Xs^\top \Auinv \yb \\
    &\quad\quad + \ya^\top \Auinv \Xs ((\mI_s + \mH_s)^{-1} - (\mI_s + \mH_s)^{-2}) \Xs^\top \Auinv \yb.
\end{align*}

As we are going to use Hanson-Wright to bound the entries of $\mZ_s^\top \Auinv \ya$, it follows that only the second term of \cref{eq:unfavored-corr-expansion} is relevant asymptotically. 

To bound the second term, we will use Cauchy-Schwarz. We see that
\begin{align}
    \ya^\top \Auinv \mM_s \Auinv \yb &\le \lambda_F \norm{(\mI_s + \mH_s)^{-1}}_2 \norm{\mZ_s^\top \Auinv \ya}_2 \norm{\mZ_s^\top \Auinv \yb}_2 \\ 
    &\le \lambda_F \min\qty{\mu, 1} O(n^{r-t+1 - 2p}) \log(ns) \\
    &\le \mu^{-1} \min\qty{\mu, 1} O(n^{r-t-p}) \log(ns) \\
    &\le \min\qty{1, \mu^{-1}} O(n^{r-t-p})\log(ns),
\end{align}
where in the second line we have used \cref{prop:concentration-hat-matrix-simple}. 

Now, note that if regression works, this yields an upper bound of $O(n^{r-t-p})\log(ns)$. But since $p > 1$, this is $o(n^{r-t-1})$, which means this contribution is dominated by the denominator. 

On the other hand, if regression fails, then the upper bound is now $\mu^{-1} O(n^{r-t-p})\log(ns)$, which we claim is $o(\mu^{-2} n^{r-t-1})$. Indeed, from the definition of the bi-level ensemble \cref{def:bilevel}, we have $p > q + r$, so 
\begin{align*}
    \min\qty{1, \mu^{-1}} n^{r-t-p} &\le \mu^{-1}  n^{r-t-1} \cdot n^{1-p} \\
    &\le \mu^{-1} o(n^{r-t-1} \cdot n^{1-q-r}) \\
    &=  \mu^{-2} o(n^{r-t-1}),
\end{align*}
as desired.

Let us now go back to \cref{eq:unfavored-corr-expansion} and combine our two bounds. Since $\lambda_U = O(1)$, we have just shown that 
\begin{align}
    \coru \le \min\qty{1, \mu^{-2}} o(n^{r-t-1}) + o(n^{1-t-p}),
\end{align}
as desired.

\section{Comparison to the straightforward non-interpolative scheme}\label{sec:averaging}

In this section, we quickly give calculations for how well a straightforward non-interpolating scheme for learning classifiers can work asymptotically. However, a similar analysis using the tools developed to prove our main results should give a rigorous proof of the below derivation. 

This scheme simply uses the sum/average of all positive training examples of a class as the vector we take an inner-product with to generate scores for classifying test points. For $m \in [k]$, define
\begin{align}
    \fhat_m = \sum_{i: \ell_i = m} \vx^w_i.
\end{align}

To understand how well this will do asymptotically, it is easy to see that the for the true label-defining direction, the positive exemplars in the bi-level model will be tightly concentrating around $\sqrt{2 \log k} \sqrt{\lambda_F}$ which, keeping only the polynomial-order scaling, will be like $n^{\frac{p - q - r}{2}}$. There will be roughly $\frac{n}{k} = n^{1-t}$ positive examples for every class with high probability. For simplicity, let us just look at $m=1$ and consider $\frac{k}{n}\fhat_1 = n^{t-1}\fhat_1$. We see 
\begin{align}
    n^{t-1}\fhat_1[1] = \Theta(n^{\frac{p - q - r}{2}}).
\end{align}

For the other directions that are not true-label defining, we will just have random Gaussians. The favored directions will be Gaussian with variance $\lambda_F = n^{p-q-r}$ while the unfavored directions will essentially be Gaussian with unit variance. By averaging over $n^{1-t}$ examples, those variances will be reduced by that factor. This means that for the $s = n^r$ favored directions, the variance of the average will be $n^{p-q-r-(1-t)}$ each and for the essentially $n^p$ unfavored directions, the variance of the average will be $n^{t-1}$ each.

On a test point, we are going to take the inner product of $n^{t-1}\fhat_m$ with an independent random draw of $\xtest^w$. 
For classification to succeed, we need this inner product to be dominated by the true $m$-th feature-defining direction. 
When that happens, the correct label will win the comparison.
One can easily see that the contribution from the true feature-defining direction will be a Gaussian with mean 0 and variance $\lambda_F \cdot (n^{\frac{p - q - r}{2}})^2 = \lambda_F^2 = n^{2p -2q -2r}$. 
Meanwhile, the $s$ favored features will have their scaled variances sum up in the score to give a total variance of $n^r \cdot \lambda_F \cdot n^{p-q-r-(1-t)} = n^{2p-2q-r-(1-t)}$. 
And finally, the unfavored features will also have their variances sum up in the score to give a total variance of $n^p \cdot 1 \cdot n^{t-1} = n^{p + t - 1}.$

For the true-feature-defining direction to dominate the contamination from other favored directions, we need 
\begin{align}
    2p - 2q -2r > 2p -2q -r - (1-t)
\end{align}
which immediately gives the condition $t < 1-r$.

For the true-feature-defining direction to dominate the contamination from other unfavored directions, we need 
\begin{align}
    2p - 2q -2r > p + t -1
\end{align}
which gives the condition $t < p + 1 - 2(q+r)$.

Here, there is no difference between regimes in which regression works or does not work. The condition for classification to asymptotically succeed is $t < \min(1-r, p+1-2(q+r))$. 

Notice that when MNI regression does not work $q+r >1$, this is identical to the tight characterization given for MNI classification in \eqref{eq:conjectured_regimes}. But in the regime where MNI regression {\em does} work $q + r < 1$, this is different. For MNI classification, \eqref{eq:conjectured_regimes} tells us that we require $t < \min(1-r, p -1)$. Consider $q = 0.1, r = 0.5$ and $p = 1.1$. MNI classification can only allow  $t < 0.1$. Meanwhile, the non-interpolating average-of-positive-examples classifier will work as long as $t < 0.5$. This demonstrates the potential for significant suboptimality (in terms of the number of distinct classes that can be learned) of MNI classifiers in this regime of benign overfitting for regression.

\section{A new variant of the Hanson-Wright inequality with soft sparsity}\label{sec:hanson-wright-bounded}
In this section, we prove \cref{thm:bounded-hanson-wright-bilinear}. First, we outline a high level idea of the proof.
The starting point of the proof is to explicitly decompose the quadratic form into diagonal and off-diagonal terms
\begin{align}
    \vx^\top \mM \vy - \EE[\vx^\top \mM \vy] &= \sum_{i, j} m_{ij} X_i Y_j - \sum_i m_{ii}\EE[X_iY_i] \\
    &= \underbrace{\sum_{i} m_{ii}(X_iY_i - \EE[X_iY_i])}_{\triangleq S_{\diag}}+ \underbrace{\sum_{i \neq j} m_{ij}X_iY_j}_{\triangleq S_{\offdiag}}  
\end{align}
where in the first line we have used the fact that for $i \neq j$, $X_i$ and $Y_j$ are independent and mean zero to conclude that $\EE[X_iY_j] = 0$.

We can start with the upper tail inequality $\PP[\vx^\top \mM \vy - \EE[\vx^\top \mM \vy] > t]$ and conclude the lower tail inequality by replacing $\mM$ with $-\mM$. To bound $S_\diag$ and $S_\offdiag$, we will proceed by explicitly bounding the MGF and applying Chernoff's inequality. 

\subsection{Diagonal terms}
For the diagonal terms, we want to bound the MGF of $S_\diag = \sum_i m_{ii} (X_iY_i - \EE[X_iY_i])$. For $\lambda^2 < \frac{1}{2\Const{c_quadratic_taylor}K^2 \max_i m_{ii}^2}$, we obtain 
\begin{align}
    \EE \exp(\lambda S_\diag) &= \prod_{i=1}^n \EE_{X_i, Y_i} \exp(\lambda m_{ii} (X_iY_i - \EE[X_iY_i])) \\
    &\le \prod_{i=1}^n \EE_{Y_i} \EE_{X_i}\qty[\exp(\lambda m_{ii}Y_i(X_i - \EE[X_i|Y_i])) | Y_i] \\
    &\le \prod_{i=1}^n  \EE_{Y_i} \exp(\Const{c_quadratic_taylor}\lambda^2 m_{ii}^2 K^2 Y_i^2)
\end{align} 
where we have applied Jensen's inequality in the second line and the subgaussian assumption on $X_i$ conditioned on $Y_i$ in the last line. Here, \newC\label{c_quadratic_taylor} is a universal positive constant relating the equivalent formulations of subgaussianity \citet{vershynin2018high}. Continuing with our calculation, we have 
\begin{align} 
\EE \exp(\lambda S_\diag) &\le \prod_{i=1}^n \EE_{Y_i} [1 + 2\Const{c_quadratic_taylor}\lambda^2 m_{ii}^2 K^2 Y_i^2]  \\
    &\le \prod_{i=1}^n (1 + 2\Const{c_quadratic_taylor}\pi\lambda^2 K^2 m_{ii}^2) \\
    &\le \exp(2\Const{c_quadratic_taylor}\pi \lambda^2 K^2\sum_{i=1}^n m_{ii}^2).
\end{align}
where in the first line we have used the inequality $\exp(x) \le 1 + 2x$ valid for $x \le \frac{1}{2}$, in the second line we have used the soft sparsity assumption on $Y_i$, and in the last line we have used the inequality $1 + x \le \exp(x)$, valid for all $x$. 

Now Markov's inequality yields for $\epsilon > 0$ that 
\begin{align}
    \Pr[S_\diag > \epsilon] &\le \frac{\EE \exp(\lambda S_\diag)}{\exp(\lambda \epsilon)} \\
    &\le \exp(-\lambda \epsilon + 2\pi\Const{c_quadratic_taylor}K^2\lambda^2 \sum_{i=1}^n m_{ii}^2),
\end{align}
and optimizing $\lambda$ in the region $\lambda^2 \le \frac{1}{2\Const{c_quadratic_taylor}K^2\max_i m_{ii}^2}$ yields 
\begin{equation}
    \lambda = \min\qty{\frac{\epsilon}{2\Const{c_quadratic_taylor}K^2 \pi\sum_{i=1}^n m_{ii}^2}, \frac{1}{2\Const{c_quadratic_taylor}K \max_i \abs{m_{ii}}}}.
\end{equation}

Plugging in this value of $\lambda$ into the Markov calculation yields the desired upper tail bound. We can repeat the argument with $-\mM$ to get the lower tail bound. A union bound completes the proof.

\subsection{Offdiagonal terms}
Following \citet{rudelson2013hanson}, for the offdiagonal terms we can decouple the terms in the sum. More precisely, the terms in $S_\offdiag$ involving indices $i$ and $j$ are precisely $m_{ij}X_iY_j + m_{ji}Y_iX_j$. The issue is that $Y_i$ can be correlated with $X_i$, which complicates the behavior of this random variable. Decoupling ensures that for any $j \in [n]$ we will have exactly one term which involves either $X_j$ or $Y_j$, so in particular we will regain independence of the terms, allowing us to bound the MGF more easily.

Let $\qty{\delta_i}_{i \in [n]}$ denote iid Bernoulli's with parameter $1/2$, which are independent of all other random variables. 

Let 
\[
S_\delta \triangleq \sum_{i \neq j} m_{ij}\delta_i(1-\delta_j)X_iY_j.
\]
Since $\EE[\delta_i(1-\delta_j)] = \frac{1}{4}$, we have 
\[
S_{\offdiag} = 4\EE_{\delta}[S_\delta],
\]

Hence, Jensen's inequality yields 
\[
\EE_{\vx, \vy} \exp(\lambda S_\offdiag) \le \EE_{\vx, \vy, \delta} \exp(4\lambda S_\delta),
\]
where we have used the independence of $\delta$ and all other random variables. It follows that it suffices to upper bound the MGF of $S_\delta$.

Define the random set $\Lambda_\delta = \qty{i \in [n]: \delta_i = 1}$ to denote the indices selected by $\delta$. For a vector $\vu \in \RR^n$ we also introduce the shorthand $\vu_{\Lambda_\delta}$ to denote the subvector of $\vu$ where $\delta_i = 1$ and $\vu_{\Lambda_\delta^c}$ to denote the subvector of $\vu$ where $\delta_i = 0$. 

Hence, we can rewrite $S_\delta \triangleq \sum_{i \in \lamdel, j \in \lamdelc} m_{ij} X_i Y_j$. For $\abs{\lambda} \le \frac{1}{2\Const{c_quadratic_taylor}K\norm{\mM}_2}$, we have
\begin{align}
    \EE \exp(\lambda S_\offdiag) &\le \EE \exp(4\lambda S_\delta) \\
    &\le \EE_\delta \prod_{i \in \lamdel, j \in \lamdelc} \EE_{\vx_{\lamdel}, \vy_{\lamdelc}}[\exp(\lambda m_{ij} X_i Y_j)] 
\end{align}
Now we can use the fact that the $X_i$ and $Y_j$ are mean zero and independent because $i \in \lamdel$ and $j \in \lamdelc$, to show that
\begin{align}
    \prod_{i \in \lamdel, j \in \lamdelc} \EE_{\vx_{\lamdel}, \vy_{\lamdelc}}[\exp(\lambda m_{ij} X_i Y_j)] &\le \prod_{i \in \lamdel, j \in \lamdelc} \EE_{\vy_{\lamdelc}}[\exp(\Const{c_quadratic_taylor}\lambda^2 K^2 m_{ij}^2 Y_j^2)] \\
    &\le \prod_{i \in \lamdel, j \in \lamdelc} \EE_{\vy_{\lamdelc}}[1 + 2\Const{c_quadratic_taylor}\lambda^2 K^2 m_{ij}^2 Y_j^2)] \\
    &\le \prod_{i \in \lamdel, j \in \lamdelc} (1 + 2\pi \Const{c_quadratic_taylor}\lambda^2 K^2 m_{ij}^2 Y_j^2) \\
    &\le \prod_{i \in \lamdel, j \in \lamdelc} \exp(2\pi\Const{c_quadratic_taylor}\lambda^2 K^2 m_{ij}^2 Y_j^2) \\
    &\le \exp(2\pi \Const{c_quadratic_taylor}\lambda^2 K^2 \norm{\mM}_F^2).
\end{align}
In the first line, we have used the subgaussianity of $X_i$; in the second line, we have used the assumption on $\lambda$, in the third line, we have used the variance bound on $Y_j$.

Again, we can apply Markov's inequality and to see that for $\epsilon > 0$,
\begin{align}
    \Pr[S_\diag > \epsilon] &\le \frac{\EE \exp(\lambda S_\diag)}{\exp(\lambda \epsilon)} \\
    &\le \exp(-\lambda \epsilon + 2\pi\Const{c_quadratic_taylor}K^2\lambda^2 \norm{\mM}_F^2),
\end{align}
Picking 
\begin{equation}
    \lambda = \min\qty{\frac{\epsilon}{2\Const{c_quadratic_taylor}K^2 \pi\norm{\mM}_F^2}, \frac{1}{2\Const{c_quadratic_taylor}K \norm{\mM}_2}}
\end{equation}
yields the desired result.
\section{Proofs of main lemmas for concentration of spectrum}\label{sec:technical-spectrum}
The goal of this section is ultimately to prove \cref{prop:concentration-hat-matrix-simple}, which asserts that for valid $(T,S)$, the hat matrix $\mH_{T,S}$ is a flat matrix whose  spectrum is $\min\qty{\mu, 1}(1+ o(1))$ with extremely high probability. First, let us recall some notation. For any $\varnothing \neq T \subseteq S \subseteq [s]$, we can define the $(T, S)$ hat matrix  as $\mH_{T, S} \triangleq \Xt^\top \mA_{-S}^{-1} \Xt$. Here, $\Xt$ is the $n \times \abs{T}$  matrix of weighted features in $T$, and $\At = \mA - \Xt \Xt^\top$ is the leave-$T$-out Gram matrix.  

First, Wishart concentration applied to $\Xt^\top \Xt$ yields the following result.
\begin{lemma}\label{lemma:xtxt-concentration}
    Recall that $\mu \triangleq n^{q+r-1}$ and $\Xt^\top \Xt \in \RR^{\abs{T} \times \abs{T}}$. 
    For any nonempty $T \subseteq [s]$, with probability at least $1 - 2e^{-\sqrt{n}}$ we have that for all $i \in [\abs{T}]$, 
    \begin{align}
        \mu_i(\Xt^\top \Xt) = \qty(1 \pm c_T \sqrt{\tfrac{\abs{T}}{n}}) \mu^{-1}n^p
    \end{align}
\end{lemma}
\begin{proof}
    We can apply \cref{lemma:wishart-concentration} with $\mM = \Zt \in \RR^{n \times \abs{T}}$, with $M = n$, $m = \abs{T} = o(n)$, and $\epsilon = n^{\frac{1}{4}} = o(\sqrt{n\abs{T}})$. Hence we have 
    \begin{align*}
        n - 2\sqrt{n\abs{T}} + o(\sqrt{n\abs{T}}) \le \mu_{\abs{T}}(\Zt^\top \Zt) \le \mu_1(\Zt^\top \Zt) \le n + 2\sqrt{n\abs{T}} + o(\sqrt{n\abs{T}}).
    \end{align*}

    Pluagging in the scaling $\lambda_F = n^{p-q-r}$ and dividing through by $n$ yields the desired result. Here, we define $c_T$ to be an appropriately defined positive constant which only depends on $\abs{T}$ (as the favored features are identically distributed).
\end{proof}

Next, we can use Wishart concentration to bound the spectrum of $\Au$.
\begin{lemma}[Concentration of spectrum for unfavored Gram matrix]\label{prop:gram-spectrum}
    Throughout this theorem, assume we are in the bi-level model (\cref{def:bilevel}). Define $\mu \triangleq n^{q+r-1}$. Recall $\Au = \lambda_U \sum_{j > s} \vz_j \vz_j^\top \in \RR^{n \times n}$. With probability at least $1-2e^{-n}$, for $i \in [n]$ we have 
        \begin{equation}
            \mu_i(\Au) = (1 \pm \const{const_delu} n^{\konst{kappa_au}})n^p,
        \end{equation}
        where \const{const_delu} and \konst{kappa_au} are positive constants. In other words, the spectrum of the unfavored Gram matrix $\Au$ is flat. 
\end{lemma}
\begin{proof}
    Note that $\Au = \lambda_U \sum_{j > s} \vz_j \vz_j^\top$. Under the bi-level model, $\lambda_U = 1 + o(1)$. Now we can apply \cref{lemma:wishart-concentration} with $\mM = \sum_{j > s} \vz_j \vz_j^\top$, $M = d - s = n^p - n^r$, $m = n = o(d)$, and $\epsilon = \sqrt{2n}$ to conclude that with probability at least $1-2e^{-n}$, we have
        \begin{equation}
        d - 2\sqrt{dn} + n - \sqrt{2n} \le \mu_n(\sum_{j > s} \vz_j \vz_j^\top) \le \mu_1(\sum_{j > s} \vz_j \vz_j^\top) \le d + 2\sqrt{dn} + n + \sqrt{2n}. \label{eq:unweighted-unfavored-spectrum}
        \end{equation}
        We can obtain the spectrum of $\Au$ by multiplying through by 
        \begin{align}
             \lambda_U &= \frac{(1-a)d}{d-s} \\
             &= 1 + n^{\max\qty{-q, r-p}} + o(n^{\max\qty{-q, r-p}}),
        \end{align}
        where in the last line we have used the power series expansion for $\frac{1}{1-x} = 1 + x + o(x)$.
       Preserving only first order terms for $\lambda_U$ and the spectrum of $\sum_{j > s} \vz_j \vz_j^\top$ in \cref{eq:unweighted-unfavored-spectrum} yields
        \begin{equation}
            \mu_i(\Au) = (1 \pm \const{const_delu}n^{\max\qty{\frac{1-p}{2}, r-p, -q}}) n^p.
        \end{equation}
    
        In fact, we know $\frac{1-p}{2} > 1-p > r-p$, since $r < 1$ and $p > 1$. This means we can neglect the $r-p$ term in the max, define $\newk\label{kappa_au} = \min\qty{\frac{p-1}{2}, -q} > 0$ and \newc\label{const_delu} to be an appropriately defined positive constant.
\end{proof}
Since $\At = \Au + \mW_{[s] \setminus T} \mW_{[s] \setminus T}^\top$, we can apply \cref{prop:gram-spectrum,lemma:xtxt-concentration} to control the spectrum of $\At$. We show that there is a (potentially) spiked portion of the spectrum corresponding to the $s - \abs{T}$ favored features which were not taken out, whereas the rest of the $n - s + \abs{T}$ eigenvalues are flat. 
\begin{lemma}\label{lemma:spectrum-at}
    Recall that $\At \in \RR^{n \times n}$. For any nonempty $T \subseteq [s]$, with probability at least $1 - 2e^{-\sqrt{n}} - 2e^{-n}$, we have that for all $i \in [s - \abs{T}]$,
    \begin{align}
        \mu_i(\At) &= \qty(1 \pm c_T \sqrt{\tfrac{\abs{T}}{n}}) \mu^{-1}n^p + \qty(1 \pm \const{const_delu} n^{-\konst{kappa_au}}) n^p. \label{eq:spiked-spectrum-at}
    \end{align}

    For all $i \in [n] \setminus [s - \abs{T}]$, we have 
    \begin{align}
        \mu_i(\At) = \qty(1 \pm \const{const_delu} n^{-\konst{kappa_au}}) n^p. \label{eq:flat-spectrum-at}
    \end{align}
\end{lemma}
\begin{proof}
    We can write 
    \begin{align}
        \At = \mW_{[s] \setminus T} \mW_{[s] \setminus T}^\top + \Au.
    \end{align}
    Weyl's inequality \citep[Corollary~4.3.15]{horn2012matrix} implies that for any $i \in [n]$, we have 
    \begin{align}
        \mu_i(\mW_{[s] \setminus T}  \mW_{[s] \setminus T} ^\top) + \mu_n(\Au) \le \mu_i(\At) \le \mu_i(\mW_{[s] \setminus T}  \mW_{[s] \setminus T} ^\top) + \mu_1(\Au).\label{eq:weyl-application}
    \end{align}
    Then applying \cref{lemma:xtxt-concentration,prop:gram-spectrum}, for $i \in [s-\abs{T}]$ we conclude that 
    \begin{align}
        \mu_i(\At) = \qty(1 \pm c_T \sqrt{\tfrac{\abs{T}}{n}}) \mu^{-1}n^p + \qty(1 \pm \const{const_delu} n^{-\konst{kappa_au}}) n^p..
    \end{align}
    which proves \cref{eq:spiked-spectrum-at}. 

    For $i > s-\abs{T}$, applying \cref{prop:gram-spectrum} and the fact that $\mu_i(\mW_{[s] \setminus T}  \mW_{[s] \setminus T}^\top) = 0$ to \cref{eq:weyl-application} yields
    \begin{align}
        \mu_i(\At) = \qty(1 \pm \const{const_delu} n^{-\konst{kappa_au}}) n^p.
    \end{align}
    which proves \cref{eq:flat-spectrum-at}.
\end{proof}

By inverting the bounds proved above, we can also control the spectrum of $\Atinv$. 
\begin{corollary}\label{cor:atinv-spectrum}
    Recall that $\At \in \RR^{n \times n}$. For any nonempty $T \subseteq [s]$, with probability at least $1 - 2e^{-\sqrt{n}} - 2e^{-n}$, we have that for all $i \in [n - s + \abs{T}]$,
    \begin{align}
        \mu_i(\Atinv) = (1 \pm \const{const_atinv_flat}n^{-\konst{kappa_au}}) n^{-p} \label{eq:flat-spectrum-atinv}
    \end{align}

    For all $i \in [n] \setminus [n - s + \abs{T}]$, we have 
    \begin{align}
        \mu_i(\Atinv) = \min\qty{\mu, 1}(1 \pm \const{const_atinv}n^{-\konst{kappa_atinv}})n^{-p}. \label{eq:spiked-spectrum-ainv}
    \end{align}
    where \konst{kappa_atinv} is a positive constant depending on $\abs{T}$.
\end{corollary}
\begin{proof}
    By inverting the bounds in \cref{lemma:spectrum-at}, using the fact that $\mu_i(\Atinv) = \frac{1}{\mu_{n-i+1}(\At)}$ we see that for $i \in [n - s + \abs{T}]$, 
    \begin{align}
        \mu_i(\Atinv) &= \frac{1}{1 \pm \const{const_delu} n^{-\konst{kappa_au}}} n^{-p} \\
        &= (1 \pm \const{const_atinv_flat}n^{-\konst{kappa_au}}) n^{-p}, 
    \end{align}
    where we have used the power series expansion $\frac{1}{1-x} = 1 + x + o(x^2)$ and \newc\label{const_atinv_flat} is a positive constant. 

    On the other hand, for $i > n-s+\abs{T}$, we get 
    \begin{align}
        \mu_i(\Atinv) &= \frac{1}{\qty(1 \pm c_T \sqrt{\tfrac{\abs{T}}{n}}) \mu^{-1} + \qty(1 \pm \const{const_delu} n^{-\konst{kappa_au}})} n^{-p} \\
        &= \min\qty{\mu, 1}(1 \pm \const{const_atinv}n^{-\konst{kappa_atinv}})n^{-p},
    \end{align}
    where \newc\label{const_atinv} and $\konst{kappa_atinv}$ are positive constants defined as follows. If $q+r<1$, i.e. regression works, then $\mu^{-1} = \omega(1)$, so the denominator becomes $\mu^{-1}\qty(1 \pm c_T \sqrt{\tfrac{\abs{T}}{n}} + \mu(1 \pm \const{const_delu} n^{-\konst{kappa_au}}))$. Then, since $\abs{T} \le s = n^r$, we see that we can pick 
    \[
        \konst{kappa_atinv}= \min\qty{\frac{1-r}{2}, 1-q-r}.
    \]

    On the other hand, if $q+r>1$, i.e. regression fails, then $\mu^{-1} = o(1)$, and so we can define 
    \[
        \konst{kappa_atinv} = \min\qty{\konst{kappa_au}, q+r-1}.
    \]
    Hence to cover both cases we can pick
    \[
        \newk\label{kappa_atinv} = \min\qty{\frac{1-r}{2}, \konst{kappa_au}, \abs{1-q-r}}.
    \]
    
    The choice of \const{const_atinv} is picked by again using the power series expansion for $\frac{1}{1-x}$. 
\end{proof}
Note that \cref{cor:atinv-spectrum} immediately implies \cref{prop:au-ak-flat}, with \newk\label{kappa_ak_spiked} defined based on picking $T = [k]$. We are now in a position to prove that the generalized hat matrices $\mH_{T, S}$, and hence the Woodbury terms $(\mI_{\abs{T}} + \mH_{T, S})^{-1}$ have a flat spectrum as well. 

\hatmatrix*
\begin{proof}
    We seek to control the spectrum of the hat matrix $\mH_{T,S} = \Xt^\top \ASinv \Xt$. We cannot directly use naive eigenvalue bounds to bound the minimum and maximum eigenvalue, as this does not rule out the possibility that $\mH_{T,S}$ has a spike. Instead, we control the spectrum from first principles.

\paragraph{The spectrum of $\mH_{T,S}$ is flat:}
 Since $\ASinv$ is symmetric, it has an eigendecomposition $\mV \mD \mV^\top$, where $\mV$ is an orthogonal matrix. Because $\Xt$ is a weighted subset of only the (equally) favored features, its law is rotationally invariant. Furthermore, since $\ASinv$ is independent of $\Xt$ (as $T \subseteq S$), we can absorb the rotation $\mV$ into $\Xt$ to reduce to the case where $\ASinv = \mD$. Here, we have 
\begin{align}
    \mD \triangleq \mqty[\dmat{\mD_{\mathsf{flat}}, \mD_{\mathsf{spiked}}}] \in \RR^{n \times n} = \mqty[\dmat{\mu_1(\ASinv), \ddots, \mu_n(\ASinv)}],
\end{align}
where $\mD_{\mathsf{flat}} \in \RR^{(n-s+\abs{T}) \times (n-s+\abs{T})}$ and $\mD_{\mathsf{spiked}} \in \RR^{(s-\abs{T}) \times (s-\abs{T})}$ correspond to the flat and (downwards) spiked portions of the spectrum of $\ASinv$. We can also correspondingly decompose 
\begin{align}
    \Zt = \mqty[\mB_T \\ \mC_T],
\end{align}
where $\mB_T \in \RR^{(n-s+\abs{T}) \times \abs{T}}$ and $\mC_T \in \RR^{(s-\abs{T}) \times \abs{T}}$. Note that each entry of these matrices are i.i.d.~ $N(0,1)$ variables. 

By direct computation we have  
\begin{align}
    \Zt^\top \mD \Zt = \mB_T^\top \mD_{\mathsf{flat}} \mB_T + \mC_T^\top \mD_{\mathsf{spiked}} \mC_T
\end{align}
We thus have by using standard eigenvalue inequalities that
\begin{align}
    \mu_{\abs{T}}(\Zt^\top \mD \Zt) &\ge \mu_{\abs{T}}(\mB_T^\top \mD_{\mathsf{flat}} \mB_T) + \mu_{\abs{T}}(\mC_T^\top \mD_{\mathsf{spiked}} \mC_T) \\
    &\ge \mu_{\abs{T}}(\mB_T^\top \mB_T) \mu_{n-s+\abs{T}}(\ASinv) + \mu_{\abs{T}}(\mC_T^\top \mC_T) \mu_n(\ASinv) \\
    &\ge \mu_{\abs{T}}(\mB_T^\top \mB_T)  \mu_{n-s+\abs{T}}(\ASinv), \label{eq:hat-matrix-lower-bound-composite}
\end{align}
where in the last line we have used $\mu_n(\ASinv) \ge 0$.

Since $n>s$, we have $n-s+\abs{T} > \abs{T}$, so we can apply Wishart concentration (\cref{lemma:wishart-concentration}) to $\mB_T^\top \mB_T$ to obtain that with probability at least $1-2e^{-\sqrt{n}}$ we have
\begin{align}
    \mu_{\abs{T}}(\mB_T^\top \mB_T) &\ge n-s+\abs{T} - 2\sqrt{(n-s+\abs{T})\abs{T}} + o(\sqrt{(n-s+\abs{T})\abs{T}}) \\
    &\ge n(1 - n^{r-1} - \const{const_bk} \sqrt{\tfrac{\abs{T}}{n}}), \label{eq:bk-lower-bound}
\end{align}
where \newc\label{const_bk} is a positive constant.

On the other hand, we can deduce that 
\begin{align}
    \mu_1(\Zt^\top \mD \Zt) \le \mu_1(\Zt^\top \Zt) \mu_1(\ASinv).
\end{align}
\Cref{lemma:xtxt-concentration} implies that with probability at least $1-2e^{-\sqrt{n}}$
\begin{align}
    \mu_1(\Zt^\top \Zt) \le n(1 + c_T \sqrt{\tfrac{\abs{T}}{n}}) \label{eq:zk-upper-bound}
\end{align}
Similarly, \cref{cor:atinv-spectrum} implies that with probability at least $1 - 2e^{-n} - 2e^{-\sqrt{n}}$, $\mu_1(\ASinv)$ and $\mu_{n-s+\abs{T}}(\ASinv)$ are both $(1 \pm \const{const_atinv}n^{\konst{kappa_atinv}})n^{-p}$. Since $\abs{T} \le s = n^r$, \cref{eq:bk-lower-bound,eq:zk-upper-bound} together demonstrate that for all $i \in [\abs{T}]$, 
\begin{align}
    \mu_i(\Zt^\top \mD \Zt) = n^{1-p}(1 \pm \const{const_hatts_intermediate} n^{-\konst{kappa_hatts_intermediate}}).
\end{align}
Here, \newc\label{const_hatts_intermediate} and $\newk\label{kappa_hatts_intermediate}$ are positive constants defined as follows.
Since $\abs{T} \le s = n^{r}$. Then $\konst{kappa_hatts_intermediate} = \min\qty{1-r, \frac{1 - r}{2}, \konst{kappa_atinv}}$, and $\const{const_hatts_intermediate}$ is a constant chosen appropriately based on \const{const_atinv} and \const{const_bk}. Plugging in the scaling $\lambda_F = n^{p-q-r}$, we conclude that with extremely high probability, for all $i \in [\abs{T}]$,
\begin{align}
    \mu_i(\mH_{T, S}) =  \mu^{-1}(1 \pm \const{const_hatts_intermediate} n^{-\konst{kappa_hatts_intermediate}}). \label{eq:hatts-spectrum}
\end{align}

From here, it is easy to compute the spectrum of $(\mI_{\abs{T}} + \mH_{T, S})^{-1}$. Indeed, reading off our result from \cref{eq:hatts-spectrum} yields  
\begin{align}
    \mu_i((\mI_{\abs{T}} + \mH_{T, S})^{-1}) &= \frac{1}{1 + \mu_{n-i+1}(\mH_{T, S})} \\
    &= \min\qty{\mu, 1}(1 \pm c_{T, S}n^{-\konst{kappa_hatts_final}}).
\end{align}
Here, the positive constant $c_{T, S}$ is picked appropriately and $\newk\label{kappa_hatts_final} = \min\qty{\konst{kappa_hatts_intermediate}, \abs{1-q-r}} > 0$. This completes the proof.
\end{proof}
\section{Hanson-Wright inequality for sparse subgaussian bilinear forms}\label{sec:hanson-wright-hard-sparsity}

For the sake of completeness, we also prove a variant of Hanson-Wright where we assume one of the vectors has subgaussian coordinates with hard sparsity, as opposed to bounded with soft sparsity as in \cref{thm:bounded-hanson-wright-bilinear}.  

For comparison, we state the version of Sparse Hanson-Wright for bilinear forms which is Theorem 1 in \citep{park2022sparse}. We then state and prove the special case which we need for our proof.

\begin{theorem}[Sparse Hanson-Wright for bilinear forms]\label{thm:sparse-hanson-wright-bilinear-general}
    Let $\vx = (X_1, \ldots, X_n) \in \RR^n$ and $\vy = (Y_1, \ldots, Y_n) \in \RR^n$ be random vectors such that $(X_j, Y_j)$ are independent of all other pairs (although $X_j$ and $Y_j$ may be correlated). Assume also that $\EE[X_i] = \EE[Y_i] = 0$ and $\max\qty{\norm{X_i}_{\psi_2}, \norm{Y_i}_{\psi_2}} \le K$. 

    Also, for $i \in [2]$ suppose $\vgamma_i = (\gamma_{i1}, \ldots, \gamma_{in}) \in \qty{0, 1}^n$ are Bernoulli vectors such that the pairs $(\gamma_{1j}, \gamma_{2j})$ are independent of all other pairs (although $\gamma_{1j}$ and $\gamma_{2j}$ can be correlated). Assume that $\EE[\gamma_{ij}] = \pi_{ij}$ and $\EE[\gamma_{1j}\gamma_{2j}] = \pi_{12,j}$. We assume that $X_i$ is independent of $\gamma_{2j}$ and $Y_i$ is independent of $\gamma_{1j}$ for any $i, j$, although $X_i$ may be correlated with $\gamma_{1i}$ and $Y_i$ may be correlated with $\gamma_{2i}$.

    Then for there exists an absolute constant $c > 0$ such that all $\mM \in \RR^{n \times n}$ and $\epsilon \ge 0$ we have 
    \begin{align}
    &\Pr\qty[|(\vx \circ \vgamma_1)^\top \mM (\vy \circ \vgamma_2) - \EE[(\vx \circ \vgamma_1)^\top \mM (\vy \circ \vgamma_2) ]| > \epsilon] \\
    &\quad\quad\quad \le 2\exp(-c\min\qty{\frac{\epsilon^2}{K^4\qty(\sum_{i=1}^n \pi_{12, i}m_{ii}^2 + \sum_{i \neq j} \pi_{1i}\pi_{2j}m_{ij}^2)}, \frac{\epsilon}{K^2\norm{\mM}_2}})\label{eq:hansonwright-sparse-bilinear-general}
\end{align}
\end{theorem}

We now state the variant of Hanson-Wright which suffices for the analysis of the forms $\vz_j^\top \Ainv \dely$ with \emph{hard sparsity}. However, this variant does not cover the case of $\vz_j^\top \Ainv \ya$, which has soft sparsity. We remark that although this variant uses hard sparsity, it does not explicitly follow from \cref{thm:sparse-hanson-wright-bilinear-general}. In our case, the mask $\vgamma$ crucially has dependence with the other vector $\vx$ in the bilinear form, but this is not allowed in \cref{thm:sparse-hanson-wright-bilinear-general}. For this reason, we have to carefully modify the proof to ensure the stated result still holds.

\hansonwrighthard*

Our proof is going to be split up into the steps as outlined above. First, we can explicitly write 
\begin{align}
    \vx^\top \mM (\vy \circ \vgamma) - \EE[\vx^\top \mM (\vy \circ \vgamma)] &= \sum_{i, j} m_{ij} \gamma_jX_iY_j - \sum_i m_{ii} \EE[\gamma_iX_iY_i] \\
    &= \underbrace{\sum_{i} m_{ii}(\gamma_iX_iY_i - \EE[\gamma_iX_iY_i])}_{\triangleq S_{\diag}}+ \underbrace{\sum_{i \neq j} m_{ij}\gamma_jX_iY_j}_{\triangleq S_{\offdiag}}  
\end{align}

\subsection{Diagonal terms}
We first prove the desired concentration inequality for
\[
S_{\diag} \triangleq \sum_i m_{ii} (\gamma_i X_iY_i - \EE[\gamma_i X_iY_i]).
\]
This corresponds to Lemma 1 in \citet{park2022sparse}.
\begin{lemma}[Diagonal concentration]\label{lemma:diagonal-concentration}
Assume the setting of \cref{thm:sparse-hanson-wright-bilinear-hard}. 
There exists a constant $c>0$ such that for any $\mM \in \RR^{n \times n}$ and $t > 0$, we have 
\begin{align}
    \Pr\qty[\abs{S_{\diag}}> t] \le 2\exp(-c\min\qty{\frac{t^2}{K^4\pi\norm{\diag(\mM)}_F^2}, \frac{t}{K^2\norm{\diag(\mM)}_{\infty}}})\label{eq:hansonwright-sparse-bilinear-diagonal}
\end{align}
\end{lemma}

Note that we can replace the RHS with a constant $1$ in front of the $\exp$ and removing the diagonal operator, since $\norm{\diag(\mM}_F \le \norm{\mM}_F$ and $\norm{\diag(\mM)}_{\infty} \le \norm{\mM}_2$.

The high level idea is to use Markov by passing to the MGF to give an exponential-type tail inequality. The first crucial ingredient is to use the fact that $\gamma_i$ is nonzero with probability $p$. Moreover, conditioned on $\gamma_i = 1$, we know by assumption that $\max\qty{\norm{Y_i}_{\psi_2}, \norm{X_i}_{\psi_2}} \le K$, so the product of $X_i$ and $Y_i$ is a subexponential random variable, allowing us to use standard subexponential MGF bounds. we can use the bounds on subexponential MGFs. Throughout, we attempt to be as explicit as possible about what variables we are conditioning on and taking expectations with respect to.

We will need the following lemma, which corresponds to Lemma 3 in \citet{park2022sparse}.
\begin{lemma}[Taylor style MGF bound for product of subgaussians]\label{lemma:mgf-product}
Suppose that $\max\qty{\norm{X_i}_{\psi_2}, \norm{Y_i}_{\psi_2}} \le K$ for all $i$. Let $m_{ii} \in \RR$ for $i \in [n]$. Then there exists positive constants $\Const{c_moment}$ and $\Const{c_quadratic_taylor_new}$ such that for $\abs{\lambda} \le \frac{1}{2\Const{c_moment}eK^2\max_i \abs{m_{ii}}}$, we have for all $i$ that 
\[
\EE_{X_i, Y_i} \exp(\lambda m_{ii} X_iY_i) \le 1 + \lambda m_{ii} \EE[X_iY_i] + \Const{c_quadratic_taylor_new} K^4 \lambda^2 m_{ii}^2.
\]
\end{lemma}
We defer the proof of \cref{lemma:mgf-product} until the end of this section, and turn to proving the main lemma.
\begin{proof}[Proof of diagonal concentration]
Let us now examine the MGF. Let $\lambda$ be a parameter to be optimized later, with $\abs{\lambda} \le \frac{1}{2\Const{c_moment}eK^2\max_i \abs{m_{ii}}}$.
We have 
\begin{align}
    \EE \exp(\lambda S_\diag) &= \frac{\EE \exp(\lambda \sum_{i=1}^n m_{ii} \gamma_i X_iY_i)}{\exp(\lambda \EE\qty[\sum_{i=1}^n m_{ii}\gamma_i X_iY_i])} \\
    &= \prod_{i=1}^n \frac{\EE_{\gamma_i} \Big[\EE_{X_i, Y_i}[\exp(\lambda m_{ii}\gamma_iX_iY_i)|\gamma_i]\Big]}{\exp(\lambda m_{ii}\EE_{\gamma_i}[\gamma_i \EE_{X_i, Y_i}[X_iY_i|\gamma_i]])}.
\end{align}
In the second line, we have used the fact that all of the random variables are independent across different $i$. Now, we evidently have 
\[
\exp(\lambda m_{ii} \gamma_i X_iY_i) = 
\begin{cases*}
1 & $\gamma_i = 0$ \\
\exp(\lambda m_{ii} X_iY_i) & $\gamma_i = 1$ \\
\end{cases*}
\]

Therefore, we can rewrite 
\begin{align}
    \EE \exp(\lambda S_\diag) &= \prod_{i=1}^n \frac{(1-\pi) + \pi\EE_{X_i, Y_i}[\exp(\lambda m_{ii}X_iY_i)|\gamma_i=1]}{\exp(\pi\lambda m_{ii}\EE_{X_i, Y_i}[X_iY_i|\gamma_i=1]])}.
\end{align}

Now, we can apply \cref{lemma:mgf-product} \emph{after conditioning on} $\gamma_i=1$. Note that we did not require that $X_i$ or $Y_i$ are mean zero after conditioning on $\gamma_i$. In any case, we obtain 

\begin{align}
    \EE \exp(\lambda S_\diag) &\le \prod_{i=1}^n \frac{(1-\pi) + \pi(1 + \lambda m_{ii} \EE_{X_i, Y_i}[X_iY_i|\gamma_i=1] + \Const{c_quadratic_taylor_new} K^4 \lambda^2m_{ii}^2)}{\exp(\pi\lambda m_{ii}\EE_{X_i, Y_i}[X_iY_i|\gamma_i=1])} \\
    &\le \prod_{i=1}^n \frac{\exp(\pi\lambda m_{ii} \EE_{X_i, Y_i}[X_iY_i|\gamma_i=1] + \pi\Const{c_quadratic_taylor_new} K^4 \lambda^2m_{ii}^2)}{\exp(\pi\lambda m_{ii}\EE_{X_i, Y_i}[X_iY_i|\gamma_i=1])} \\
    &\le \exp(\pi\Const{c_quadratic_taylor_new} K^4\lambda^2\sum_{i=1}^n m_{ii}^2),
\end{align}
where in the second line we have used the inequality $1+x \le e^x$.

Now Markov yields for $t > 0$ that 
\begin{align}
    \PP[S_\diag > t] &\le \frac{\EE \exp(\lambda S_\diag)}{\exp(\lambda t)} \\
    &\le \exp(-\lambda t + \pi\Const{c_quadratic_taylor_new}K^4\lambda^2 \sum_{i=1}^n m_{ii}^2),
\end{align}
and optimizing $\lambda$ in the region $\abs{\lambda} \le \frac{1}{2\Const{c_moment}eK^2\max_i \abs{m_{ii}}}$ yields 
\begin{equation}
    \lambda = \min\qty{\frac{t}{2\Const{c_quadratic_taylor_new}K^4 \pi\sum_{i=1}^n m_{ii}^2}, \frac{1}{2\Const{c_moment}eK^2 \max_i \abs{m_{ii}}}}.
\end{equation}

Plugging this into the inequality yields the desired bound. We can repeat the argument with $-\mM$ to get the lower tail bound. A union bound completes the proof.
\end{proof}
Before we prove the offdiagonal bound, we prove \cref{lemma:mgf-product}.
\begin{proof}[Proof of \cref{lemma:mgf-product}]
 By Lemma 2.7.7 in \citep{vershynin2018high}, we have $\norm{X_iY_i}_{\psi_1} \le \norm{X_i}_{\psi_2} \norm{Y_i}_{\psi_2} \le K^2$. Set $W_i \triangleq \frac{X_iY_i}{\norm{X_iY_i}_{\psi_1}}$ and $t_i = \lambda m_{ii} \norm{X_iY_i}_{\psi_1}$. Hence, $\norm{W_i}_{\psi_1} = 1$. For the rest of this proof, all expectations are understood to be taken with respect to $W_i$ without chance of confusion.  By Proposition 2.7.1 in \citet{vershynin2018high}, this implies that there exists an absolute constant $\newC\label{c_moment}$ such that $\EE[\abs{W_i}^p] \le \Const{c_moment}^pp^p$ for all $p \ge 1$. 
 
 Now, we can use the Taylor expansion of $\exp(x)$ to obtain 
\begin{align}
    \EE \exp(t_iW_i) &= 1 + t_i\EE[W_i] + \sum_{k \ge 2} \frac{t_i^k \EE[W_i^k]}{k!} \\
    &\le 1 + t_i\EE[W_i] + \sum_{k \ge 2} \frac{\abs{t_i}^k \EE[\abs{W_i}^k]}{k!} \\
    &\le 1 + t_i\EE[W_i] + \sum_{k \ge 2} \frac{\Const{c_moment}^k\abs{t_i}^k k^k}{k!}. 
\end{align}
Now, we use the nonasymptotic version of Stirling's formula, which is cited in Proposition 2.5.2 of \citet{vershynin2018high}, which says that $k! \ge \qty(\frac{k}{e})^k$ for $k \ge 1$. Plugging this into the above yields that 
\begin{align}
    \EE \exp(t_iW_i) &\le 1 + t_i\EE[W_i] + \sum_{k \ge 2} (\Const{c_moment}\abs{t_i} e)^k.
\end{align}

Now, if $\abs{\lambda} < \frac{1}{2\Const{c_moment}eK^2\max_i \abs{m_{ii}}}$, then using the definition of $t_i$ we see 
\begin{align}
    \Const{c_moment}e\abs{t_i} &= \Const{c_moment}e\lambda \abs{m_{ii}} \norm{X_iY_i}_{\psi_1} \\
    &\le \frac{\Const{c_moment}e\abs{m_{ii}} \norm{X_iY_i}_{\psi_1}}{2\Const{c_moment}eK^2 \max_i \abs{m_{ii}}} \\
    &\le \frac{1}{2},
\end{align}

so it follows from a geometric series estimate that 
\[
\EE \exp(t_iW_i) \le 1 + t_i\EE[W_i] + \Const{c_quadratic_taylor_new}K^4\lambda^2m_{ii}^2,
\]
where $\newC\label{c_quadratic_taylor_new} = 2\Const{c_moment}^2e^2$.

Recalling our definition of $W_i$ and $t_i$ yields the desired result. In particular, $t_iW_i = \lambda m_{ii} X_iY_i$, so 
\begin{align}
    \EE \exp(\lambda m_{ii} X_iY_i) \le 1 + \lambda m_{ii} \EE[X_iY_i] + \Const{c_quadratic_taylor_new}K^4\lambda^2m_{ii}^2,
\end{align}
as desired.
\end{proof}

\subsection{Offdiagonal terms}
We now turn to the offdiagonal terms, which is Lemma 2 in \citet{park2022sparse}.
\begin{lemma}[Offdiagonal concentration]\label{lemma:offdiagonal-concentration}
Assume the setting of \cref{thm:sparse-hanson-wright-bilinear-hard}. 
There exists a constant $c>0$ such that for any $\mM \in \RR^{n \times n}$ and $t > 0$, we have 
\begin{align}
    \PP\qty[\abs{S_{\offdiag}}> t] \le 2\exp(-c\min\qty{\frac{t^2}{K^4\pi\sum_{i \neq j} m_{ij}^2}, \frac{t}{K^2\norm{\mM}_2}})\label{eq:hansonwright-sparse-bilinear-offdiagonal}
\end{align}
\end{lemma}

\vspace{3ex}
As before, we will achieve this concentration inequality by explicitly bounding the MGF. The high level idea for this is summarized in Section 2. Following \citet{rudelson2013hanson}, we decouple the terms in the sum. More precisely, the terms in $S_\offdiag$ involving indices $i$ and $j$ are precisely $m_{ij}\gamma_jX_iY_j + m_{ji}\gamma_iY_iX_j$. The issue is that $\gamma_i$ and $Y_i$ can be correlated with $X_i$, which complicates the behavior of this random variable. Decoupling ensures that for any $j \in [n]$ we will have exactly one term which involves either $X_j$ or $Y_j$, so in particular we will regain independence of the terms. We can then reduce the problem to standard normals and integrate out both the normal variables and the Bernoulli mask $\vgamma$.
\begin{proof}
We begin with the same steps as in the diagonal case. Recall 
\begin{equation}
S_\offdiag \triangleq \sum_{i \neq j} m_{ij}\gamma_jX_iY_j
\end{equation}Here $\lambda$ is a parameter to be optimized later; we consider the MGF
\[
\EE \exp(\lambda S_\offdiag).
\]
First, we decouple the terms.
\paragraph{Decoupling the terms:} Let $\qty{\delta_i}_{i \in [n]}$ denote iid Bernoulli's with parameter $1/2$, which are independent of all other random variables. 

Let 
\[
S_\delta \triangleq \sum_{i \neq j} m_{ij}\delta_i(1-\delta_j)\gamma_jX_iY_j.
\]
Since $\EE[\delta_i(1-\delta_j)] = \frac{1}{4}$, we have 
\[
S_{\offdiag} = 4\EE_{\delta}[S_\delta],
\]

Hence, Jensen's inequality yields 
\[
\EE_{\vx, \vy, \vgamma} \exp(\lambda S_\offdiag) \le \EE_{\vx, \vy, \vgamma, \delta} \exp(4\lambda S_\delta),
\]
where we have used the independence of $\delta$ and all other random variables. It follows that it suffices to upper bound the MGF of $S_\delta$.

Define the random set $\Lambda_\delta = \qty{i \in [n]: \delta_i = 1}$ to denote the indices selected by $\delta$. For a vector $\vu \in \RR^n$ we also introduce the shorthand $\vu_{\Lambda_\delta}$ to denote the subvector of $\vu$ where $\delta_i = 1$ and $\vu_{\Lambda_\delta^c}$ to denote the subvector of $\vu$ where $\delta_i = 0$. 

We can rewrite 
\[
S_\delta = \sum_{i \in \Lambda_\delta, j \in \Lambda_\delta^c} m_{ij}\gamma_j X_iY_j = \sum_{j \in \Lambda_\delta^c} Y_j \qty(\sum_{i \in \Lambda_\delta} \gamma_j m_{ij}X_i).
\]
Note that $S_\delta$ only depends on $\delta$, $\vx_{\lamdel}$,$\vgamma_{\lamdelc}$, and $\vy_{\lamdelc}$. Hence we have 
\begin{equation}
\EE_{\vx, \vy, \vgamma, \delta} \exp(4\lambda S_\delta) = \EE_{\delta} \EE_{\vgamma_{\lamdelc}}\qty[\EE_{\vx_{\lamdel}}\qty[\EE_{\vy_{\lamdelc}}[\exp(4\lambda S_\delta) | \vx_{\lamdel}, \vgamma_{\lamdelc}, \delta]]] \label{eq:S_delta_exp}
\end{equation}
Recall that by assumption, $\gamma_j$ is independent of $Y_j$. Furthermore, since $i \neq j$ we have that $Y_j$ is independent of $X_i$. Also, $Y_j$ is a mean zero subgaussian random variable with subgaussian norm $K$. So if we condition on $\vgamma_{\lamdelc}$, $\vx_{\lamdel}$, and $\delta$, we find that $S_\delta = f(\vgamma_{\lamdelc}, \vx_{\lamdel}, \delta)$ is (conditionally) a mean zero subgaussian random variable.

Let 
\begin{equation}
    \sigma_{\delta, \vgamma}^2 \triangleq \sum_{j \in \Lambda_\delta^c} \qty(\sum_{i \in \Lambda_\delta} \gamma_j m_{ij}X_i)^2. \label{eq:sigma}
\end{equation}

Note that $\sigma_{\delta, \vgamma}^2$ is a random variable which depends on $\vgamma_{\lamdelc}$, $\vx_{\lamdel}$, and $\delta$, which are all independent of $\vy$.
Then Proposition 2.6.1 in \citet{vershynin2018high} implies that $S_\delta$ conditionally has subgaussian norm which satisfies
\begin{align}
    \norm{S_\delta}_{\psi_2}^2 &\le \Const{c_S_delta_subgaussian}\sum_{j \in \lamdelc} \norm{Y_j \sum_{i \in \lamdelc} \gamma_j m_{ij} X_i}_{\psi_2}^2 \\
    &\le  \Const{c_S_delta_subgaussian} K^2 \sum_{j \in \lamdelc} \qty(\sum_{i \in \lamdelc} \gamma_j m_{ij} X_i)^2 \\
    &\le \Const{c_S_delta_subgaussian} K^2 \sigma_{\delta, \vgamma}^2, \label{eq:S_delta_subgaussian}
\end{align}
    
with $\newC\label{c_S_delta_subgaussian}$ being a positive constant. 

Now, since $S_\delta$ is a mean zero subgaussian random variable, Proposition 2.5.2(e) of \citet{vershynin2018high} yields that 
\begin{equation}
    \EE_{\vy_{\lamdelc}}\qty[\exp(4\lambda S_\delta) | \vgamma_{\lamdelc}, \vx_{\lamdel}, \delta] \le \exp(\Const{c_S_delta_subgaussian}K^2\lambda^2 \sigma_{\delta, \vgamma}^2).\label{eq:conditional-mgf}
\end{equation}
Note that the RHS of \cref{eq:conditional-mgf} is a random variable.
Now define 
\begin{equation}
E_{\delta} \triangleq \EE_{\vgamma_{\lamdelc}}\qty[\EE_{\vx_{\lamdel}}\qty[\exp(\Const{c_S_delta_subgaussian}K^2\lambda^2 \sigma_{\delta, \vgamma}^2)|\delta]],
\end{equation}
which is well-defined because given $\delta$, $\vx_{\lamdel}$ is independent of $\vgamma_{\lamdelc}$. Hence taking expectations with respect to $\vgamma_{\lamdelc}$ and $\vx_{\lamdel}$ in \cref{eq:conditional-mgf}, we have 
\begin{equation} \EE_{\vgamma_{\lamdelc}}\qty[\EE_{\vx_{\lamdel}}\qty[\EE_{\vy_{\lamdelc}}\qty[\exp(4\lambda S_\delta) | \vgamma_{\lamdelc}, \vx_{\lamdel}, \delta]]] \le E_\delta \label{eq:E_delta_ineq}.
\end{equation}
This exact reduction is used in \citet{rudelson2013hanson}, and we turn to bounding $E_\delta$.
\paragraph{Reducing to normal random variables:}
The RHS of \cref{eq:conditional-mgf} is strikingly similar to the MGF of a gaussian. We can use this to our advantage to reduce to the normal case. In particular, let $\vg = (G_1, \ldots, G_n) \in \RR^n$ with $G_i \sim N(0,1)$ and independent of all other random variables. Then define the random variable
\[
Z = \sum_{j \in \lamdelc} G_j\qty(\sum_{i \in \lamdel} \gamma_j m_{ij}X_i),
\]
so conditioned on $\vgamma_{\lamdelc}, \vx_{\lamdel}, \delta$, we have $Z \sim N(0, \sigma_{\delta, \vgamma}^2)$.
Hence, the conditional MGF
\begin{equation}
\EE_{\vg}[ \exp(sZ) | \vgamma_{\lamdelc}, \vx_{\lamdel}, \delta] = \exp(s^2 \sigma_{\delta, \vgamma}^2).\label{eq:gaussian-mgf}
\end{equation}
If we select $s$ such that $s^2 = \Const{c_S_delta_subgaussian}K^2\lambda^2$, then we can match the RHS of \cref{eq:conditional-mgf}, which replaces the $Y_j$'s with $G_j$'s.

Hence, taking expectations over $\vgamma_{\lamdelc}$ and $\vx_{\lamdel}$ in \cref{eq:gaussian-mgf} gives $E_\delta$ on the RHS. Plugging this into \cref{eq:E_delta_ineq} yields
\[
\EE_{\vgamma_{\lamdelc}, \vx_{\lamdel}}\qty[\EE_{\vy_{\lamdelc}}[ \exp(4\lambda S_\delta) | \vgamma_{\lamdelc}, \vx_{\lamdel}, \delta]] \le \EE_{\vgamma_{\lamdelc}, \vx_{\lamdel}}\qty[\EE_{\vg}[ \exp(sZ) | \vgamma_{\lamdelc}, \vx_{\lamdel}, \delta]]
\]

Now, since $\vg$ is independent of all other random variables, we have 
\begin{equation}
    E_{\delta} = \EE_{\vgamma_{\lamdelc}}\qty[\EE_{\vx_{\lamdel}}\qty[\EE_{\vg}[\exp(sZ) | \vgamma_{\lamdelc}, \vx_{\lamdel}, \delta]]] = \EE_{\vgamma_{\lamdelc}}\qty[\EE_{\vg}\qty[\EE_{x_{\lamdel}}[\exp(sZ) | \vg, \vgamma_{\lamdelc}, \delta]]] \label{eq:reduction}
\end{equation}

Now, let us obtain an upper bound on the innermost term on the RHS of \cref{eq:reduction}.
Swapping the order of summation in $Z$ yields
\[
Z = \sum_{i \in \lamdel} X_i \qty(\sum_{j \in \lamdelc} G_j \gamma_j m_{ij}).
\]
Since we don't condition on $\vgamma_{\lamdel}$, and $\vg$ and $\delta$ are independent of $\vx$, we have by assumption that each $X_i$ in the sum is still a centered subgaussian variable of subgaussian norm $K$. 
We can thus repeat the same argument we used for $Y_j$ that culminated in \cref{eq:conditional-mgf}, except this time we condition on $\vg$, $\vgamma_{\lamdelc}$, and $\delta$. In particular, there exists a positive constant \newC\label{c_Z_subgaussian} such that 
\[
\EE_{\vx_{\lamdel}}[\exp(sZ)|\vg, \vgamma_{\lamdelc}, \delta] \le \exp(\Const{c_Z_subgaussian}K^4 \lambda^2 \sum_{i \in \lamdel} \qty(\sum_{j \in \lamdelc} G_j \gamma_j m_{ij})^2).
\]

Plugging this into \cref{eq:reduction}, we get 
\[
E_{\delta} \le \EE_{\vgamma_{\lamdelc}}\qty[\EE_{\vg}\qty[\exp(\Const{c_Z_subgaussian}K^4 \lambda^2 \sum_{i \in \lamdel} \qty(\sum_{j \in \lamdelc} G_j \gamma_j m_{ij})^2)]]
\]

\paragraph{Integrating out the normal variables: }
Define the projection matrices $\mP_{\delta} \in \RR^{n \times n}$ to denote coordinate projection onto the indices in $\lamdel$ and $\mGamma = \diag(\vgamma) \in \RR^{n \times n}$ to denote the masking action of $\vgamma$. Define 
\begin{equation}
    \mB_{\delta, \vgamma} \triangleq \mP_{\delta}\mM(\mI_n - \mP_{\delta})\mGamma. \label{eq:B_delta_gamma}
\end{equation}
When $\mB_{\delta, \vgamma}$ acts on $\vg$, it first removes indices $j$ where $\gamma_j = 0$ or $j \not\in \lamdelc$. Then $\mM$ acts on the remaining vector, and finally it kills those indices $i$ such that $i \not\in \lamdel$. Hence, 
\begin{equation}
    \sum_{i \in \lamdel} \qty(\sum_{j \in \lamdelc} G_j \gamma_j m_{ij})^2 = \norm{\mB_{\delta, \vgamma} \vg}_2^2
\end{equation}
Let $\sigma_i$, $i\in [n]$ denote the singular values of $\mB_{\delta, \vgamma} $. 

Since the spectral norm of projection matrices are at most 1, we can use the definition of $\mB_{\delta, \vgamma}$ in \cref{eq:B_delta_gamma} to conclude that 
\begin{equation}
    \max_{i \in [n]} \sigma_i = \norm{\mB_{\delta, \vgamma}}_2 \le \norm{\mM}_2.\label{eq:spectral-norm}
\end{equation}
On the other hand, since $\gamma_j^2 = \gamma_j$ we have
\begin{equation}
    \sum_{i=1}^n \sigma_i^2 = \Tr(\mB_{\delta, \vgamma}^2) = \sum_{i \in \lamdel} \sum_{j \in \lamdelc} \gamma_j m_{ij}^2\label{eq:trace}
\end{equation}

By the rotation invariance of $\vg \sim N(0, \mI_n)$, we have that $\norm{\mB_{\delta, \vgamma} \vg}_2^2$ is distributed identically to $\sum_{i=1}^n \sigma_i^2 g_i^2$. Hence 
\begin{align}
    &\EE_{\vg}\qty[\exp(\Const{c_Z_subgaussian}K^4\lambda^2 \sum_{i \in \lamdel} \qty(\sum_{j \in \lamdelc} G_j \gamma_j m_{ij})^2) | \vgamma_{\delta^c}, \delta] \\
    &\quad\quad = \EE_{\vg}\qty[\exp(\Const{c_Z_subgaussian}K^4\lambda^2 \sum_{i=1}^n \sigma_i^2 g_i^2) | \vgamma_{\delta^c}, \delta] \\
    &\quad\quad = \prod_{i=1}^n \EE_{\vg}\qty[\exp(\Const{c_Z_subgaussian}K^4\lambda^2  \sigma_i^2 g_i^2) | \vgamma_{\delta^c}, \delta].
\end{align}
Recall the folklore fact that $g_i^2$ is a $\chi^2$ distribution with 1 degree of freedom whose MGF satisfies 
\begin{equation}
    \EE \exp(sg_i^2) \le (1-2s)^{-1/2} \le \exp(2s) \quad \text{ for }s < \frac{1}{4};\label{eq:chi_squared_mgf}
\end{equation}
this is used in \citet{rudelson2013hanson}. For $\lambda^2 < \frac{1}{4\Const{c_Z_subgaussian}K^4\norm{\mM}_2^2}$, we have 
\begin{equation}
    \Const{c_Z_subgaussian}K^4 \lambda^2 \sigma_i^2 < \frac{\Const{c_Z_subgaussian}K^4 \sigma_i^2}{4\Const{c_Z_subgaussian}K^4 \norm{\mM}_2^2} \le \frac{1}{4},
\end{equation} due to \cref{eq:spectral-norm}. 
Hence, we can use the $\chi^2$ MGF bound to obtain
\begin{align}
E_{\delta} &\le \prod_{i=1}^n \EE_{\gamma_{\lamdelc}}\qty[ \EE_{\vg}\qty[2\exp(\Const{c_Z_subgaussian}K^4\lambda^2 \sigma_i^2)| \vgamma_{\lamdelc}, \delta]] \\
&= \EE_{\gamma_{\lamdelc}}\qty[ \exp(2\Const{c_Z_subgaussian}K^4\lambda^2 \sum_{i \in \lamdel} \sum_{j \in \lamdelc} \gamma_j m_{ij}^2) | \delta], \label{eq:bernoulli-mgf}
\end{align}
where in the last line we used \cref{eq:trace}. 

\paragraph{Integrating out the Bernoulli mask: }
Now, we can use Lemma 5 in \citet{park2022sparse}.
\begin{lemma}[MGF bound for Bernoulli random variables]\label{lemma:bernoulli-mgf-bound}
    For $0 < s < \frac{1}{2\norm{\mM}_2^2}$, we have 
    \[
    \EE_{\vgamma_{\lamdelc}}\qty[\exp(s\sum_{i \in \lamdel} \sum_{j \in \lamdelc} \gamma_jm_{ij}^2) | \delta] \le \exp(2sp\sum_{i\neq j}m_{ij}^2).
    \]
\end{lemma}
The proof of this lemma uses similar proof techniques, so we defer its proof to the end of the section. The key step is to use the inequality $e^x \le 1+2x$ for $x \le \frac{1}{2}$ and $1+x \le e^x$ for all $x$. This allows us to pass the expectation into the exponential function and replace $\gamma_j$ with $p$.

For $\lambda^2 < \frac{1}{4\Const{c_Z_subgaussian}K^4\norm{\mM}_2^2}$, we can take $s = 2\Const{c_Z_subgaussian}K^4 \lambda^2$ and verify that $s < \frac{1}{2\norm{\mM}_2^2}$. Hence the lemma yields that
\begin{equation}
    E_{\delta} \le \EE_{\vgamma_{\lamdelc}}\qty[\exp(2\Const{c_Z_subgaussian}K^4 \lambda^2\sum_{i \in \lamdel} \sum_{j \in \lamdelc} \gamma_jm_{ij}^2) | \delta] \le \exp(4\Const{c_Z_subgaussian}K^4 \lambda^2\pi\sum_{i\neq j}m_{ij}^2)\label{eq:bernoulli-mgf-bound}
\end{equation}

\paragraph{Finishing off the proof:}
Let us collect our requirements on $\lambda$. To use the $\chi^2$ bound to obtain \cref{eq:bernoulli-mgf} and then bound it using \cref{lemma:bernoulli-mgf-bound}, we required $\abs{\lambda} \le \frac{1}{2\sqrt{\Const{c_Z_subgaussian}}K^2\norm{\mM}_2}$.

Hence, for $\abs{\lambda} \le \frac{1}{2\sqrt{\Const{c_Z_subgaussian}}K^2\norm{\mM}_2}$, we have 
\begin{align}
    \EE \exp(\lambda S_\offdiag) &\le \EE_{\delta}\EE_{\vgamma_{\lamdelc}} \EE_{\vx_{\lamdel}} \EE_{\vy_{\lamdelc}}\qty[\exp(4\lambda S_{\delta})|\vgamma_{\lamdelc}, \vx_{\lamdel}, \delta] \\
    &\le \EE_\delta E_\delta\\ 
    &\le \exp(4C_2\lambda^2\pi\sum_{i \neq j} m_{ij}^2).
\end{align}
where in the first line we have used \cref{eq:S_delta_exp}, in the second line we have used \cref{eq:E_delta_ineq}, and in the last line we have used \cref{eq:bernoulli-mgf-bound}.

Applying Markov as in the diagonal case yields the desired result.
\end{proof}

We now prove \cref{lemma:bernoulli-mgf-bound}.
\begin{proof}[Proof of \cref{lemma:bernoulli-mgf-bound}]
Since all random variables are independent across $j$, we have 
\begin{align}
\EE_{\vgamma_{\lamdelc}}\qty[\exp(s\sum_{i \in \lamdel} \sum_{j \in \lamdelc} \gamma_jm_{ij}^2) | \delta] &= \prod_{j \in \lamdelc} \EE_{\vgamma_{\lamdelc}}\qty[\exp(s \gamma_j\sum_{i \in \lamdel} m_{ij}^2) | \delta] \\
&= \prod_{j \in \lamdelc} \qty[(1-\pi) + \pi\exp(s \sum_{i \in \lamdel} m_{ij}^2)].
\end{align}
We use the approximation $e^x \le 1+2x$, which holds for $0 \le x \le \frac{1}{2}$ and the inequality $1+x \le e^x$ which holds for all $x$. In order for us to apply these inequalities, we need to check that $s \sum_{i \in \lamdel} m_{ij}^2 < \frac{1}{2}$. Indeed, note that 
\begin{equation}
    \sum_{i \in \lamdel} m_{ij}^2 \le \max_j \sum_{i = 1}^n m_{ij}^2 \le \norm{\mM}_2^2,
\end{equation}
so for $s < \frac{1}{2\norm{\mM}_2^2}$ we have 
\begin{equation}
    s \sum_{i \in \lamdel} m_{ij}^2 < \frac{\norm{\mM}_2^2}{2\norm{\mM}_2^2} \le \frac{1}{2}.
\end{equation}
Hence, we have 
\begin{align}
    \prod_{j \in \lamdelc} \qty[(1-\pi) + \pi\exp(s \sum_{i \in \lamdel} m_{ij}^2)] &\le \prod_{j \in \lamdelc} \qty(1 + 2s p\sum_{i \in \lamdel} m_{ij}^2) \\
    &\le \prod_{j \in \lamdelc} \exp(2s\pi\sum_{i \in \lamdel} m_{ij}^2) \\
    &\le \exp(2s \pi\sum_{j \in \lamdelc} \sum_{i \in \lamdel} m_{ij}^2) \\
    &\le \exp(2s \pi\sum_{i \neq j} m_{ij}^2).
\end{align}

This proves that 
\begin{equation}
\EE_{\vgamma_{\lamdelc}}\qty[\exp(s\sum_{i \in \lamdel} \sum_{j \in \lamdelc} \gamma_jm_{ij}^2) | \delta] \le \exp(2s\pi \sum_{i \neq j} m_{ij}^2),
\end{equation}
irrespective of $\delta$, as desired.
\end{proof}

\section{Alternative framing of bi-level ensemble}\label{sec:exposition}
In this section we provide an alternative  exposition for the bi-level ensemble, which may be more intuitive for some readers. 

The high level conceptual picture of the setup is as follows. The learner simply observes high dimensional jointly Gaussian zero-mean features with some unknown covariance $\Sigma \in \RR^{d \times d}$. It then performs min-norm interpolation of (essentially) one-hot-encoded labels to learn the score functions for the $k$ different classes. These score functions are used at test-time to do multiclass classification.

Note that the learning algorithm has no knowledge of $\Sigma$, nor does it even know that the features are jointly Gaussian --- all it has is the training data. For analysis purposes, we parameterize $\Sigma$. In the spirit of spiked covariance models, where a low-dimensional subspace has higher variance, we study the case that the eigenvalues of $\Sigma$ follow the simplified bi-level model parameterized by $(p, q, r)$. 

The bi-level model stipulates that (i) the number of features is $d = n^p$ and (ii) there are two discrete variance levels for the Gaussian features. The higher variance component resides in a low-dimensional subspace of dimension $s = n^r$. For this bi-level model, we are able to prove sharp phase-transition style results telling us when successful generalization will happen. Here, the number of classes also matters, and that is where the parameter $t$ enters.

To simplify notation for analysis in the paper, \citet{subramanian2022generalization} just assume that $\Sigma$ is diagonal to begin with, instead of explicitly rotating coordinates to the eigenbasis of a general $\Sigma$. Because Gaussianity is preserved under rotations and min-norm interpolation only cares about norms, this transformation is without loss of generality. We reiterate that the learner is agnostic to all of these choices and does not know $\Sigma$ at any point.


\end{document}